\pdfoutput=1
\documentclass[11pt]{article}

\usepackage[T1]{fontenc}
\usepackage{lmodern}

\usepackage[margin=1.05in]{geometry}
\usepackage{amsmath,amssymb,amsthm}
\usepackage{booktabs}
\usepackage{longtable}
\usepackage{array}
\usepackage{graphicx}
\usepackage{microtype}
\usepackage{tikz}
\usetikzlibrary{arrows.meta,positioning,shapes.geometric}
\usepackage{algorithm}
\usepackage{algpseudocode}
\definecolor{linknavy}{RGB}{0,51,102}
\usepackage[colorlinks,linkcolor=linknavy,citecolor=linknavy,urlcolor=linknavy]{hyperref}
\hypersetup{pdftitle={SoftModel: A Neural Model That Grows Its
Own Topology --- Governed Structural Growth for Continual
In-Service Learning},pdfauthor={Zhoumin Xie},
pdfsubject={growable soft models; continual in-service
learning},pdfkeywords={continual learning, lifelong learning,
catastrophic forgetting, growing neural networks, structural
plasticity}}
\usepackage{titlesec}
\usepackage[font=small,labelfont=bf]{caption}
\renewcommand{\thepart}{Part~\Roman{part}}
\titleformat{\part}[block]
  {\filcenter\fontsize{20}{24}\selectfont\bfseries}{\thepart.}{0.6em}{\fontsize{20}{24}\selectfont\bfseries}
\titlespacing*{\part}{0pt}{2.0em}{1.4em}
\titleformat{\section}
  {\normalfont\fontsize{16}{19}\selectfont\bfseries\raggedright}
  {\thesection}{1em}{}
\titleformat{\subsection}
  {\normalfont\fontsize{13}{15.5}\selectfont\bfseries}
  {\thesubsection}{1em}{}

\theoremstyle{definition}
\newtheorem{definition}{Definition}
\newtheorem*{axiom}{Axiom (total plasticity)}
\newtheorem*{principle}{Principle (stability through governance, not immobility)}
\newtheorem*{invariant}{Invariant (gate soundness)}
\theoremstyle{plain}
\newtheorem{proposition}{Proposition}
\newtheorem*{lockin}{First-batch lock-in (conjecture)}

\title{\textbf{SoftModel: A Neural Model That Grows Its Own
Topology}\\[0.35em]
{\large Governed Structural Growth for Continual In-Service
Learning}}
\author{Zhoumin Xie\\ \small Independent Researcher\\
\small \href{mailto:xiezm75@gmail.com}{xiezm75@gmail.com}
$\cdot$ ORCID
\href{https://orcid.org/0000-0002-9361-1358}{0000-0002-9361-1358}}
\date{}

\usepackage{fancyhdr}

\begin{document}
\maketitle

\begin{abstract}
Today, a neural system is almost always used in two phases ---
trained, then deployed --- and in that regime it freezes twice:
training ends, and the topology itself was never a degree of
freedom. We take the
opposite premise as an axiom --- total plasticity: no part of a
model, including its structure, is ever frozen --- and derive the
governance a lifelong learner then requires. The design's target
regime is continual, in-service learning: a long-lived model on a
non-stationary stream, whose stability is obtained from
governance rather than immobility and whose capacity follows
demand. The result is a
growable soft model: an algebra of structural operators (width,
hierarchy, composition, input interface, grown cycles, attention
heads), each exact at the moment of application, each budgeted
and audited, with adoption decided solely by a held-out reality
gate that treats parametric and structural change uniformly.
Local solving loops --- grown directed cycles under an enforced
contraction certificate, and self-processing units that refine
newborn structure against a local objective --- extend the same
governance beneath the global pass. A complete from-scratch
system realizes the whole account; its factory surface is
operated end-to-end by a production LLM through a
self-documenting tool surface.
Two conclusions follow from the axiom by construction.
Stability under lifelong change becomes an audit property of
the lifecycle, not an immobility property of parameters ---
nothing need be frozen to be safe. Structure that follows
demand removes the silent cap a fixed topology places on
later capability where the capacity floor binds: a model born
small need not remain bounded by its birth. A third is measured: in the worlds where this was
measured, the marginal value of new capacity was unobservable
before adoption, so workable growth governance took its ex-post form
--- speculation free, adoption earned on held-out reality,
and silence in an unchanging world a governed outcome with a
price. The same governance extends to evaluative
signals --- structure preference and policy optimization on the
growing model under one gate --- and the core method is
evaluated on standard continual-learning benchmarks, where
governed growth preserves the ability to keep learning along
long task sequences and converts announced review into
recovered competence. A pre-registered experimental program adjudicates the
mechanism and value claims on the tested problems and reports
its failures at full prominence; the map --- positive and
negative --- is the contribution.
\end{abstract}


\tableofcontents

\section{Introduction}

\subsection{The two freezes}

Today, a neural model is almost always used in two phases ---
weights updated in training, then held constant in
service~\cite{dohare} --- and in that regime it is frozen twice.

The first freeze is familiar and explicit: after training, weights
are fixed. The model becomes an artifact --- versioned, shipped, and
never again improved by its own experience. The continual-learning
literature exists to lift this freeze, and its central obstacle is
equally familiar: unconstrained weight plasticity causes interference
and catastrophic forgetting, so the standard remedies re-introduce
immobility in softer forms --- freeze earlier columns and attach new
ones; penalize departure from past weights; freeze a backbone and
train low-rank residuals around it. Stability is purchased with
plasticity, at some exchange rate, and the field's implicit consensus
is that this purchase is unavoidable.

The second freeze is quieter, seldom treated as a design axis in
its own right (the dynamic-architecture literature works on it
without this name; surveyed in App.~\ref{sec:related}), and --- we will
argue --- more fundamental: \textbf{the topology is fixed at design
time.} The number of units, the connectivity, the depth, and the
input interface are decided before the first example arrives, and
they never change again. Every subsequent act of learning, replay,
regularization, or adaptation moves the model \emph{within} the
function space this topology determines; none of it can move
the space itself. A model in this condition satisfies every conventional
definition of ``trainable'' while being unable, even in principle, to
follow a world whose demands leave its function space.
Section~\ref{sec:theory} makes this precise as \emph{structural
freezing by omission}, with a practical conjecture we call
\emph{first-batch lock-in}: when capacity and interface are fitted to
the first data a model ever sees, early data permanently caps later
capability.

\subsection{The inversion}

This paper develops the consequences of refusing both freezes at
once. We take as an axiom --- a definition of the object of study,
not an empirical claim --- that in a \textbf{soft model} every
parameter, at every scale of structure, remains trainable for the
lifetime of the model. The axiom immediately forces topology itself
to become a trainable degree of freedom: a learner that cannot grow its width, its hierarchy, its
serial composition, and its own input interface is not fully
plastic.

The obvious objection is that this is exactly the regime in which
lifelong learning fails. Our answer is an inversion of the standard
resolution:

\begin{quote}
\textbf{Stability through governance, not immobility.} Every change
to the model --- a gradient step, a new unit, a new input feature, a
rebuilt network --- must be (i) \emph{harmless at the moment it is
applied} and (ii) \emph{adopted only if it pays on held-out reality}.
Nothing else about the model is restricted, beyond declared,
adjustable budgets.
\end{quote}

Condition (i) is made mechanical by constructing every structural
operator to be \textbf{exact-preserving}: at application time the
model's input--output function is unchanged (new units enter with
zero outgoing weights; a refined node's inner network enters
contributing exactly zero, so the node initially computes what it
computed before; a new input feature enters with zero input weights).
Growth never damages the present; it only adds trainable directions
the future may use. Condition (ii) is made mechanical by a
\textbf{reality gate}: all training and growth happen in a
speculative \emph{working state}, and promotion to the
\emph{committed version} --- the one that serves --- requires
strictly surpassing the incumbent's score on the most recent slice of a held-out
stream of labeled reality, quarantined from all training. Under these
two conditions, stability stops being a property of parameters
(immobility) and becomes an audit property of the lifecycle
(governance): the served lineage never adopts a change that scored below
the incumbent on the evidence available when it was adopted --- and
this guarantee costs no plasticity at all.

Governance changes the character of growth. Because a structural edit
is harmless now and adjudicated later, growth becomes \textbf{low-risk speculation}: the system may widen where its own instability signals
concentrate, refine an unstable node inward, deepen a scope whose
additive route has stopped paying, admit a newly observed input,
or re-derive itself from its own experience store --- even
carelessly --- because the committed model can only improve or stay.
There is no architecture search over training runs and no outer
optimization loop outside one life;
growth is need-driven and evidence-adjudicated inside the model's own
lifetime.

\emph{How} to grow is not an arbitrary design choice either; it
proceeds coarse-to-fine: coarse
structure first, finer refinement nested adaptively
and non-uniformly, exactly where the data prove the current scale
insufficient.

\subsection{The system}

We realize this position in a complete, released system. Its
substrate is a recursive multi-scale network in which any hidden node
may contain an inner network, so \textbf{structure is grown, not
designed}: it appears where the model's experience demands it, and
one learning primitive operates identically at every level. This
recursive network is the \emph{reference} substrate; the substrate
contract itself is architecture-agnostic, and the released system
provides multi-layer hosts under the same lifecycle, among them a
transformer encoder (\S\ref{sec:algebra}). Growth proceeds in two
directions --- additive refinement and widening, and
scope-interior composition --- with the direction selected per
scope by the model's own signals (\S\ref{sec:direction}); a
system-level mode restores the purely additive regime.
The containment relation of scopes forms a rooted tree --- the
\emph{inclusion tree}, in the standard graph-theoretic sense of
levels and height~\cite{knuth}, a part--whole hierarchy in Simon's
sense~\cite{simon}. The model carries two depth
quantities, both grown and both
non-uniform: the \emph{height of the inclusion tree}~\cite{knuth,clrs} --- ``tree height'', evolved per branch --- and
each scope's \emph{layer depth} --- the length of its composition
chain, evolved per scope (\S\ref{sec:direction}).
Over this substrate we define the operator algebra --- $\omega$ (widen), $\rho$ (hierarchical
refinement), $\delta$ (deepen by scope-interior composition),
$\sigma$ (grow the input interface), $\Phi$ (re-found from the
experience store) --- with
exactness measured, not assumed, in the released acceptance suite.
Models are born by \textbf{self-shaping}: creation takes only a name
(with one measured qualification --- birth capacity remains a
configuration obligation for symbol-heavy tasks, App.~\ref{sec:gseries}),
and the feature space, the output form (numeric regression or a
categorical head over a learned vocabulary), and the capacity are
inferred from the data the model is taught --- the zeroth growth
operation, applied at birth.

Finally, the system is built to be operated by a general LLM. The
division of labor is strict: the LLM --- the \emph{brain} ---
supplies all general capability (language, context, feature
extraction) and drives the entire lifecycle through a standard tool
protocol; the soft model is the brain's growable extension, carrying
one domain's learned judgment in kilobytes of forever-trainable
weights. The interface guarantees make careless --- though not adversarial
(\S\ref{sec:llm}) --- operation safe under a careless-but-honest operator
(\S\ref{sec:lifecycle}, threat model): teaching cannot degrade the
served model, growth cannot damage the present, and every decision
is audited. (System name in prose:
SoftModel; the implementation is named \texttt{GrowableSoftModel}.)

\subsection{Contributions and epistemic stance}
\label{sec:contributions}

\begin{enumerate}
\item \textbf{A definitional theory} of total plasticity: the axiom,
the lock-in argument (fixed topology is structural freezing by
omission, with capacity and interface instances of different
severity), and the governance principle that replaces immobilization
(\S\ref{sec:theory}).
\item \textbf{A substrate and operator algebra} for growing topology
under governance: a recursive substrate whose structure is grown,
not designed; a scale-recursive learning rule --- \emph{cross-scale
backpropagation} --- in which corrective information travels
backward as residual \emph{targets} down the inclusion tree rather
than as gradients through it; four
exact-preserving, budgeted, audited operators spanning the
additive and compositional directions, plus gated
re-founding; a signal-based selection
of the growth direction (prediction under a predictability
certificate, probes read only at extrapolated asymptotes,
adoption by the gate), with every algorithmic role a replaceable
registry part; growth signals and verdict discipline; and a
pluggable substrate contract whose
released hosts include a from-scratch transformer encoder on which
the same operators and lifecycle verify (\S\ref{sec:algebra}).
\item \textbf{A principle of local solving loops} ---
a further method contribution: computation that solves
\emph{inside} itself, at the scale that demands it, without
changing the global learning contract --- realized twice
(\S\ref{sec:loops}). The \emph{loop operator} ($\lambda$) grows a
governed directed cycle at a scope: a third growth direction in
which the bought resource is neither capacity nor composition
but \emph{iteration and state} --- stability-certified by an
enforced contraction bound, exact at application, budgeted,
audited, and adopted only through the gate. \emph{Self-processing
units} give every newly grown structure a bounded, label-blind
weight-refinement loop against a local process objective ---
process, then participate, then learn, at every depth of the
inclusion tree. Both mechanisms obey the same governance
contract as every operator of (2), and both are selectable
policy, never fiat.
\item \textbf{Growable attention} --- a method
contribution: growth and self-processing reach the attention
component itself (\S\ref{sec:growatt}). Heads are added and
widened per head --- unequal widths as a \emph{lifecycle
outcome} rather than a design choice --- under the same
evidence-triggered, exactly function-preserving, gate-adjudicated
contract as every other operator; and each head's attention
distributions answer to a local entropy-band objective --- a
third \emph{local objective} under the split-locality
discipline (a penalty, not a new inner loop) --- whose gradient
is
certified and whose locality rule (parameter gradients from the
combined signal, block input from the task signal alone) is
enforced by measurement, not convention.
\item \textbf{A governed lifecycle} unifying learning and growth
under one gate, with self-shaping birth, store-per-version
reproducibility, a drift protocol, and consent-gated self-study
(\S\ref{sec:lifecycle}).
\item \textbf{An LLM-operated design}: the model as a tool the brain
both uses and teaches, over self-documenting MCP servers; validated
end-to-end on the factory surface by a production LLM
(\S\ref{sec:llm}).
\item \textbf{A pre-registered empirical program with FAIL branches
at full prominence} (App.~\ref{sec:empirical}). We report what the
protocol adjudicated, not what we hoped: the mechanics verify
exactly; several value-of-growth predictions failed or their
scenarios failed to bind at the tested scale; and the model's self-knowledge signals proved real --- the
variation self-check is the program's best measured predictor of
transfer error (SV1) --- even though naive interventions built on
them did not clear their bars.
\item \textbf{Evaluative learning and benchmark evidence}: one
evaluative primitive at two levels of the system --- preference
over growth moves and clipped-surrogate policy optimization
through a pseudo-target identity --- under the same gate,
verified quasi-statically against official components
(\S\ref{sec:eval-learn}); and a registered continual-learning
benchmark program whose protocol-fit discipline (the training
protocol belongs to the method under test; each evaluation
subject receives its fitting measurement; examinations never
reuse study items) yields the keep-learning and
just-in-time-refresher results (\S\ref{sec:cl-bench}).
\end{enumerate}

Five further contributions join those itemized above: the
\emph{2026 campaign record} (\S\ref{sec:campaigns-results}:
in-service insertion exactness at scale, the
continual-learning treatment boundaries and review-scheduling
laws, retention measured as relearning savings, the
depth-separation mapping, and the delayed-payoff selection
signal); a \emph{pedagogy for lifelong models}
(App.~\ref{sec:pedagogy}: ten principles, each tied to a
measured result, enforced in a registered curriculum, or marked
open); the \emph{generative
growth-value result} (governed growth pays under arriving
structure; against budget-matched fixed models the static
scores serve as a sanity floor --- parity or better, held
under a fair-uncertainty rerun --- while the design's
evaluation lives on the time axis, \S\ref{sec:cl}); the \emph{fidelity program}
(four registered rounds from a defection baseline to a
supported education thesis under a corrected time-axis
instrument --- 2/5 under the original level instrument, the
correction disclosed); and a digest of \emph{empirical observations}
closing the evidence part (\S\ref{sec:empirical-laws}), each
carrying a registered artifact.
The stance in (7) deserves one sentence of defense. A theory whose
central claim is \emph{judge every change against reality} would be
self-undermining if its own paper cherry-picked; the pre-registered
verdicts, including the negative ones, are the theory applied to
itself.

\paragraph{What we claim, and what we do not.} The contribution is
a governed \emph{mechanism}: exactness and safety are demonstrated
by measurement, governance and auditability by construction, and
the adaptive choice of growth direction at simulation scale. A first pre-registered matched-parameter comparison against an
EWC baseline is now on record (App.~\ref{sec:gseries}):
governed plasticity matched or improved on the baseline's best grid cell
on adaptation and retention jointly on three of four lifelong
tasks, with the fourth --- where EWC adapted better ---
reported as a miss; we do not claim superiority over the
broader immobilization family (progressive columns, adapter
freezing) --- that comparison remains open --- nor that composition
is more effective than widening \emph{in general}: the measured
picture is regime-dependent --- the first-generation bar failed
(App.~\ref{app:twodir}), and the 2026 depth mapping locates the
compositional advantage precisely where iterated composition is
real, at fixed parameter budget, and nowhere else
(\S\ref{sec:campaigns-results}). The claims are bounded by the evidence table,
and the boundary is part of the method.

\subsection{Organization and reader's paths}
\label{sec:organization}

The paper is organized in three parts.
\textbf{Part~I, Theory and System} (\S2--\S10), develops the
axiom and its consequences, the substrate and its instruments, the
growth-direction question, the two local-solving realizations,
growable attention, the governed lifecycle, its extension to
evaluative learning (\S\ref{sec:eval-learn}), the system in
operation under an LLM operator --- a system demonstration, not
an empirical claim --- and closes by placing the design in its
target regime, continual in-service learning (\S10).
\textbf{Part~II, Empirical Evidence} (\S11--\S14), is the
measured record, organized as benchmarks $\to$ discipline
$\to$ results $\to$ observations:
\S\ref{sec:cl-bench} evaluates the core method on the standard
continual-learning benchmark family, \S12 states the registered
program and the rules every verdict obeys, \S13 reports the
main results --- including the 2026 depth-axis and
continual-learning campaigns --- and \S14 collects the
empirical observations and bounds the measurements' validity.
\textbf{Part~III, Discussion and Conclusion} (\S15--\S16),
interprets and closes.
Three reader's paths follow the parts: a reader here for the theory
reads Part~I; a reader here to reproduce or audit reads Part~II
with the process appendices (the registered program, the
growth-value campaigns, and the fidelity program, with the
remaining appendices as their instruments); the evaluative-learning
extension (\S\ref{sec:eval-learn}) sits with the system
sections it extends, and the benchmark evaluation
(\S\ref{sec:cl-bench}) opens the empirical part by testing
the core method on the field's own instruments; a reader who wants
conclusions only reads \S\ref{sec:contributions}, \S13, and
Part~III. The appendices carry process --- protocols, per-seed
tables, registered rounds --- and their evidence chain closes
with the complete correspondence table
(Table~\ref{tab:evidence}), the single table every claim in
this paper answers to, followed by a glossary of named
quantities. Section~\ref{sec:contributions}
says \emph{what} is claimed; this subsection says \emph{where}
each claim lives.

\part{Theory and System}

\section{Total Plasticity: The Axiom, Lock-In, and Governed Growth}
\label{sec:theory}

\subsection{The axiom}

We call a model \textbf{soft} when its capability resides in continuously mutable state --- weights that training can always move
--- rather than in code, rules, or frozen artifacts. We take softness
to be definitional and state it as an axiom:

\begin{axiom}
Every parameter of the model, at every scale of its structure,
remains trainable for the lifetime of the model --- in operation as
in experiment. There is no stop-gradient, no frozen column, no
protected backbone.
\end{axiom}

Three corollaries fix the axiom's meaning. \emph{(i) No disguised
freezing.} Mechanisms that leave parameters formally trainable but
practically immobile --- effective learning rates driven to zero,
hard penalties pinning weights to their past --- violate the axiom as
surely as a stop-gradient does. The axiom legislates
\emph{capability}, not \emph{pace}: a mechanism violates it
only when learning cannot resume --- rates pinned to zero, or
unrecoverable by the system's own evidence. How slowly a healthy
model should move through a quiet regime is an operating policy,
judged by the operator against the problem at hand; the axiom
requires only that the ability to learn remain intact, at every
scale, at every moment. \emph{(ii) Optimizability.}
Plasticity without training signal is vacuous: every scale must also
receive effective optimization pressure, and a system in which deep
structure cannot mature is out of compliance in practice even if no
gradient is blocked in principle. \emph{(iii) Structural plasticity.}
The axiom quantifies over ``every scale of structure'' --- and
structure itself, the next subsection argues, is a parameter one can
freeze.

\subsection{Lock-in: fixed topology is freezing by omission}
\label{sec:lockin}

Let a topology $G$ --- the units, their connectivity, and the input
interface --- induce the function family
\begin{equation}
\mathcal{F}(G) \;=\; \{\, f_\theta : \theta \in \Theta(G) \,\},
\end{equation}
the set of everything weight training can ever reach. Weight
plasticity moves the model \emph{within} $\mathcal{F}(G)$; it never
moves $\mathcal{F}(G)$ itself. Two instances make the resulting
freeze concrete, in increasing order of severity.

\emph{The capacity instance.} Universal approximation is an asymptotic statement about
topology \emph{families}~\cite{cybenko}, not about any fixed
member: a topology of fixed size carries an irreducible
approximation floor, and when the world drifts toward targets above
that floor, weight training, replay, and regularization tuning all
saturate at it. The failure is silent: training converges to the best
approximation available within $\mathcal{F}(G)$, whose error is
bounded below by that floor.

\emph{The interface instance.} If the drift introduces a causal
factor that is not among the model's inputs, the best available
hypothesis is not merely hard to reach --- it is not in
$\mathcal{F}(G)$ at all, at any capacity. No interior mechanism can
compensate for a variable the model cannot see. This instance is
structural and scale-free.

In both instances the model is frozen at the level that matters while
satisfying every conventional definition of ``trainable.'' We call the condition \textbf{structural freezing by
omission}: no freeze is imposed; the freeze consists in the
absence of any operator that could move $\mathcal{F}(G)$. It carries a sharp practical conjecture:

\begin{lockin}
If capacity and interface are fixed at creation, they are necessarily
fitted to the first data the model ever sees. Everything learned
afterward is confined to a function space sized by that first batch:
early data permanently caps later capability. (The direct
attempt to \emph{induce} this cap at kilobyte scale did not
bind --- E11a, App.~\ref{sec:empirical}; the corollary stands as the mechanism the
axiom removes, not as a demonstrated failure mode at this
scale.)
\end{lockin}

The axiom therefore forces a conclusion the continual-learning
literature usually treats as optional: \textbf{topology must be a
trainable degree of freedom} --- reachable by operators, under
optimization pressure, for the model's lifetime. A system without
such operators satisfies parameter-level plasticity while
$\mathcal{F}(G)$ remains fixed: plastic in $\theta$, frozen in
$G$. The contrast with a conventional deep network is categorical:
there, depth is a single global hyperparameter --- fixed at design
time, identical along every path; under the axiom, depth is a
grown, non-uniform field over the structure --- each branch of the
inclusion tree and each scope's composition chain evolves its own,
by local signals, under the gate
(\S\ref{sec:algebra}--\ref{sec:direction}).

\emph{The direction instance.} The omission argument applies once
more, one level up --- to the growth algebra itself. Suppose every
operator the algebra offers adds terms to an additive expansion.
For such targets the classical depth separations exhibit
function families that shallow additive forms can approximate
only at sizes growing exponentially in the depth imitated or in
the input dimension (\S\ref{sec:discussion}: a question of
required size, not of attainable class): the reachable class is
nonempty at every size, but for those families approximants of
feasible size are excluded --- again not by an imposed freeze,
but by the absence of any operator that could compose.
An algebra that can grow solely by addition therefore freezes the
\emph{direction} of growth by omission. The axiom demands
direction freedom as it demands topological freedom; and since no
direction may be imposed by designer fiat either, the choice of
direction is itself a governed decision, read from the model's own
signals (\S\ref{sec:direction}).

\subsection{The stability dilemma and its governed resolution}
\label{sec:governance}

The classical objection arrives immediately: unrestricted plasticity
is precisely what makes lifelong learning fail --- interference,
catastrophic forgetting~\cite{mccloskey}, runaway structure ---
the classical stability--plasticity dilemma~\cite{grossberg80,artnn}. The field's standard
remedies are forms of \emph{immobilization}: freeze earlier columns (progressive networks~\cite{prognets}), penalize movement away from past weights
(EWC-family~\cite{ewc} --- viscous freezing), freeze the backbone
and adapt residuals (adapter methods~\cite{lora}). Immobilization buys stability by
spending exactly the plasticity the axiom demands, and by
\S\ref{sec:lockin} the fixed topology beneath these methods was never
plastic to begin with.

We resolve the dilemma on the other side:

\begin{principle}
Every change to the model --- parametric or structural --- must be
(i) harmless at the moment it is applied and (ii) adopted only if it
pays on held-out reality. Nothing else is restricted, beyond
declared, adjustable budgets.
\end{principle}

Both conditions are mechanical, not aspirational.

\paragraph{(i) Exactness at application.} Every structural operator
is constructed so that the model's input--output function is
unchanged at the moment of application: new units enter with
zero-initialized outgoing weights; a refined node's inner network
enters contributing exactly zero, so the refined node initially
computes exactly what it computed before; a new input feature enters
with zero input weights and a neutral backfill for stored rows.
Growth adds trainable directions; it never moves the current
function. (Measured, not assumed, in the released system ---
App.~\ref{sec:empirical}.)

One boundary case is principled rather than engineered:
\emph{vocabulary growth}. When a new class label arrives
(\S\ref{sec:lifecycle}), the categorical head's output space
itself must widen, and exact preservation of the old classes'
probabilities would require the new logit to enter at $-\infty$
--- an infinite loss on the class's first observation and an
untrainable direction, which the no-freezing axiom forbids. The
new class therefore enters with a zero-weight logit row and a
bias of $-10$: the old classes' probabilities move by less than
$10^{-3}$ (the $e^{-10}$ leakage through the softmax) while every
direction stays finite and trainable from step one. Vocabulary
growth is thus the one $\varepsilon$-preserving structural edit;
the operators of Table~\ref{tab:operators} are exact to machine
precision.

\paragraph{(ii) The reality gate.} The model maintains a speculative
\textbf{working state}, where training and growth are free, and a
\textbf{committed version}, which serves. Promotion is a strict
comparison on the model's own \textbf{held-out stream} ---
time-stamped labeled reality, quarantined from all training ---
evaluated on its most recent slice: the candidate replaces the
incumbent only if strictly better. Refusal is an outcome, not an error; a rejected candidate's
data never enters the served lineage. The guarantee is stated at
its true strength and no more: adoption-level protection of the
served lineage --- no forgetting bound and no regret-style
statement is claimed. The criterion has the form
of a variational principle: as virtual work admits a displacement
only if it lowers the governing functional, the gate realizes a
variation --- parametric or structural --- only if it lowers the
held-out measure~\cite{lanczos}.

Together the two conditions yield the property that replaces
frozen-backbone safety:

\begin{invariant}
At every promotion, the newly served version scored strictly higher
than the freshly re-scored incumbent on the gate's evaluation slice;
at every refusal, the served version is unchanged. The served lineage
therefore never adopts a change that scored below the incumbent on the
reality used to judge it.
\end{invariant}

We state the invariant's scope as plainly as the invariant. The gate
certifies decisions against reality \emph{as sampled by the recent
held-out slice at decision time} --- governance, not omniscience.
Under drift the slice itself moves; that is its purpose. What is
guaranteed is an audit property of the lifecycle: no adopted change
ever scored below the incumbent on the evidence available when it was
adopted. Stability, so construed, is not an immobility property of
parameters at all --- and it costs no plasticity.

\paragraph{Why exactness, given the gate?} The two conditions
are not redundant. The gate protects the \emph{served} model; it
says nothing about the working state in which growth happens.
Exactness protects that working state's accumulated fit ---
a structural edit costs nothing at the moment it lands, so growth
composes with ongoing learning without a restart --- and it makes
every edit's effect \emph{attributable}: since the function is
unchanged at application, whatever the ledgers record afterwards
is the edit's earned contribution, not an artifact of the
splice.

\subsection{Growth as governed speculation}

Exactness makes an edit harmless now; the gate decides later whether
it paid. Between the two, structural growth changes character: it
becomes \textbf{low-risk speculation}. The system may widen where its
instability signals concentrate, refine an unstable node inward,
deepen a scope whose delayed-payoff slope certifies the need, admit a
feature that recently became observable, or rebuild itself from
its own experience store --- carelessly, even --- because the
committed model can only improve or stay. Three disciplines keep
speculation disciplined without restricting it: \textbf{budgets}
(parameter and depth caps as policy; a refused growth is an outcome
to respect, not an error to fight), \textbf{audit} (every operator
application is a lineage event), and \textbf{verdicts}
(\S\ref{sec:signals}: never act on progress signals classified as
false). There is no architecture search over training runs here
and no outer loop outside one life;
growth is need-driven and evidence-adjudicated inside one continuing
life (Figure~\ref{fig:theorygrowth} illustrates the algebra
schematically; the two directions are set side by side in
Figure~\ref{fig:twodir}). The axiom then closes the loop on itself: because nothing is
ever frozen, every parameter that speculation adds is itself forever
trainable --- including by later speculation, recursively.

\subsection{Falsifiable predictions}
\emph{(Several of these predictions were graded by the 2026
campaigns; the measured verdicts, including the boundaries, are
collected in \S\ref{sec:campaigns-results}.)}
\label{sec:predictions}

The theory is not self-certifying, and we commit it to predictions
tested by the pre-registered program of App.~\ref{sec:empirical}:
\textbf{T1} with all scales training, adaptation to a drifted world
localizes at the scale at which the drift occurs; \textbf{T2}
narrow-born models recover most of the gap to well-born models once
$\omega$ releases first-batch capacity lock-in; \textbf{T3} the
never-frozen coarse scale shows no excess damage under label noise
relative to a governance-matched control; \textbf{T4} amplitude
inversion --- fine-scale corrections outgrowing the coarse signal ---
is an early warning that predicts when re-founding ($\Phi$) is
necessary. We report the adjudicated outcomes, including the branches
where a prediction failed or a scenario failed to bind, at full
prominence. One outcome tempers this very section and we say so here:
at the kilobyte scale tested, our pre-registered T2 scenario failed
to \emph{induce} capacity lock-in at all (App.~\ref{sec:eseries}) ---
the capacity instance of \S\ref{sec:lockin} awaits scenarios where
the approximation floor genuinely binds, while the interface instance
is structural and holds at any scale. Each prediction's adjudicated
row appears in Table~\ref{tab:evidence}, and the program-wide
discipline under which these verdicts were reached is stated once
in \S\ref{sec:evidence-discipline}.

\section{The Substrate, Learning, and the Shared Instruments}
\label{sec:algebra}

This section gives the shared core: the recursive substrate
(Definition~\ref{def:net}), the scale-invariant learning
primitive (Algorithm~\ref{alg:primitive}), the instability signal
that aims growth, the interface and re-founding operators
($\sigma$, $\Phi$), and the gate (Algorithm~\ref{alg:commit});
the growth operators and their exactness propositions follow in
\S\ref{sec:direction}. Everything here is the released
implementation stated mathematically; nothing is idealized.

Named quantities and roles used throughout are collected in
App.~\ref{app:glossary}.

\subsection{A substrate closed under refinement}

\begin{definition}[recursive network]
\label{def:net}
A \emph{network} over input space $\mathbb{R}^d$ is a tuple
\begin{equation}
N \;=\; \big(W_1 \in \mathbb{R}^{H\times d},\; b_1 \in
\mathbb{R}^{H},\; W_2 \in \mathbb{R}^{1\times H},\; c \in
\mathbb{R},\; \mathcal{I}\big),
\end{equation}
where $\mathcal{I}$, with
$J \subseteq \{1,\dots,H\}$, assigns to certain hidden units --- the \textbf{composite
nodes} --- an \emph{inner network over the same input space}
together with its attachment matrix,
$\mathcal{I}(j) = (N_j, A_j)$ with $A_j \in
\mathbb{R}^{k_j \times H}$.
We call each network in the inclusion tree --- the root and every
inner network --- a \emph{scope}. The scalar head is
definitional convenience: the released system also serves a
categorical head over a learned vocabulary (\S\ref{sec:lifecycle}),
and every statement below is head-agnostic.
\end{definition}

The class is defined by this self-reference and is therefore
\textbf{closed under refinement}: any node of any inner network may
itself become composite.

\paragraph{Forward map.} With $\hat{x}$ the standardized input and
$\varphi$ the GELU nonlinearity~\cite{gelu}, the
pre-activations and hidden
outputs are
\begin{equation}
a(x) = W_1\hat{x} + b_1, \qquad
h(x) = \varphi\big(a(x)\big) \;+\; \sum_{j\in J}
g_{j}(\hat{x})\,A_{j},
\label{eq:forward}
\end{equation}
where $g_{j}(\hat{x}) \in \mathbb{R}^{k_j}$ is inner network
$N_j$'s vector output and $A_{j} \in \mathbb{R}^{k_j \times H}$
is its attachment matrix --- the attachment carries the scope's
full width. The network output is $f_N(x) = \psi\big(W_2\,h(x) +
c\big)$,
where $\psi$ is the output de-standardization. A composite node is
thus an \emph{additive refinement}: the inner network, sited at
its node, corrects the scope's hidden state, at any
level, recursively.

\paragraph{Structure as a record of experience.} The height of the
inclusion tree is $\mathrm{depth}(N) = 1 + \max_{j\in J}
\mathrm{depth}(N_j)$ (with $\max\emptyset = 0$). Because depth exists
only where an inner network has been grown, the model's depth profile
is a record of \emph{where its experience demanded refinement} --- an
outcome of its life, not an architectural hyperparameter.
Figure~\ref{fig:substrate} sketches the structure.

\begin{figure}[t]
\centering
\begin{tikzpicture}[
  unit/.style={circle, draw=black!70, minimum size=5.5mm, inner sep=0pt},
  comp/.style={circle, draw=black!70, very thick, minimum size=7mm, inner sep=0pt},
  net/.style={rectangle, rounded corners=2pt, draw=black!60, dashed,
              inner sep=3pt},
  lab/.style={font=\scriptsize},
  >={Stealth[length=2mm]}]
\node[lab] (x) at (-2.1, 0) {$\hat{x}\in\mathbb{R}^d$};
\node[unit] (h1) at (0, 1.4) {};
\node[unit] (h2) at (0, 0.7) {};
\node[comp] (hj) at (0, -0.2) {};
\node[unit] (h4) at (0, -1.1) {};
\node[lab, left=1.5mm of hj] {$h_j$};
\node[unit] (y) at (2.0, 0) {};
\node[lab, right=1mm of y] {$f_N(x)$};
\foreach \h in {h1,h2,hj,h4}{ \draw[->] (x) -- (\h); \draw[->] (\h) -- (y); }
\node[net, minimum width=34mm, minimum height=26mm] (inner) at (0, -3.35) {};
\node[lab, anchor=north west] at (inner.north west) {inner network $N_j$};
\node[lab] (ix) at (-1.25, -3.45) {$\hat{x}$};
\node[unit, scale=.7] (i1) at (-0.15, -2.85) {};
\node[unit, scale=.7] (i2) at (-0.15, -3.45) {};
\node[comp, scale=.7] (i3) at (-0.15, -4.05) {};
\node[unit, scale=.7] (io) at (1.05, -3.45) {};
\foreach \i in {i1,i2,i3}{ \draw[->] (ix) -- (\i); \draw[->] (\i) -- (io); }
\draw[->, black!60] (inner.north) .. controls (0.8,-1.8) .. (hj.south east);
\node[lab, anchor=west] at (1.15, -1.9)
  {$h = \varphi(a) + g_j(\hat{x})A_j$};
\draw[->, black!50, dotted] (i3.south) -- ++(0,-0.42)
  node[lab, below] {\dots recursively};
\end{tikzpicture}
\caption{The recursive substrate (Definition~\ref{def:net}). A
composite node adds its inner network's output to its own
activation; inner networks read the same input and may themselves
contain composite nodes --- structure is grown, not designed.}
\label{fig:substrate}
\end{figure}
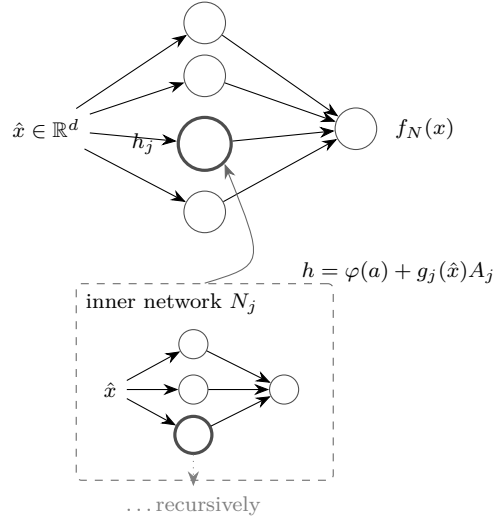

\paragraph{Standardization as governed state.} Input/output
scalers --- the textbook conditioning of the optimization problem~\cite{efficientbackprop} ---
are fit once at first training and thereafter refreshed only through
an exact compensation: when the target scale is re-estimated
$(\mu,\sigma) \to (\mu',\sigma')$, the output layer is remapped
$W_2 \leftarrow W_2\,(\sigma/\sigma')$,
$c \leftarrow (c\,\sigma + \mu - \mu')/\sigma'$, so the
end-to-end function is preserved exactly; the optimizer's (Adam~\cite{adam}) moment estimates, which live in the old scale, are
reset rather than carried across.

\paragraph{Substrate pluggability.} Definition~\ref{def:net} is the
\emph{reference} substrate --- the minimal body on which the
operators are stated and on which every adjudicated claim of
App.~\ref{sec:empirical} is judged. In the released system,
substrates are instantiated only through a registry that enforces
a single contract (training step, prediction, growth-site
enumeration, site refinement, height, persistence), and five hosts
are included in the release: the reference recursive network and a
variant; a transformer encoder~\cite{vaswani} implemented from
first principles --- attention and feed-forward blocks, with
explicitly implemented backpropagation verified against finite
differences, deterministic under fixed seed --- with a variant;
and a causal sequence host.
The contract is architecture-agnostic: any host that exposes it
serves under the same lifecycle, so a domain whose laws require
compositional depth receives it through the substrate choice
rather than through growth. On the transformer host the growth
sites are the feed-forward units of \emph{every} layer: an unstable unit
grows the same zero-initialized recursive inner network, trained
by the same cross-scale rule, and $\omega$ widens at every scale
--- lifecycle, gate, and audit unchanged; the released substrate
kit and a committed growth-audit log verify these mechanics. Two
remarks delimit the scope. Attention was untouched by growth
previously as a single-variable discipline;
\S\ref{sec:growatt} lifts this restriction. The flattening result of
Prop.~\ref{prop:rhostatus} (\S\ref{sec:additive}),
moreover, is a property of the reference substrate, where a refinement's
correction reaches the output through linear aggregation alone; on
the transformer host the same correction, injected at layer
$\ell$, passes through the nonlinearities of every subsequent
layer --- refinements are composed downstream, and the flattening
argument does not apply.

\subsection{One learning primitive at every level (every scale)}

Learning uses a single rule, applied identically at every scope of
the tree. Within a scale, it is ordinary gradient descent on that
scale's parameters, treating inner outputs as constants (the gradient
passes through the atomic branch only). Across scales, the enclosing
network does not send gradients \emph{through} its composite nodes; it
hands each one a \textbf{target} and lets the inner network
take its own step --- the same primitive, one level down; computational
mechanics arrived at the same coupling independently: in the
variational multiscale principle, fine-scale equations are driven
by the residual of the coarse~\cite{hughesvms}. This is a \emph{cross-scale
backpropagation} in the general sense of the term --- what
backpropagation has meant since its origin is the backward
propagation of \emph{errors}, i.e., of corrective information~\cite{rumelhart} --- but with a different transmission rule. The mechanical reading makes the
difference concrete. With the loss as a potential, the output error
acts as a generalized force, and a gradient step is a displacement
proportional to that force --- the overdamped limit of the
classical mechanical reading of optimization dynamics~\cite{polyak}. Classical backpropagation transmits the force
through the composition of Jacobians; here each scale hands its
substructure the correction arriving through that structure's own
attachment ($G_j = (\partial\mathcal{L}/\partial H)A_j^{\top}$),
and the substructure relaxes locally by the same rule ---
a pattern numerical mathematics knows as block relaxation for
nonlinear systems~\cite{ortega}. In
both schemes the same quantity flows backward --- corrective
information; the schemes differ in the transmission rule. Nor is
the gradient itself essential: it is one encoding of the
correction among several, so the backward pass is defined by what it must
deliver (a correction), not by how the correction is computed. The consequence is dynamical, not
representational: a flattened network computing the \emph{identical}
function takes a \emph{different} step under the same data, so the
inclusion tree acts on the learning trajectory even where it is
invisible in the instantaneous function
(Prop.~\ref{prop:rhostatus}).

\begin{algorithm}[t]
\caption{Coupled training step at scope $N$ (recursive)}
\label{alg:primitive}
\begin{algorithmic}[1]
\Require batch $(X, y)$, standardized $(\hat{X}, \hat{y})$,
        $n = |X|$;\; $A = \hat{X}W_1^{\top} + b_1$,
        $H = h(\hat{X})$ (Def.~\ref{def:net})
\State $e \gets (HW_2^{\top} + c) - \hat{y}$;\quad
       $\mathcal{L} \gets \tfrac{1}{n}\lVert e\rVert^2$
\State \textbf{within-scale step:}
       $D \gets (e\,W_2) \odot \varphi'(A)$;\;
       update $W_1, b_1, W_2, c$ from
       $g_{W_1} = D^{\top}\hat{X}/n$,\;
       $g_{b_1} = \bar{D}$,\;
       $g_{W_2} = e^{\top}H/n$,\;
       $g_c = \bar{e}$
       \Comment{inner outputs held constant; the loss's constant
       factor 2 is absorbed into the step size; scopes carrying
       composition blocks extend this step through the block chain
       (\S\ref{sec:compositional})}
\For{each composite node $j \in J$} \Comment{cross-scale correction handing}
  \State $G_j \gets (\partial\mathcal{L}/\partial H)\,A_j^{\top}$
         \Comment{the correction reaching $N_j$ through its
         attachment}
  \State update $A_j$ from
         $g_{A_j} = g_j(\hat{X})^{\top}\,
         \partial\mathcal{L}/\partial H$
  \State $N_j.\Call{TrainStep}{X,\, G_j}$
         \Comment{\emph{this same algorithm}, one step, recursing
         into $N_j$'s own composite nodes}
\EndFor
\end{algorithmic}
\end{algorithm}

The rule is \emph{scale-invariant by construction}: no scope is
special, no scope is frozen, and the cross-scale signal is a target
rather than a gradient --- each scale is trained as a small
supervised learner on the residual its parent asks it to explain.
Corollary (ii) of the axiom (every scale receives effective
optimization pressure) is thus discharged mechanically: pressure arrives at level $k{+}1$ whenever level $k$ has residual error at a
composite node. One boundary is
named rather than hidden: the discharge presumes the enclosing
exit weights remain nonzero --- a composite node whose exit
coefficient vanishes starves its interior while remaining
formally trainable; the corollary is mechanical wherever that
does not occur, and the case itself is left as a named
condition, not silently assumed away.

\subsection{Where to grow: the instability signal}
\label{sec:signals}

Growth is aimed by a per-unit statistic computed from the updates
themselves. Let $\Delta_j^{(t)}$ be the update to unit $j$'s input
weights at step $t$, and let $\mu_j, \nu_j$ be exponential moving
averages (decay $0.95$) of $\Delta_j$ and $|\Delta_j|$ respectively.
The \textbf{instability} of unit $j$ is
\begin{equation}
u_j \;=\; 1 - \frac{\lVert \mu_j \rVert}{\lVert \nu_j \rVert +
\varepsilon} \;\in\; [0,1].
\end{equation}
We state explicitly what is and is not claimed for this rule:
no descent or convergence property is asserted; its
characterization in this paper is empirical --- the acceptance
suite pins its trajectories bit-exactly, the finite-difference
checks verify its gradients, and prediction T1-i records its
first behavioral signature.

Steady drift (updates agree in sign) gives
$\lVert\mu_j\rVert \approx \lVert\nu_j\rVert$, hence $u_j \to 0$;
oscillation (updates cancel) gives $u_j \to 1$: the unit is being
pulled in contradictory directions by data it lacks the capacity to
reconcile --- read as the signature of a site that wants
refinement (an interpretation; its operational content is the
trigger record it drives). The
instability report ranks all units across all scopes; it
\emph{suggests} sites. A second instrument, the \textbf{trajectory
verdict}, classifies apparent progress as REAL, FALSE\_SPIKE,
FALSE\_SWAP, or STUCK, with one binding rule: \emph{never advance on
a FALSE verdict}; STUCK is the designed trigger for remediation or
structural speculation. These instruments suggest; only the gate
\emph{decides}.

\subsection{Direction-neutral operators: the interface and
re-founding}
\label{sec:neutralops}

Two operators of the algebra act along axes orthogonal to the
direction question of \S\ref{sec:direction} --- they serve
every growth regime identically and are therefore stated here,
with the shared machinery.

\paragraph{$\sigma$ (interface --- the input axis).} Applied
recursively to every scope (all scopes share the input space): append
a zero column to each $W_1$, extend the input standardization with
$(\mu_{\mathrm{new}}, \sigma_{\mathrm{new}}) = (v, 1)$ where $v$ is
the backfill default for stored rows.

\begin{proposition}[$\sigma$-exactness]\label{prop:sigma}
For any value $\xi$ of the new coordinate --- not merely the backfill
--- the forward map is unchanged: the new column of every $W_1$ is
zero, so $\xi$'s standardized value is annihilated before it reaches
any activation.
\end{proposition}
\noindent (Empirically exact --- App.~\ref{sec:empirical}.)
Participation is therefore \emph{earned
by training}, never granted by arrival --- and $\sigma$ is the
operator answer to the interface instance of lock-in
(\S\ref{sec:lockin}), the case no interior mechanism can compensate.

\paragraph{$\Phi$ (re-found --- the escape axis).} When the
foundation itself is outgrown, refinement only bloats fine scales
(each patch corrects a structure that is itself wrong). $\Phi$ builds
a \emph{candidate}
\begin{equation}
C \;=\; \mathrm{TrainFromScratch}\big(S_t;\;
H_{\mathrm{birth}} = \beta(|\mathrm{schema}_t|,\, |S_t|)\big),
\end{equation}
where $S_t$ is the model's accumulated experience store of
\textbf{real rows only} --- never the old model's own outputs, which
would inherit its errors --- and the birth-sizing rule $\beta$ is
re-run on the \emph{current} schema, never on batch one (the direct
repair of first-batch lock-in). In the released system $\beta$ is
the heuristic $\max(16, \min(64,\, 4(n_{\mathrm{in}} +
n_{\mathrm{out}})))$ --- a starting capacity only; structure
thereafter is determined by data through the operators, under the
gate, not by formula. Adoption is
solely by the gate (Algorithm~\ref{alg:commit}): the old self serves
until a candidate scores higher and remains in the lineage for rollback --- identity
resides in the record, not in the weight matrix. The pre-registered
trigger monitors the \textbf{amplitude-hierarchy inversion}
\begin{equation}
R_{\mathrm{inv}} \;=\;
\frac{\lVert W_2\, h^{\mathrm{fine}} \rVert}
     {\lVert W_2\, h^{\mathrm{coarse}} \rVert + \varepsilon},
\end{equation}
the output-projected mass of all inner-network corrections relative
to the GELU backbone at the top split: a healthy decomposition has
$R_{\mathrm{inv}} \ll 1$; sustained $R_{\mathrm{inv}} \ge 0.5$ (three
consecutive probes), repeated disposed growths under saturation, or
flat learning progress under high demand raise the re-founding
proposal. (App.~\ref{sec:eseries} reports that in our tested
scenarios this trigger's \emph{necessity} claim was not supported ---
takeover safety, by contrast, is unconditional, because it is the
gate's.)

\subsection{Governance formalized}
\label{sec:gate}

\begin{definition}[states and gate]
\label{def:gate}
A model carries a speculative \textbf{working state} $w$ and a
\textbf{committed version} $s$ (the one that serves), plus a
time-stamped \textbf{held-out stream}
$\mathcal{H} = \{(x_i, y_i, \tau_i)\}$ quarantined from all training.
Let $R_M(\mathcal{H})$ be its most recent $M$ examples and
$Q(\cdot\,; R)$ the held-out metric on slice $R$ (in the released system, the factory and full-system servers'
gates both score the fraction of held-out answers matched exactly
--- numeric within $\pm 0.5$).
\end{definition}

\begin{algorithm}[t]
\caption{Commit (the reality gate), recency parameter $M$}
\label{alg:commit}
\begin{algorithmic}[1]
\State $q_s \gets Q(s;\, R_M(\mathcal{H}))$
       \Comment{re-score the incumbent --- fresh, never cached}
\State $q_w \gets Q(w;\, R_M(\mathcal{H}))$
\If{$q_w > q_s$} \Comment{strict}
  \State $s \gets w$; append lineage event with both scores
  \Else
  \State refuse; $s$ unchanged \Comment{a logged outcome, not an error}
\EndIf
\end{algorithmic}
\end{algorithm}

Teaching composes the same gate with from-scratch candidate training:
$\mathrm{teach}(E;\, W, M)$ trains a candidate on the most recent $W$
stored examples plus $E$ and commits with recency $M$ --- the two
recency dials being what adaptation under drift requires (train on
the reality that still holds; be judged on the reality that holds
now). Budgets close the loop on structure: an operator application
that would exceed the parameter cap $P_{\max} = P_0 \cdot m$ (or the
depth cap) is \emph{refused and logged} before it touches the working
state.

The gate-soundness invariant of \S\ref{sec:governance} now reads:
along the served lineage $s_0, s_1, \dots$, every promotion $k$
satisfied $Q(s_{k+1}; R^{(k)}) > Q(s_k; R^{(k)})$ on the slice
$R^{(k)}$ used to judge it --- and refusals left the lineage
untouched. Learning steps, $\omega$, $\rho$, $\delta$, $\sigma$, and
$\Phi$ all enter the served model through this one gate; that is the safety argument's core --- adoption-level protection, uniform across operator types (its named boundaries: the reuse and provenance exposures named below, \S\ref{sec:gate}). Real traces --- including
refusals --- are shown in Appendices~\ref{app:validation}
and~\ref{app:topo}.

Two exposures lie outside the invariant's scope and should be
named. \emph{Selection reuse:} quarantine protects the held-out
stream from training, not from \emph{selection} --- under many
commit attempts against the same small slice, a strict
$>$ with no margin will eventually promote a candidate whose margin is
noise alone (the reusable-holdout phenomenon;
\S\ref{sec:discussion}), and the
invariant will
truthfully record the promotion as evidence-sanctioned. Promotion
margins and per-slice query budgets are the natural remedies ---
and by the architecture's own layering they are the
\emph{operator's} policy, not the mechanism's: the released
interfaces already carry everything their enforcement needs
(public pre-commit scoring for margins, every adjudication in
the append-only ledger for query counting, holdout replenishment
on demand for rotation). The mechanism supplies the bookkeeping;
the examination strategy belongs to the caller. \emph{Evidence provenance:} the invariant certifies
candidates against whatever the held-out stream contains; if that
stream is corrupted, every guarantee downstream of it is vacuous
(\S\ref{sec:llm}).

\section{Growth: The Direction Question}
\label{sec:direction}

\noindent\emph{This section in one arc: the algebra offers five
operators (\S\ref{sec:operators}); the additive three are exactly
characterized by Prop.~\ref{prop:rhostatus}
(\S\ref{sec:additive}); $\delta$ supplies the compositional axis
that proposition does not reach (\S\ref{sec:compositional}); and
a three-tier policy --- predict from history, verify by probes,
adopt by the gate --- chooses the direction per scope, with a
system-level mode restoring the purely additive regime
(\S\ref{sec:mode}).}

The substrate of \S\ref{sec:algebra} can grow in more than
one direction. This section presents the operator algebra's two
growth directions as two realizations of the framework: the
\emph{additive} realization --- multiscale superposition, the
regime analyzed in full below --- and the \emph{compositional}
realization, which augments it with scope-interior composition and
a signal-based selection of the direction itself
(\S\ref{sec:compositional}). The two are one released system: the
additive realization is exactly the adaptive system restricted to
its widen-only mode (\S\ref{sec:mode}).

The organizing ideas of this section are classical across the
mathematical sciences, developed for a nature that is rough and
multiscale~\cite{mandelbrot}: numerical mathematics builds
solutions scale by scale (multigrid~\cite{brandt}) and spends
resolution only where an indicator demands it (adaptive mesh
refinement~\cite{amr}); signal processing expands a function
scale by scale (multiresolution analysis~\cite{mallat}); and
physics stacks effective descriptions scale by scale (the
renormalization group~\cite{wilson}). Notably, these traditions
do not commit a scale hierarchy to any particular local form: in
a multiresolution expansion the scales \emph{add}, whereas in the
renormalization group each effective description is a
\emph{transformation} of the one below --- scale structure
without additivity. The two realizations that follow are these
two classical forms inside one learner: an additive expansion
grown term by term, and a composed hierarchy grown transformation
by transformation, with the choice between them read from the
data.

\subsection{The operator algebra}
\label{sec:operators}

\begin{figure}[t]
\centering
\includegraphics[width=\linewidth]{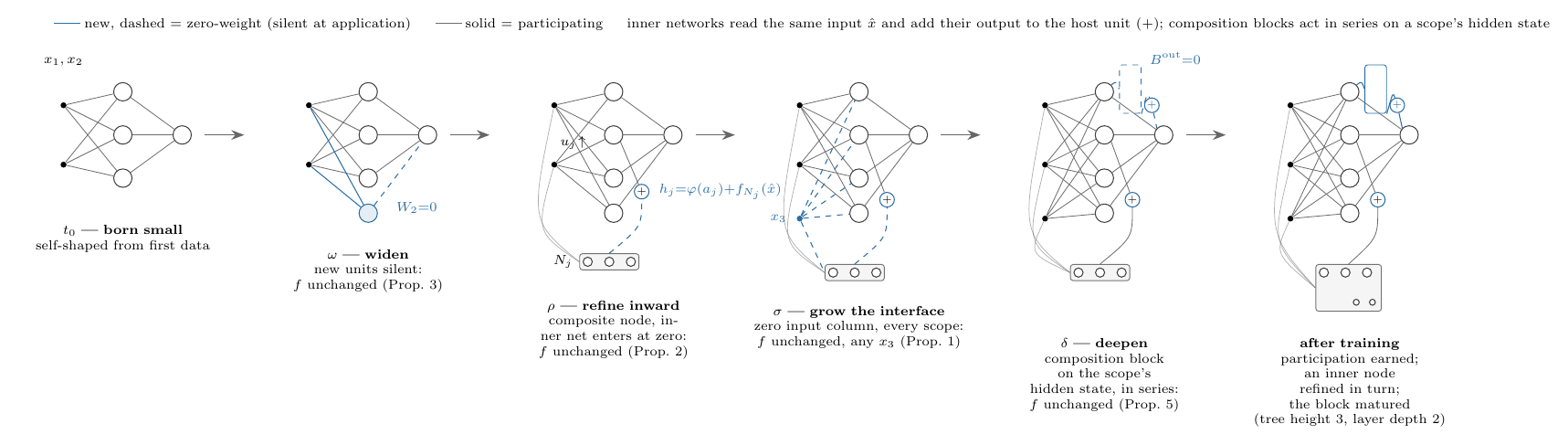}
\caption{Schematic (not data): the operator algebra growing a
topology over time. New elements (accent) enter with zero-weight
dashed connections --- silent at application, so $f$ is unchanged
(Props.~1--3, 5) --- and earn participation through training. An
inner network reads the \emph{same} input $\hat{x}$ and adds its
vector output to the scope's hidden state through its zero-born
attachment at the $+$ junction
($h = \varphi(a) + g_j(\hat{x})A_j$); the $\delta$ panel
appends a
zero-initialized composition block in series on the scope's
hidden state (Prop.~\ref{prop:delta}). The final panel shows an inner node
refined in turn and the block matured (tree height 3, layer
depth 2). Every
parameter remains trainable from step one; adoption is always by
the gate.}
\label{fig:theorygrowth}
\end{figure}

Each in-place operator satisfies one governance contract
($\Phi$ builds a candidate; its takeover is gated):
\textbf{exact at application, every new parameter trainable from
step one, budgeted, logged} (Table~\ref{tab:operators}).

\begin{table}[t]
\centering\small
\caption{The operator algebra. Every operator is budgeted
(parameter/depth caps as policy; refusals are logged outcomes),
audited (lineage events), and closed under exact removal (each
grown structure carries its inverse); adoption is always by the
gate of Algorithm~\ref{alg:commit}. The additive family
($\omega$, $\rho$, $\sigma$, per-head growth, skips) leaves
the composition-degree invariant unchanged; $\delta$ alone
raises it --- the boundary that separates multiscale
organization from true deepening. The loop operator $\lambda$
(\S\ref{sec:loops}) and the per-head attention operators
(\S\ref{sec:growatt}) join these five under the identical
contract --- the abstract's operator list spans all of them.
Measured exactness values:
App.~\ref{sec:empirical}.}
\label{tab:operators}
\begin{tabular}{@{}llll@{}}
\toprule
Operator & Axis & Exactness at application & Typical trigger \\
\midrule
$\rho$ (refine) & hierarchy & inner net enters at zero (Prop.~\ref{prop:rho}) & high $u_j$ at a node \\
$\omega$ (widen) & width & zero output-weight block (Prop.~\ref{prop:omega}) & scope-wide saturation \\
$\sigma$ (interface) & inputs & zero input column, any value (Prop.~\ref{prop:sigma}) & new observable arrives \\
$\delta$ (deepen) & composition & block enters at zero (Prop.~\ref{prop:delta}) & delayed-payoff slope signal (\S\ref{sec:compositional}) \\
$\Phi$ (re-found) & whole model & candidate; takeover gated & inversion / disposal / flat LP \\
\bottomrule
\end{tabular}
\end{table}

\paragraph{The grown body is a contract, not a class.} The
requirements that $\rho$ places on the body it grows at a node
are exhausted by the governance contract: exact entry (the body's
contribution is identically zero at application), every parameter
trainable from step one, countable parameters, and removability.
Nothing in the exactness argument inspects the body's interior
--- the zero-entry proof never opens the box --- so the reference
recursive network is the \emph{default} body, not a commitment:
any function approximator satisfying the contract may be grown at
a node, selected per event by policy. The propositions of
\S\ref{sec:theory} apply to every such body unchanged, and the
scale hierarchy below (\S\ref{sec:scalehierarchy}) governs all of
them alike.

\subsection{The additive direction: refinement and widening}
\label{sec:additive}

\paragraph{The three axes are not alike.} Widening ($\omega$)
adds quantity at an existing scale; interface growth ($\sigma$)
adds observables; refinement ($\rho$) nests a finer description
inside the one node whose data demand it --- an organizational
axis, not an expressive one (Prop.~\ref{prop:rhostatus}), and
structurally local (the refinement attaches to one node; in input
space every unit remains global). The instability signal is a
this-scale-fails detector, and refinement is closed under itself,
so the ladder coarse $\to$ fine $\to$ finer has no end in the
algebra (budgets are policy, \S\ref{sec:gate}): structure is
grown, not designed. The resulting topology is multiscale in the established sense
of numerical mathematics --- resolution added scale by scale, and
only where an indicator demands it~\cite{mallat,amr}; the label
describes how information is \emph{organized}, and is orthogonal
to expressive enlargement (Prop.~\ref{prop:rhostatus}) --- the
canonical multiscale methods are themselves additive and
depthless. The idea has a precise kin in the multiscale
finite element method~\cite{houwu,houwucai}: an element's shape
function ceasing to be an analytic design and becoming a
\emph{computed object} interior to a coarse cell.

\paragraph{$\rho$ (refine --- the hierarchy axis).} At scope $N$, node
$j \notin J$:
\begin{equation}
\rho_j(N):\quad \mathcal{I}(j) \leftarrow N' \text{ fresh, with }
A_{j} = 0 \text{ at attach time.}
\end{equation}

\begin{proposition}[$\rho$-exactness]\label{prop:rho}
$f_{\rho_j(N)} = f_N$ pointwise.
\end{proposition}
\begin{proof}
The only change to the forward map Eq.~\eqref{eq:forward} is the
additive term $g_{j}(\hat{x})\,A_{j}$; with $A_{j} = 0$ at
attach time the term is identically zero for every input.
\end{proof}
\noindent (Empirically exact --- App.~\ref{sec:empirical}.) By closure
(Definition~\ref{def:net}), $\rho$ applies at any level --- including to units created by $\omega$.

\paragraph{$\omega$ (widen --- the capacity axis).} At any scope,
append $k$ units:
\begin{equation}
W_1 \leftarrow \begin{bmatrix} W_1 \\ W_1^{\mathrm{new}}
\end{bmatrix},\qquad
b_1 \leftarrow \begin{bmatrix} b_1 \\ 0 \end{bmatrix},\qquad
W_2 \leftarrow \begin{bmatrix} W_2 & 0_{1\times k} \end{bmatrix},
\end{equation}
with $W_1^{\mathrm{new}}$ randomly initialized.

\begin{proposition}[$\omega$-exactness]\label{prop:omega}
$f$ is unchanged: the new units' contribution enters the output only
through the appended zero block of $W_2$.
\end{proposition}
\noindent (Empirically exact to floating-point roundoff ---
App.~\ref{sec:empirical}.) Two governance details do real work here.
Optimizer slots for the new parameters are extended with zero moments
--- the new capacity is trainable \emph{from step one} (the axiom,
mechanized; contrast masking schemes). The instability
EMAs of new units, moreover, are warm-started at the scope's
current row-mean, so
newborn units compete fairly in the growth rankings instead of being
spuriously ranked coldest or hottest.

\begin{proposition}[Representational status of $\rho$]\label{prop:rhostatus}
In the absence of composition blocks, and for matched
atomic-unit counts,
$\mathcal{F}(G_{\rho}) = \mathcal{F}(G_{\omega})$. Every inner
network reads the same input $\hat{x}$ and contributes additively,
so the substrate at any height of
the inclusion tree evaluates to a weighted sum
of $\varphi(\mathrm{affine}(\hat{x}))$ terms --- the function
class of a single-hidden-layer network of the same total width
($\subseteq$); conversely, any flat network of that width is
realized by a nested topology whose inner output weights carry
the flat coefficients ($\supseteq$). $\rho$ does not enlarge the
reachable class beyond $\omega$.
\end{proposition}

\noindent What $\rho$ adds is therefore not representation but
\emph{organization}: a refinement trains its inner network on the
correction reaching it through its attachment
(Algorithm~\ref{alg:primitive}) --- a localized, stagewise-additive credit assignment --- and
gives growth a site,
a trigger ($u_j$), and an audit unit.  Extensional equivalence, however, does not imply
\emph{process} equivalence: substrates with identical function
classes can differ in the trajectories a lifetime of local, governed
growth can traverse --- and whether they differ here is an
empirical question, with T1-i
(App.~\ref{sec:empirical}) as first evidence.

\subsection{The compositional direction and its selection}
\label{sec:compositional}

\paragraph{The operator $\delta$ (deepen --- the composition
axis).} At any scope with hidden state $H_0$ (the base activations
plus inner-network contributions), $\delta$ appends one
zero-initialized residual composition block (the residual form
is the standard tool of deep architectures~\cite{he2016resnet}):
\begin{equation}
H_k \;=\; H_{k-1} \;+\; \mathrm{gelu}\!\big(H_{k-1} B^{\mathrm{in}\top}_k
+ b_k\big)\, B^{\mathrm{out}\top}_k, \qquad
B^{\mathrm{out}}_k = 0 \text{ at application},
\label{eq:delta}
\end{equation}
with the scope's head reading the final $H_L$. The block applies
a nonlinearity to a state that already contains nonlinear unit
outputs: the result is $\varphi(\mathrm{mix\ of}\
\varphi(\cdot))$ cross-terms, which leave the additive class
$\sum_i \alpha_i\,\varphi(\mathrm{affine}(\hat{x}))$ in general
--- the separation between the classes is the subject of the
classical depth-separation results (\S\ref{sec:discussion}) --- so
composition is genuine, and the
collapse of Prop.~\ref{prop:rhostatus} does not apply to
scopes carrying blocks.

\begin{proposition}[$\delta$-exactness]\label{prop:delta}
Appending a block leaves the computed function unchanged: with
$B^{\mathrm{out}}_L = 0$ the new block's contribution is
identically zero, so the chain's output equals its pre-append
value, $H_L = H_{L-1}$, pointwise at the moment of application
(for the first block, $H_1 = H_0$).
\end{proposition}
\noindent The released suite verifies this bitwise at root and
inner scopes, and the in-service insertion campaign verifies it
at serving time --- bitwise at every pure-$\delta$ instant;
whole acts that include an aspect-law widen rider preserve the
function to machine epsilon
(\S\ref{sec:campaigns-results}).

Four properties place $\delta$ inside the governance discipline.
(i) \emph{Positional and service-safe}: blocks live in the
scope's composition chain and may be appended or INSERTED at any
position of that chain --- including in a LIVE serving model ---
with the same zero-entry exactness at the instant of application
(the identity-calibrated whole-layer form serves the attention
hosts through the same verb); the in-service insertion campaign
measures this exactness directly
(\S\ref{sec:campaigns-results}). (ii) \emph{Recursive and unbounded}: any scope at any level
of the inclusion tree may deepen, repeatedly; we call
$1+\text{(blocks)}$ the scope's \emph{layer depth}, distinct
from the height of the inclusion tree. (iii) \emph{Additive
aggregation preserved}: attachment across nodes remains a sum, so
the per-scale audit quantities of the additive direction survive
unchanged. (iv) \emph{Removable and budgeted}: blocks are audited
ledger events, removable exactly, and capped by policy
(\texttt{max\_blocks}), with refusals logged. Training extends the
within-scale chain rule through the block chain back to $H_0$; the
cross-scale corrections of
Algorithm~\ref{alg:primitive} are formed at the $H_0$ slot with
their original expression --- the derivative
$\partial\mathcal{L}/\partial H$ computed through the block
chain, matching the released implementation --- and all block
gradients are verified
against central finite differences in the released suite.

\begin{figure}[t]
\centering
\includegraphics[width=0.98\linewidth]{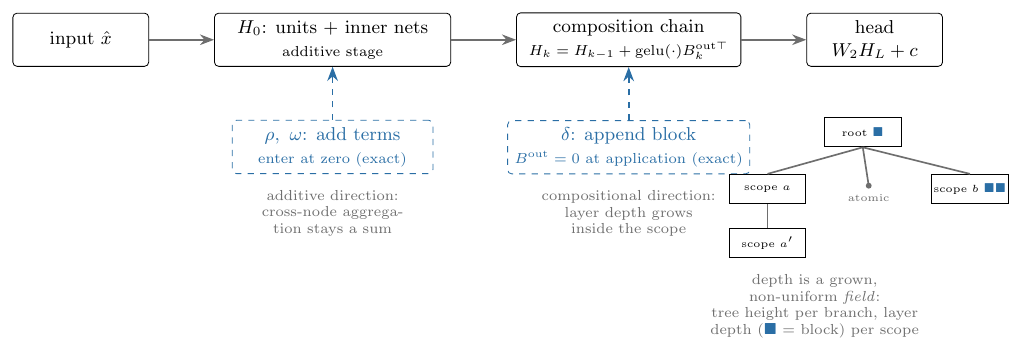}
\caption{The two growth directions on one pipeline. Additive
operators ($\rho$, $\omega$) add terms to the scope's hidden state
$H_0$; the compositional operator ($\delta$) appends
zero-initialized blocks after it. Both directions enter exactly
(dashed: silent at application); cross-node aggregation remains a
sum, and layer depth grows only inside the scope. Inset: depth
is a grown, non-uniform \emph{field} --- each branch evolves its
tree height and each scope its layer depth independently.}
\label{fig:twodir}
\end{figure}

\paragraph{Choosing the direction.} Which direction a scope
should grow is not knowable by fiat: the two directions answer two
different kinds of target. Greedy term addition against the current residual carries
dimension-independent approximation-rate guarantees~\cite{jones,barron} on the classes
served by superposition; compositional targets, by contrast, admit
exponential shallow-versus-deep separations~\cite{telgarsky,eldanshamir} (the fuller treatment of depth is
in \S\ref{sec:discussion}). Which jurisdiction a scope's
residual inhabits must therefore be \emph{read from the scope's
own signals}, and the reading proceeds in three tiers ---
prediction first, probes only to verify, adoption always by the
gate (Algorithm~\ref{alg:direction}).

\emph{Tier 1: history prediction.} Each scope maintains a gain
ledger (the realized, fixed-horizon relative energy reduction of
every past growth event) and a dense residual-energy series. The
energy series is extrapolated to its asymptote with a
BIC-weighted~\cite{bic} ensemble of parametric curve families~\cite{domhan}. Tier~1 may act only
under a \emph{predictability certificate} --- a gate this design
defines, whose four checks are each filled by a deliberately
standard instrument sitting in a replaceable slot: forecastability by normalized spectral
entropy~\cite{foreca}; regime validity by online changepoint
detection~\cite{bocpd}; demonstrated skill of the extrapolator on
the scope's own recent history by rolling-origin backtest~\cite{tashman} --- a held-out gate for the forecaster itself ---
and a confidence bound on the asymptote's posterior interval.

\emph{Tier 2: probe verification.} When the certificate fails or
history is insufficient, both directions are priced by
zero-attached probes trained on the scope's recent window, in the
same units (the scope's window energy). The additive probe's
gradient at zero attachment is exactly the unit--residual
correlation. How a probe is READ is itself an empirical question,
and the measured answer replaced this design's first draft:
single-horizon readings --- the raw gain and the gain per
parameter alike --- carry no axis information (early in training
every candidate prices positive, and the two directions price
alike; measured in the direction campaigns,
\S\ref{sec:campaigns-results}). What does carry the axis is the
\emph{two-horizon slope}: each candidate is probed at a short and
a long budget, and the difference (long-horizon gain minus
short-horizon gain) reads the delayed-payoff signature that
distinguishes composition --- additive gains are front-loaded,
compositional gains are back-loaded, and the slope measures
exactly this asymmetry. In Algorithm~\ref{alg:direction},
\emph{stall} denotes the certificate's no-useful-gain verdict on
the additive history and \emph{slope} the delayed-payoff
signature; both thresholds are policy-level in the released
code.

\emph{Tier 3: the gate} (the full policy flow is drawn in
Figure~\ref{fig:policy}). Either proposal is adopted exactly as
every other structural change: by scoring higher on held-out
reality (Algorithm~\ref{alg:commit}); and no adoption is final, since
drift re-opens every comparison. Every algorithmic role above ---
extrapolator, forecastability, changepoint, backtest, pricer,
combiner --- is a replaceable part behind a
registry, selected by policy and never switched by the machinery
itself.

\begin{algorithm}[t]
\caption{Direction selection at a scope (predict, verify, adopt)}
\label{alg:direction}
\begin{algorithmic}[1]
\Require scope with gain ledger, energy series, recent window;
policy thresholds
\If{history below the minimum-data guards}
  \State \textbf{cold start:} probes if the window suffices,
  else additive growth (the default direction)
\ElsIf{certificate passes (forecastability, changepoint,
backtest, confidence)}
  \If{no stall and extrapolated additive gain useful}
    \State \textbf{propose additive growth} ($\rho$ at the most
    unstable non-composite node; $\omega$ when all are composite)
  \ElsIf{stall $\wedge$ energy high}
    \State \textbf{propose deepen} when the two-horizon slope
    favors it (delayed-payoff signature)
  \Else{} \State fall through to probes
  \EndIf
\Else{} \State \textbf{probes:} price both candidates at two
  horizons; select by slope among positive gains; ties favor
  the additive direction
\EndIf
\State \textbf{adopt} only by the gate
(Algorithm~\ref{alg:commit}); log the decision record (signals,
certificate values, prices, part names, policy snapshot)
\end{algorithmic}
\end{algorithm}

\begin{figure}[t]
\centering
\includegraphics[width=0.8\linewidth]{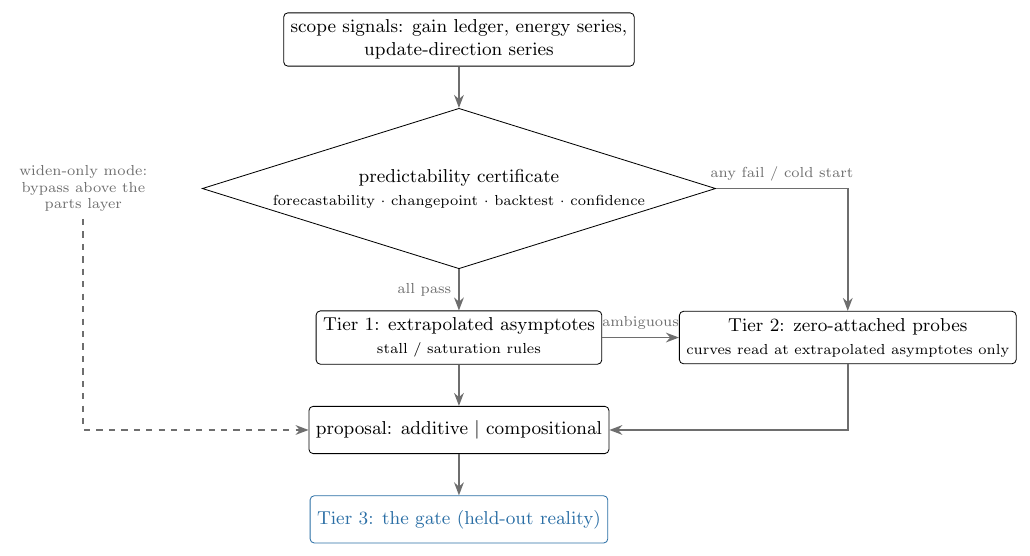}
\caption{Direction selection: predict first (under the
predictability certificate), probe only to verify (asymptote-read
curves), adopt only by the gate. The probe READING is the measured
part: single-horizon prices carry no axis signal; the
two-horizon slope does (\S\ref{sec:campaigns-results}).
The widen-only mode bypasses the
machinery above the replaceable-parts layer
(\S\ref{sec:mode}).}
\label{fig:policy}
\end{figure}

\subsection{Regime containment and the growth mode}
\label{sec:mode}

The two realizations are one released system. The additive
realization is \emph{exactly} the adaptive system restricted to
its widen-only mode: a system-level control with two named values
selects between adaptive two-direction growth (the default) and
the purely additive regime (the code-level mode name; the
regime includes both additive operators, $\rho$ and $\omega$), in
which the machinery never invokes $\delta$, no certificate or
probes are computed, and Prop.~\ref{prop:rhostatus} applies
in full --- the proposition is the exact characterization of the
widen-only mode. These are two governed paradigms, not a weak and
a strong system: the additive regime is the multiscale discipline
of the classical numerical methods --- per-scale audit quantities
well-defined, perturbations bounded --- while the adaptive regime
extends the axiom to the compositional axis, adoption still
earned per site. The control is enforced above the
replaceable-parts layer, so no part substitution can bypass it;
it is a callable system interface in the released code (policy
level; tool-surface exposure is future work); and a
mode flip never alters an existing model's function --- the
forward computation is determined by the model's \emph{state}
(its blocks), never by policy, so switching modes changes what may
grow next, not what the model is. The active mode travels in
every decision record.

\subsection{The scale hierarchy: subordination as a growth
invariant}
\label{sec:scalehierarchy}

The multiscale reading of refinement carries an obligation the
operator algebra must enforce, not merely suggest: a grown inner
body is a \emph{fine-scale correction}, and a correction is
subordinate to the field it corrects. In the numerical treatment
of multiscale physics the fine model lives strictly inside one
coarse element and never approaches the coarse field's mass;
the
transported invariant here is quantitative. First,
\emph{subordination in size}: the host scope's own parameter mass
must dominate the body's by a large factor (a policy ratio, with
a default two orders of magnitude), so that no single fine-scale
correction can carry a load comparable to its host. Second,
\emph{an absolute floor}: a host below a minimum trained mass does
not refine at all --- growing corrections onto a coarse field that
is itself negligible inverts the hierarchy. Third,
\emph{subordination in time}: topology changes only after the base
has taken shape --- a minimum of trained steps precedes the first
structural event, and each event is small relative to the trained
mass it joins. All three are policy quantities with logged
violations (warn or refuse, selectable); none is a property the
model could drift out of silently. The invariant also bounds what
any one mechanism of \S\ref{sec:loops} can do: a local loop lives
on a scope whose share of the whole is already governed.

\section{Local Solving Loops: Computation That Solves Inside
Itself}
\label{sec:loops}

\subsection{The principle}
\label{sec:loops-principle}

In the numerical treatment of nonlinear continua, the global
solve routinely embeds \emph{local} solves. At every integration
point of a plastically deforming body, a small constitutive
system is satisfied by its own inner iteration --- the return
mapping, a local Newton loop --- invisible to the global
equilibrium iteration that contains it, convergent inside every
global step, with its own tolerance and its own budget. The
global method does not become fixed-point-free because the local
one iterates; the two live at different scales of the same
computation, and the discipline that keeps them compatible is
locality: the inner loop solves only its own point's equations,
touches only its own point's state, and returns control having
converged or having spent its budget.

This section transports that structure to a growing model. A
scope of the substrate may embed a \emph{bounded local solving
loop}: a computation with its own objective, its own convergence
guarantee, its own budget, running strictly inside one scope,
leaving the global learning contract --- gradient flow, gate
adjudication, audit --- untouched. Two quantities can be locally
solved. The first is the scope's \emph{state}: the forward pass
itself may relax to a fixed point where the task demands
iteration (\S\ref{sec:loops-lambda}). The second is a newborn
structure's \emph{weights}: a freshly grown body may refine
itself against a local objective before it participates in the
whole (\S\ref{sec:loops-spu}). The two realizations are
independent mechanisms --- separately implemented, separately
enabled, separately governed --- of one methodological move, and
both inherit the multiscale lineage of this work: they are the
transported forms of solving locally what is local.

\subsection{Realization I: grown cycles --- the loop operator
(\texorpdfstring{$\lambda$}{lambda})}
\label{sec:loops-lambda}

Every operator of \S\ref{sec:operators} grows an acyclic
structure: signals flow one way, once. Where a law's natural
computation is iterative --- its solution defined by an equation
rather than by an expression --- an acyclic pass of fixed depth
must approximate the \emph{result} of an iteration without
performing one. The loop operator makes iteration itself
growable.

\paragraph{Definition.} A $\lambda$-block on a scope is a triple
$(L_{\mathrm{in}} \in \mathbb{R}^{m\times H},\ b_\lambda \in
\mathbb{R}^{m},\ L_{\mathrm{out}} \in \mathbb{R}^{H\times m})$
attached at the end of the scope's composition chain. Where the
chain's output was $H_L$, the head now reads the fixed point
$z^{*}$ of
\begin{equation}
z \;=\; H_L \;+\; \varphi\!\left(z L_{\mathrm{in}}^{\top} +
b_\lambda\right) L_{\mathrm{out}}^{\top},
\label{eq:lambda}
\end{equation}
computed by relaxation from $z^{(0)} = H_L$ and stopped at a
tolerance or at a hard iteration cap $K_{\max}$; reaching the cap
returns the last iterate --- bounded compute is the contract, not
an exception. The scope's computation is thereby no longer a
traversal but a \emph{convergence}: the cycle is the first
structure in the algebra whose forward semantics is an equation
--- the fixed-point-layer form of computation introduced by
deep equilibrium models~\cite{deq}, employed here as a mature
tool.

\paragraph{Properties.} Four properties make the cycle a
governed operator rather than a hazard. \emph{(i) Well-posedness
under an enforced certificate.} If the loop map is a contraction
--- guaranteed when $c_\varphi\,\sigma_{\max}(L_{\mathrm{in}})\,
\sigma_{\max}(L_{\mathrm{out}}) \le \rho_{\max} < 1$, with
$c_\varphi$ the activation's Lipschitz constant --- the fixed
point of Eq.~\eqref{eq:lambda} exists, is unique, and the
relaxation converges to it geometrically (Banach's fixed-point
theorem~\cite{banach1922}; the certificate is precisely the
theorem). The inequality is not an assumption but an
\emph{enforced invariant}: after every training step on a looped
scope the bound is recomputed, and a violating
$L_{\mathrm{out}}$ is rescaled to restore it exactly ---
a projection, counted and auditable, never silent. The
certificate triple (cap, tolerance, budget) is validated for
mutual feasibility ($\rho_{\max}^{K_{\max}} \le$ tolerance) at
application --- the budget then suffices to reach the tolerance
from unit-scale initial gaps, the geometric rate covering larger
gaps with logarithmic surcharge. \emph{(ii) Exactness at application.} The operator
enters with $L_{\mathrm{out}} = 0$: the first iterate returns
$H_L$ unchanged and the relaxation terminates immediately, so
$z^{*} \equiv H_L$ and the scope's function is bit-identical at
the moment of growth --- the same zero-entry argument that
governs every operator of Table~\ref{tab:operators}, applied to
a cycle. \emph{(iii) Bounded serving cost.} A model that has
grown a cycle iterates when it serves; $K_{\max}$ caps the cost
of every forward pass, and the iteration count is an audited
quantity. \emph{(iv) Benign credit assignment.} Training
differentiates through the executed iterates (the stop index
held constant); under the certificate the backward products
decay geometrically, so the unrolled gradient neither explodes
nor requires implicit differentiation --- verifiability is kept.

\paragraph{Position in the algebra.} The cycle joins the
algebra as a \emph{sixth operator} --- the five of
Table~\ref{tab:operators} grow acyclic structure; $\lambda$ is
the one that does not --- and as a \emph{third growth
direction}; the resource it buys is distinct in kind: width
buys capacity at a scale, composition buys serial depth inside a
scope, the loop buys \emph{iteration and state}. It obeys the
full governance contract --- budgeted (one cycle per scope, at
the chain's end), audited (growth, removal, projection, and
iteration counts are ledgered), removable (exactly restoring the
pre-cycle function when untrained), and \emph{off by default}: a
policy switch must enable it, and the default world grows no
cycles. What is claimed is deliberately bounded: no enlargement
of the reachable function class on compact domains is asserted
--- the claim concerns the \emph{required size and structure} of
acyclic approximations to iteratively defined laws, and where
that translates into measured value is a question the paper
leaves to its pre-registered program (\S\ref{sec:loops-status}).

\subsection{Realization II: self-processing units}
\label{sec:loops-spu}

Growth as defined so far is structurally exact but
developmentally passive: a newborn inner body enters at zero and
waits for the global gradient to shape it. The second
realization gives the newborn a bounded period of \emph{active
self-shaping} --- a local solving loop over its own weights ---
before and while it earns participation.

\paragraph{Definition.} A newly grown inner body, during a
policy-bounded newborn window, runs a small number of
weight-adjustment steps per training step against a \emph{local
process objective}: a quantity computed entirely from the body's
own forward behavior on the data it actually receives, blind to
labels and to the global loss. The objective is \emph{functional
consistency, guarded against collapse}: writing $z^{(k)}_i$ for
the body's response to sample $i$ under the $k$-th bounded
internal perturbation of its computation, $\bar z_i$ for their
mean, $s(\theta)$ for the batch dispersion of the unperturbed
response, and $s_{\mathrm{entry}}$ for that dispersion at
enrollment, the loop minimizes
\begin{equation}
J(\theta) \;=\;
\frac{1}{Kn}\sum_{k=1}^{K}\sum_{i=1}^{n}
\bigl(z^{(k)}_i - \bar z_i\bigr)^{2}
\;+\;
\gamma\,\max\!\bigl(0,\;
\rho_{\mathrm{floor}}\, s_{\mathrm{entry}} -
s(\theta)\bigr)^{2}.
\label{eq:spu}
\end{equation}
In the released system $K = 4$ perturbation copies,
$\gamma = 1$, and $\rho_{\mathrm{floor}} = 0.5$ (policy
defaults, per-model overridable). The first term asks the
response to be stable under
perturbation; the second is a hinge that fires only when the
response scale falls below a floored fraction of its entry
scale --- the collapse guard, since the degenerate minimizer of
consistency alone (answer everything identically) drives
$s(\theta)$ to zero and sits deep inside the hinge's penalty
region. The loop is bounded by an iteration budget and a
relative tolerance; convergence or budget exhaustion both
return control.

\paragraph{The order is structural.} Self-processing composes
with the recursive substrate through one invariant: \emph{process,
then participate, then learn}, at every depth. On each training
step, eligible newborn bodies first self-process, then the
released forward computation runs with their refined weights,
then the global learning step assigns credit --- one visit per
unit, no re-entry, at any nesting depth. The invariant is what
makes self-processing a \emph{scale-local} mechanism rather than
a second optimizer: the global contract sees only a model whose
fine structure arrives at each step slightly better prepared.

\paragraph{Governance.} The loop is label-blind by construction;
serving-pure (it exists only at evolution time --- a serving
model never self-processes and records nothing); bounded (budget
and tolerance are policy); optimizer-neutral (the global
optimizer's state is untouched by construction, an invariant the
released tests pin bitwise); hierarchy-respecting (eligible
bodies are exactly the newborn fine-scale corrections of
\S\ref{sec:scalehierarchy}, so the mechanism follows the growth
frontier); and fully audited (every processing event, skip, and
its reason is ledgered). Where a body's realization does not yet
support the loop, the unit is skipped with a disclosed reason ---
staging is auditable, never silent.

\subsection{One discipline, two solved quantities}
\label{sec:loops-compare}

\begin{table}[t]
\centering\small
\caption{One principle, two realizations. Both are local solving
loops embedded in the global pass; they solve different
quantities on different clocks under the same governance
contract, and each is enabled by policy, never by fiat.}
\label{tab:loops}
\begin{tabular}{@{}lll@{}}
\toprule
 & Grown cycle ($\lambda$) & Self-processing unit \\
\midrule
Solved quantity & scope \emph{state} $z^{*}$ & newborn \emph{weights} \\
When it runs & every pass on the scope & newborn windows only \\
At serving & iterates (capped) & absent by construction \\
Guarantee & enforced contraction & bounded loop; collapse-guarded
objective \\
Cost bound & $K_{\max}$ per pass & $S_{\max}$ per step \\
Entry & exact ($L_{\mathrm{out}}=0$) & n/a --- refines weights;
adoption gated \\
Adoption & the gate, as always & the gate, as always \\
\bottomrule
\end{tabular}
\end{table}

Table~\ref{tab:loops} aligns the two realizations. They share
locality (one scope, one body), boundedness (hard caps validated
for feasibility), auditability (every event ledgered), exact or
neutral entry, and gate adoption; they differ in the quantity
solved and the clock on which they run. Neither knows of the
other, and a model may carry both. The shared discipline is the
section's thesis: iteration is admitted into a growing model not
as an architectural conviction but as a \emph{governed local
resource}, purchased where a scale demands it and accounted like
every other structural liberty in this paper.

\subsection{Epistemic status}
\label{sec:loops-status}

Both mechanisms are released with the system and carry
pre-registered examination programs in the style of
App.~\ref{sec:empirical}; their MECHANICS carry measured
verdicts in the evidence table (exact entry bitwise, the
enforced contraction bound with its certified activation
Lipschitz constant $C_G = 1.128994$,
finite-difference-verified loop gradients, bitwise
optimizer-neutrality for the self-processing loop); their value
adjudication at qualified scale is open at the time of
writing. This section
accordingly claims \emph{mechanism and governance} ---
well-posedness, exactness or neutrality at entry, boundedness,
auditability, and the multiscale discipline that contains both
--- and no empirical superiority.

\section{Growable Attention: Capacity Follows Demand}
\label{sec:growatt}

Everything before this point grew feed-forward and cycle
structure and governed it; the attention component itself
remained as construction cast it. This section removes that exception. The
result is a demanding test of the theory: attention is the
substrate family's load-bearing mechanism, and if the axiom of
total plasticity is to mean anything, it must reach the heads.

\subsection{The object and its two layers of existence}
\label{sec:growatt-object}

A multi-head attention block exists twice over. As
\emph{parameters} it is a set of per-head projection matrices
$(W^Q_h, W^K_h, W^V_h, W^O_h)$; as \emph{behavior} it is, for
every input row, a family of probability distributions --- each
head, at each position, distributes one unit of attention over
the admissible positions. The standard form freezes both layers
at construction, and each freeze is a lock-in of exactly the
kind \S\ref{sec:lockin} names. First, \textbf{head lock-in}: the
per-head width is fixed by fiat at $d_h = d_{\mathrm{model}}/H$,
equal for all heads forever, so a head serving a heavy relation
and a head serving a trivial one are constrained to the same
capacity --- the demand inequality is well documented from the
destructive side by the pruning
literature~\cite{michel2019sixteen,voita2019analyzing}, which finds a large
fraction of heads removable at little cost; the constructive
response has been missing. Second, \textbf{no local
accountability}: no objective in the standard training loop ever
evaluates the attention distributions themselves; a head may
collapse its distribution to a delta or diffuse it to uniform,
and nothing in the loss says so until the damage has propagated
to the output, where a global signal cannot localize it.

The section's claim is that both freezes are removable under the
same contract that governs every other operator in this theory:
evidence-triggered, exactly function-preserving at application,
gate-adjudicated on held-out reality, budgeted and audited.
Unequal head widths then appear not as an architectural
hyperparameter but as a \emph{lifecycle outcome}: the record of
where demand actually fell.

\subsection{Exact growth of heads}
\label{sec:growatt-ops}

Two operators extend the algebra of \S\ref{sec:algebra} to the
attention block: \textsc{head-add} births a new head of minimal
width beside its siblings; \textsc{head-widen} grows one head's
$d_h$ by a chosen increment. Both are exact at application, and
their exactness is not a numerical accident but a structural
consequence worth stating.

\paragraph{The two-step trainability property.} Exactness forces a
zero factor: every pathway a growth operator opens must carry
at least one exactly-zero factor at application, or the model
function moves. The same zero then buys its price --- the side
of the pathway whose gradient passes \emph{through} the zero
factor is first-order blind for exactly one training step ---
and \emph{where} the zero sits decides who waits. The two
operators place it differently, so the property has two
instantiations, both measured on the live substrate.
\textsc{head-add} puts the zero at the exit: born with
$W^O_{\mathrm{new}} = 0$ (generators seeded at host scale),
the new head contributes exactly nothing (measured: bitwise
$0.0$ deviation at application), and the task loss is
first-order blind to \emph{every} generator matrix --- at step
one only $W^O$ moves, with $W^Q, W^K, W^V$ measured exactly
zero; at step two, with $W^O \neq 0$, all come alive.
\textsc{head-widen} places the zeros \emph{pairwise inside}
the pathway: the new key columns and output rows are the zero
side, the new query and value columns are seeded, so every new
product carries one zero factor (exact at application,
deviation $\le 10^{-12}$) while each zero-side matrix faces a
seeded partner --- the task gradient reaches it at once. At
step one the zero side ($W^K$ columns, $W^O$ rows) moves under
the plain task signal; the seeded side ($\widetilde{W}^Q, W^V$
columns) is the blind one, its task gradient exactly zero
until step two. An
all-zero widening --- both sides zero --- is refused by the
operator itself: it is a certified total saddle, a pathway no
gradient can ever enter, which is freezing by construction and
therefore out of compliance with the axiom. One step of
latency, placed by the operator's choice of zero, is the whole
price of growth that never moves the function. The property holds identically for widening, where the new
columns are the zero side. This is the price and the proof of
exact growth: one training step of latency, in exchange for a
function that never jumps.

\paragraph{The absorbed normalizer.} The textbook
$1/\sqrt{d_k}$ is evaluated at runtime from the current width,
so widening a head would silently rescale its existing logits
--- a measured $O(1)$ shift (1.27 in the reference
configuration) that violates exactness before training even
begins. The scale is therefore \emph{absorbed} into the
projection parameters at birth and carried by them thereafter;
widening then composes with the existing scale exactly
(deviation $\le 10^{-12}$ under the uniform absorbed
convention).

\paragraph{The additive output identity.} Per-head outputs
concatenated and passed through the block projection equal the
sum of per-head width-slices of that projection (blockwise
product $=$ concatenation; measured $3.6\times 10^{-15}$,
i.e.\ floating-point exact). This identity is what makes
\emph{per-head} growth well defined: a head's capacity can
change without renegotiating any other head's contract with the
output.

\subsection{Evidence and decision}
\label{sec:growatt-evidence}

Where to grow is decided by the same instrument family that
serves the rest of the theory (\S\ref{sec:signals}), now
computed per head: the instability signal $u_h$; the
\emph{loading} of each head (its share of the block's output
energy on live data); and the block's capacity participation
ratio. The two signals answer two different questions.
\textsc{widen} answers \emph{one head under demand}: among heads
past a minimum age, the head whose $u_h$ stands out against its
peers ($u_h \ge \kappa_{\mathrm{att}}\times$ the aged mean,
$\kappa_{\mathrm{att}}=2$ as policy) is widened.
\textsc{add} answers \emph{demand not attributable to any single
head}: when the loading spreads across the layer --- participation
ratio at least $(1-\beta_{\mathrm{att}})H$ over the $H$ heads,
$\beta_{\mathrm{att}}=0.1$ as policy --- a new head is born. A single-head
layer widens first and escalates to \textsc{add} only if a
recently accepted widen failed to relieve the same trigger (the
ledger keeps that memory); newborn heads are excluded from
judgment until a minimum age, because a newborn's instability
statistic starts at its construction value --- an artifact of
the estimator, not evidence of demand. The event
itself is the standard governed speculation of
\S\ref{sec:governance}: suggest from evidence, apply exactly,
probe-train, and let the gate adjudicate on held-out data ---
accept or refuse, with the refused candidate discarded and the
incumbent function untouched. One disclosure about protocol
history: an earlier generation of these runs drew its probe
epoch from the adjudication batch itself; the released system
separates the two lanes (the probe trains on the training lane
only, the holdout is scored only, a probe-less proposal is
refused with an explicit logged error), and \emph{every gate-adjudicated lane reported
in this paper was re-run under the separated protocol} --- the
numbers printed here come from the clean runs, and no verdict
changed under the migration. Nothing about attention required a new \emph{adoption}
concept; this uniformity is the design's intent.

\subsection{The attention discipline: a third local
objective}
\label{sec:growatt-disc}

The second freeze --- no local accountability --- is lifted by
giving the distributions their own objective. For each head and
each input row $i$ with $F_i$ admissible positions, the row's
attention entropy $H_i$ is held to a band
$[\,\alpha_{\mathrm{lo}} \log F_i,\; \alpha_{\mathrm{hi}} \log
F_i\,]$, and the discipline objective $J_{\mathrm{att}}$
penalizes quadratic excursions beyond either edge. The band is
a \emph{plasticity floor}, and the connection to the axiom is
direct: a collapsed row is not merely inelegant --- its softmax
is saturated, its gradient is dead, and a dead gradient is
freezing by another name. A row at uniform is the complementary
death: attention that selects nothing has ceased to compute.
The band keeps every row in the regime where it can still
learn; entropy~\cite{shannon1948} is here the measure of a
distribution's
\emph{order}, and the objective is a thermostat, not a
maximizer --- both too much order (collapse) and too little
(diffusion) are penalized.

The gradient of $J_{\mathrm{att}}$ through the softmax is
derived in closed form and certified against finite differences
at $3.0\times 10^{-11}$ across the causal, vector, and
binding-hinge configurations. Its application obeys a
\emph{locality rule}: parameter gradients flow from the combined
signal (task $+\ \lambda J_{\mathrm{att}}$), but the gradient
handed \emph{onward} to the block input carries the task signal
alone. The rule is enforced by measurement rather than
convention --- the uncontrolled variant leaks a measured
discipline component into upstream blocks (the red-arm
experiment), quietly converting a local objective into an
unaccounted global regularizer; the split keeps the discipline
a strictly local affair. (The axiom is not implicated: it
governs \emph{trainability}, and the split shields no
parameter from adaptation --- what it routes is \emph{credit},
the same defense the Discussion gives for target handing.) Gates bound the discipline's
reach: a warmup period, a minimum head age, and a per-head
switch, all policy, never fiat.

The entropy-collapse phenomenon itself is documented in the
transformer literature, together with a global preventive
solution --- a spectral-norm reparameterization that keeps
entropy from collapsing anywhere, ever~\cite{zhai2023stabilizing}; the underlying mechanism (softmax
saturation killing the gradient) is the same one the band's
lower edge guards. The present mechanism
is its local restorative complement: it does not forbid the
pathology globally but detects and heals it where it occurs ---
measured on induced collapse ($76.9\%$ of rows out of band at
induction), the end-of-window fraction is $94.4\%$ untreated
($\lambda = 0$: on task gradient alone the pathology
\emph{worsens}) and falls monotonically in the dose ---
$94.1\%$, $83.4\%$, $39.6\%$ at $\lambda = 0.05, 0.5, 2.0$.
The comparison of costs is reported in
\S\ref{sec:growatt-verdicts}.

\subsection{The exact-expectation SPU objective}
The bounded internal perturbations of \S\ref{sec:loops-spu}
are realized in the released system as random coordinate
masks; write $J_{\mathrm{inv}}$ for that estimator.
\label{sec:growatt-spu}

Self-processing sessions (\S\ref{sec:loops-spu}) judge convergence by
an objective that was, in its first form, estimated by random
masking --- a standard Monte-Carlo estimator~\cite{metropolis1949}: cheap, but stochastic. For the leaf case the
mask estimator $J_{\mathrm{inv}}$ (the session objective
evaluated on random coordinate masks; $K$ the mask count) has a
closed-form expectation under its own truncated masking law:
$\mathbb{E}[J_{\mathrm{inv}}] = (1 - 1/K)\, J_{\mathrm{an}}$,
with $J_{\mathrm{an}}$ the analytic form --- a small algebraic
identity, so the analytic objective serves in the mask form's
place, with any threshold calibrated on the mask scale rescaled
by $(1-1/K)$. On the certification ensemble the empirical means
differ by less than one standard error ($|z| = 0.73$), as an
identity should. What determinism buys is quantified by
the committed certification script: the mask estimate's
per-draw standard deviation is of the order of its mean
($\mathrm{sd}/\mathrm{mean} = 0.899$), giving an $18.4\%$
probability that a single draw crosses the stop
threshold by chance at a step where the true objective does not, and a
$12.2\%$ incidence of degenerate identical-mask draws at the
default policy; the analytic form has neither pathology, at
lower compute cost. The verdict-parity experiment below confirms
that at decisive regimes the two forms never disagree ---
the stochastic pathology lives entirely in the ambiguous zone.
Outside that zone the closed form is first-order and the mask
audit remains in service; the boundary is stated, not hidden.

\subsection{Falsifiable predictions and their verdicts}
\label{sec:growatt-verdicts}

Four predictions were printed before the experiments, each with
its failure condition; the pre-registered stop rule (a negative
or null verdict is reported at full prominence and no
post-hoc mechanism is added to rescue it) applied throughout;
the appendix carries the verdict table
(App.~\ref{sec:aseries}), the stream case
(App.~\ref{sec:streamcase}), and the generalization campaign
(App.~\ref{sec:gseries}).
The experiments ran at kilobyte scale (single layer,
$d_{\mathrm{model}} \le 64$), five seeds per cell, every run
recorded with its resolved configuration in a write-once store;
every number below carries run identifiers in the committed
analysis twin or its campaign store. The principal capacity results these
predictions feed --- including the comparisons against a
budget-matched transformer and a lifelong-trained fixed network
--- are reported in \S\ref{sec:res-capacity}, and each verdict's
row appears in Table~\ref{tab:evidence}.

\paragraph{P1 (growth tracks demand).} \emph{Refuted in its
strong form; the mechanism's payment verified in its intended
regime.} Across a complexity ladder at adequate birth capacity,
grown width did not correlate with task complexity --- growth
was neither needed nor harmful (the gate's no-harm floor held:
paired ablation deltas $\approx 0$). Under \emph{starved} birth
capacity, however, whenever the trigger fired, growth paid
large (held-out error $-35\%$ and $-64\%$ in the firing seeds),
and the archive census localizes the strong form's failure to
the \emph{acceptance} step, not the trigger and not the
operators: the plateau trigger fired in all five starved seeds
(2--6 widen attempts each), but the widen probe's own training
degraded the trial in three of them, so the gate accepted
growth in $2/5$ --- the same probe self-injury later measured
directly in the governed-gate rounds
(App.~\ref{app:x2}). The precise
statement: growth pays where demand exceeds capacity; sensing
that condition reliably is the open engineering item, and the
theory's own instruments (\S\ref{sec:growatt-evidence}) are the
designated path.

\paragraph{P2 (local discipline is more effective than global repair).}
\emph{Confirmed, with the locality dividend visible in cost.}
On induced attention collapse, the local discipline restored
$97\%$ of lost held-out performance at roughly $7\%$ of the
task-work cost (gradient-update work on task batches) of a
global-retraining control --- ordinary task training over the
same window, the natural repair a practitioner would reach for
--- which restored $49\%$; no arm against the published
\emph{preventive} method was run (the two answer different
regimes, App.~\ref{sec:related}); the dose-response of
\S\ref{sec:growatt-disc} establishes the causal chain.

\paragraph{P3 (instruments lead the loss).} \emph{Confirmed,
with a width-dependent boundary.}
Under mid-service distribution shift, the per-head instrument
trajectories crossed their pre-shift envelopes on average
$4.9$ evaluation periods before held-out error crossed its own
(positive lead in $11/15$ runs --- five seeds at each of three
widths; sd $16.4$, range $-35$ to $+24$). The four misses
concentrate at the narrow end: $3/5$ at $d_{\mathrm{model}}=16$
(worst $-35$), $1/5$ at $32$ ($-5$), $0/5$ at $64$ --- with few
heads the per-head statistics are coarse and their smoothed
trajectories can lag an abrupt error jump, so the instrument's
lead is reliable at the widest tested width and noisy at the
narrowest. A false-alarm characterization ran as a post-review control
(App.~\ref{sec:aseries}): the identical lane with no shift
fires the $2\sigma$ envelope in $13/15$ runs, so the raw
envelope-crossing criterion is sensitive but noisy --- the
measured lead is a median tendency, a statement about
\emph{when} the instruments move, not a usable standalone
alarm at $2\sigma$; any deployment needs a debounced
threshold, and the false-alarm rate is on the record.
Within that boundary, the model's self-knowledge
is temporally ahead of its performance record, which is what
makes governed \emph{anticipation} possible in principle.

\paragraph{P4 (analytic $=$ mask at decisive regimes, cheaper).}
\emph{Confirmed exactly.} Across the convergence-verdict
ensemble, decisive-regime parity was perfect, with the analytic
form's variance identically zero; the mask form's
spurious-stop and degeneracy rates are the quantified stochastic cost of the
estimate (\S\ref{sec:growatt-spu}).

Two boundary results complete this section's scope.
First, the discipline's healing is \emph{local} in reach by
design: heads it is gated off from are untouched (the per-head
switch verified live). Second, growth under this section's
operators changes \emph{capacity}, not computational class ---
a task whose difficulty is an exact accumulator the attention
family lacks is not rescued by wider heads; that boundary
belongs to the architecture literature and is stated here as
scope, not tested as a defect of the mechanism.

\paragraph{Serving heads.} Attention hosts serve point and
distributional outputs through standard head components; these
heads are standard components and are claimed only as
governed constituents --- the contributions here are the exact
growth, grading, and governance of the structure beneath them.

\section{The Governed Lifecycle}
\label{sec:lifecycle}

The lifecycle is administered by a \textbf{factory} that
manufactures, serves, versions, gates, and audits models --- and
itself \textbf{learns nothing}: business data flows only into each
model's own versioned store, so the surrounding machinery stays
generic across domains and carries no residue of any one of them.
Everything below is a factory verb operating on one model's two
states (Figure~\ref{fig:lifecycle}).

\begin{figure}[t]
\centering
\begin{tikzpicture}[
  box/.style={rectangle, rounded corners=3pt, draw=black!70,
              align=center, inner sep=5pt, font=\small},
  gate/.style={diamond, draw=black!70, aspect=2.2, inner sep=1.5pt,
               font=\small},
  lab/.style={font=\scriptsize},
  >={Stealth[length=2.2mm]}]
\node[box, fill=black!4] (brain) at (0, 2.6)
  {\textbf{LLM (the brain)}\\ \scriptsize language, context, feature
   extraction --- operates everything via MCP tools};
\node[box] (working) at (-4.6, 0)
  {\textbf{working state} $w$\\ \scriptsize study / practice /
   $\omega$ $\rho$ $\delta$ $\sigma$ $\Phi$\\ \scriptsize (speculation is free)};
\node[gate] (gaten) at (0, 0) {\shortstack{gate\\
  \scriptsize $q_w > q_s$?}};
\node[box] (served) at (5.0, 0)
  {\textbf{committed version} $s$\\ \scriptsize serves \texttt{infer};
   versioned lineage;\\ \scriptsize rollback};
\node[box] (holdout) at (0, -2.3)
  {\textbf{held-out stream} $\mathcal{H}$\\ \scriptsize time-stamped
   reality; quarantined;\\ \scriptsize the gate scores its recent slice};
\draw[->] (brain) -- node[lab, left] {teach / grow / self-study grants}
  (working);
\draw[->] (brain) -- node[lab, right, xshift=1mm] {infer\, /\,
  add\_holdout\, /\, check\_drift} (served);
\draw[->] (working) -- node[lab, above] {commit} (gaten);
\draw[->] (gaten) -- node[lab, above, pos=0.34, xshift=-1mm] {promote (strict)} (served);
\draw[->, black!60] (gaten.south) ++(-0.15,0.05)
  to[bend left=18] node[lab, below, yshift=-0.5mm]
  {refuse (logged)} (working.south);
\draw[->] (holdout) -- (gaten);
\end{tikzpicture}
\caption{The governed lifecycle. Speculation (learning and all
growth operators) is free in the working state; serving changes only
through the gate; the held-out stream is quarantined from all
training. Teaching by the brain can
therefore never degrade the served model (threat model:
\S\ref{sec:llm}).}
\label{fig:lifecycle}
\end{figure}
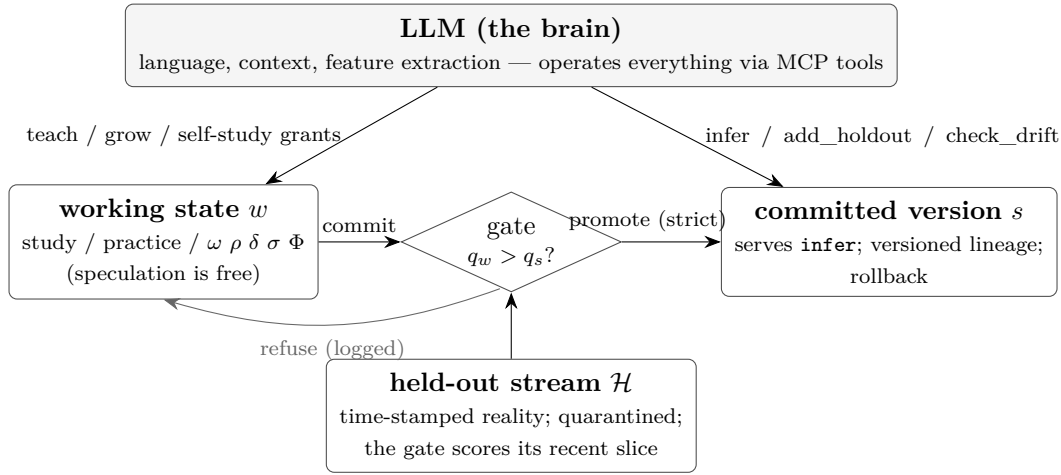

\paragraph{Working vs committed.} All learning verbs act on the
working state: supervised study, practice
updates, every operator of Table~\ref{tab:operators}. Serving reads the committed version.
\texttt{commit} runs the gate; \texttt{reset} discards the session;
\texttt{rollback} re-points the active version within the audited
lineage. The asymmetry is deliberate: speculation is free;
adoption is earned.

\paragraph{Store-per-version and forgetting.} Every version carries
its own accumulated training store; a candidate is trained \emph{from
scratch} from its store. At kilobyte scale this full-replay
discipline is cheaper than any forgetting-mitigation machinery and
makes catastrophic forgetting a non-issue at this scale (full
replay: nothing is forgotten that is re-trained from the record
every time); where the drift protocol's recency window sheds
older data, forgetting returns as \emph{explicit, logged
policy} rather than as an accident of optimization. It
also makes every version a pure function of its store ---
reproducibility as a side effect --- and confines a rejected
candidate's data to the rejected version's directory, outside the
served lineage.

\paragraph{The drift protocol.} \texttt{check\_drift} compares the
active version's metric on the recent held-out slice against the
score recorded at its promotion; decay beyond tolerance raises
\texttt{needs\_reteach}. Adaptation is then ordinary teaching under
two recency parameters: $\texttt{window} = N$ trains on only the most
recent $N$ stored examples (shedding labels the new reality
contradicts, while the full store remains in the lineage), and
$\texttt{recent\_n} = N$ judges the gate on the most recent $N$
held-out examples --- necessary because a mixed-era holdout can score
the adapted candidate and the stale incumbent to an exact tie, and a
strict gate would then block adaptation forever. Drift handling is
thus not an operator's habit but an API contract: notice decay,
re-teach on the recent window, be judged on recent reality.

\paragraph{Granted self-study.} The model can also improve
\emph{itself}: \texttt{self\_review} is a read-only self-report
(retention trends, learning progress, saturation, a self-generated
question list --- see App.~\ref{sec:sseries} for what proves real), and
\texttt{run\_self} executes a budgeted self-study session ---
consolidation over the model's own store, optionally proposing growth
--- under explicit consent: granted budgets only, every block logged,
nothing autonomous, and the gate retaining sole promotion authority.
The holdout is never studied: self-study that could train on the
gate's own evidence would corrupt the certificate, so the store and
the holdout are disjoint by construction, and violations are
refused with an explicit logged error.

\paragraph{Readable regularities.} Categorically-shaped models
additionally mine their store into an ordered decision list in the
classical separate-and-conquer family~\cite{cn2,ripper} --- human-
and LLM-readable IF/THEN regularities with confidence and support,
surfaced through \texttt{discoveries} and cited by \texttt{infer}
when the mined rule agrees with the network's prediction. The
extension thereby returns to its brain something pre-training cannot
contain: the regularities of \emph{this} domain, as this model's data
shows them.

\paragraph{Two governance-integrity contracts.} Two
contracts of the served surface, added after the campaigns'
audit trails exposed their absence, are now part of the
lifecycle's specification: policy writes to the growth-policy
dictionary \emph{merge} one level deep (installed governance
keys survive unrelated writes; an explicit null removes a key),
and the parameter budget judges every growth entrance ---
direct verbs, plans, and bounded trials alike --- as a
\emph{whole act}, the primary move plus any auto-compliance
rider, before any mutation. Both contracts are enforced by the
same refusal discipline as every other gate.

\section{Evaluative Learning under Total Plasticity}
\label{sec:eval-learn}

Teaching supplies the model with labeled evidence; service
supplies it with outcomes. This section extends the governed
lifecycle to the second kind of signal: learning from
\emph{evaluations} --- scalar judgments of how well an episode
went --- without giving up any part of the total-plasticity
contract. One evaluative primitive appears at two levels of the
system, and both levels answer to the same gate.

\subsection{One evaluative primitive, two loops}
\label{sec:two-loops}

\paragraph{The S-loop: preference over structure.} Growth moves
(widen, deepen, and their sites) are discrete choices whose value
is only observable after the fact. The structure loop keeps, per
move family, a bucketed record of \emph{credited advantages}: when
the gate adjudicates a grown candidate against its incumbent, the
scored advantage is credited to the move that produced the
candidate, and the running sufficient statistics (count, mean,
variance) accumulate in the move's bucket. Selection over buckets
supports five interchangeable rules; the production default
samples each bucket's posterior mean under a fixed-scale Gaussian
perturbation, $\tilde{q}_m = \hat{q}_m + \sigma_0\,
\sqrt{\smash[b]{v_m/(w_m+1)}}\;\varepsilon$ --- where
$\hat{q}_m$, $v_m$, $w_m$ are the bucket's running mean,
variance, and count, $\varepsilon \sim \mathcal{N}(0,1)$, and
$\sigma_0$ is a fixed exploration scale in the posterior-sampling
family~\cite{thompson} --- so that exploration
narrows as evidence accumulates while no bucket's scale is ever
re-tuned against outcomes (the five rules: a fixed preference, a
clipped greedy mean, the sampling rule above, an
upper-confidence rule, and an $\varepsilon$-greedy rule). Preference never mutates the model: it
only orders which candidate is \emph{offered}; adoption remains
the gate's alone.

\paragraph{The P-loop: policy optimization on the model.} The
same model that learns from teaching can learn from episode
returns. The policy loop implements clipped-surrogate policy
optimization (the PPO/GRPO family~\cite{ppo,grpo}) directly on the growable
substrate through the pseudo-target identity
\begin{equation}
y^{*} \;=\; h \;-\; \frac{\partial \mathcal{L}}{\partial h},
\label{eq:pseudo-target}
\end{equation}
which converts any differentiable policy loss $\mathcal{L}$ at
the model's output $h$ into a teaching target for the existing
learning kernel: one mechanism serves both supervised evidence
and evaluative gradients. Advantage estimation~\cite{gae}, ratio clipping,
entropy regularization, and the optional KL term to a reference
policy are computed outside the substrate and enter only through
Eq.~\eqref{eq:pseudo-target}; the model itself remains a
teachable network with no reinforcement-specific machinery inside
its weights.

\paragraph{One governance.} Both loops answer to the paired-episode
gate: a candidate (grown or not) and its incumbent are scored on
episodes drawn from a quarantined evaluation namespace, adoption
requires a strictly better score, every adjudication appends an
audit record (episode count, scores, provenance of the move), and
rollback restores the pre-adoption version byte-for-byte. Five
interface laws fix how the two loops co-exist: the two trainers
are \emph{co-resident}, not sequential phases of a lifecycle
(LAW-1); teaching and reward learning \emph{interleave} within
one service life (LAW-2); anti-interference instruments ---
the incumbent-anchored reference and interleaved refresher steps
--- protect each regime from the other (LAW-3); the gate
adjudicates \emph{across} regimes, so an adoption must not trade
one regime's competence away for the other's (LAW-4); and the
S-loop is \emph{indifferent} to which regime produced the
evidence it credits (LAW-5).

\subsection{The quasi-static law extends to reward learning}
\label{sec:quasistatic-rl}

The paper's verification discipline --- a growable network at any
instant is a fixed network, so every claim about the growing
system must reduce, at each instant, to a claim checkable on a
fixed twin built from independent components --- extends to the
policy loop. The twin is constructed from PyTorch official
components only (\texttt{nn.Linear}, \texttt{Categorical},
autograd, \texttt{torch.optim})~\cite{paszke2019pytorch}, with the model's
weights transplanted; nothing of our implementation is on the
reference side. At pre-growth, post-deepen, and post-widen
instants the model and its twin agree in function to
$5.3\times10^{-13}$--$2.2\times10^{-11}$ maximum absolute
error (cross-library
\texttt{tanh} unit-in-last-place noise, arbitrated against a
50-digit \texttt{mpmath} evaluation~\cite{mpmath}), and the
function change \emph{at} a growth instant is exactly zero. The
policy-side computations --- values, generalized advantage
estimation, and the clipped loss --- agree at every checked
instant; a full 40-update training run yields parameter
trajectories identical to the official-autograd twin to
$1.8\times10^{-15}$, and the anchor-term pipeline matches the
official divergence-and-autograd path to $2.8\times10^{-17}$.
At a grown instant the model's internal Adam reproduces
\texttt{torch.optim.Adam} on the transplanted twin to
$7.1\times10^{-12}$ over a 50-step weight trajectory, and a
loss-convention arbitration identifies the model's gradient chain
as exactly $0.5\times$ the mean-square autograd convention
($3.0\times10^{-12}$) rather than assuming it. Beyond the twin,
component-and-outcome comparisons against authoritative
implementations --- Stable-Baselines3 full-pipeline training,
its official advantage buffer, gymnasium environments, and an
industry bandit library~\cite{sb3,gymnasium,mabwiser} --- pass in
all nine registered cases.

\subsection{Experimental design}
\label{sec:rl-design}

Worlds span an in-house seeded family --- a stationary control
world and a staged world run at a capacity-sufficient width ---
and the quadratic staged world (capacity-binding), whose stages
arrive on a schedule, plus external gymnasium
environments~\cite{gymnasium}: CartPole with a mid-life dynamics
change, LunarLander, and the sparse-reward task of
\S\ref{sec:rl-boundaries}. The central control is the
\emph{protocol-matched twin}: an arm whose candidates are trained
clones under the identical offer schedule, warmup budget, and
gate, so that any difference from the growing arm isolates
structure itself from the compute and protocol that surround it.
Instrument pre-verification --- registered after the first
battery and applied to every capacity-claim battery thereafter
--- runs a capacity-binding probe (does a larger fixed net
actually score higher here?) and an outcome-variance probe (is
the world's signal readable at the registered episode budget?)
before any arm, so that
arms are run only where the instrument can measure. Pass criteria
are registered before any arm runs; evaluation seeds for
selection and for scoring live in disjoint namespaces; adoption
costs (warmup episodes, evaluation episodes) are accounted
per-arm. Two anti-interference designs are part of the method:
the incumbent-anchored reference is installed before training
begins (its free twin otherwise byte-identical), and
task-consistent alternation compares regimes at an equal
reward-learning budget. Each battery runs five seeded lives per
arm (ten for the LunarLander gate battery); worlds, model sizes,
episode budgets, offer schedules, and pass criteria are frozen
in the registered design documents and per-run JSON stores named
in the Reproducibility Statement. The design laws here are
scale-independent; they are stated so that the same experiment
shapes can be applied to larger models unchanged.

\subsection{Results: a three-regime law}
\label{sec:rl-results}

Across worlds and arms the outcomes organize into three regimes
(Table~\ref{tab:three-regime}; Figures~\ref{fig:staged}
and~\ref{fig:cartpole}).

\begin{table}[t]
\centering\small
\caption{The three-regime law for structural growth under
evaluative learning. ``Time-axis value'' is area under the
per-episode score curve over the life; gaps are grown-arm minus
protocol-matched twin, and score scales are per-world ---
magnitudes are not comparable across rows. Registered criteria:
at least 3 of 5 per five-life comparison; for the LunarLander
gate comparison, an adoption in at least 6 of 10 lives.}
\label{tab:three-regime}
\begin{tabular}{@{}>{\raggedright\arraybackslash}p{0.30\linewidth}>{\raggedright\arraybackslash}p{0.24\linewidth}>{\raggedright\arraybackslash}p{0.38\linewidth}@{}}
\toprule
Regime & Net effect of growth & Evidence \\
\midrule
Capacity binds & higher time-axis value & quadratic staged
world: reward-only 3/5 (mean gap $+10.8$), alternating regime
3/5 at its registered criterion (mean $+1.1$); CartPole after
the dynamics change: 3/5 paired lives with one tie
(Figure~\ref{fig:cartpole}); LunarLander, gate-level merit only:
grown candidates pass the gate in 8/10 lives, trained clones in
0/10 (long-run net effect unresolved, \S\ref{sec:rl-boundaries}) \\
Capacity suffices (capacity-sufficient staged world), designed
alternating teach--reward regime &
cost-neutral & mean AUC gap $-1.9$ (ties dominate; the
registered value line missed, 0/5 --- the wash itself is the
recorded outcome, per the registered clause) \\
Capacity suffices (same world), reward-only operation (outside
the designed regime) & transient cost & mean AUC gap $-33.9$;
the designed regime removes 94\% of it \\
\bottomrule
\end{tabular}

\end{table}

Where capacity binds, the growing model attains higher time-axis
value than its protocol-matched twin (per-seed counts and
margins in Table~\ref{tab:three-regime}; within the designed
alternating regime the capacity-binding margin is small at this
scale --- 3/5 at its registered criterion, mean $+1.1$); on
LunarLander the
asymmetry appears at the gate itself --- grown candidates
demonstrate immediate merit that trained clones at equal budget
do not. Where capacity suffices and the system runs in its
designed alternating regime, growth is cost-neutral: the
alternation itself absorbs the adaptation transient that
reward-only operation would pay (94\% of the reward-only gap
removed). The only setting in which growth shows a net transient
cost is reward-only operation outside the designed regime ---
a regime the method does not prescribe.

Governance results are exact: candidates with nothing to offer
(untrained clones) are refused in 5/5 lives with final
evaluations identical to the no-offer control (a deterministic
pipeline with zero adoptions), and every adoption in
every life carries a complete audit record. The
anti-interference instruments measure as designed: the
incumbent-anchored reference lowers policy drift in 5/5 lives
(11.6--20.4 versus 25.8--33.7) at finals not worse in 3 of 5,
and interleaved refresher steps leave taught-knowledge probe
error $5.9$--$95\times$ lower at an equal
reward-learning budget (Figure~\ref{fig:anchor}). Finally, the two loops run together on
one model: in 5/5 stacked lives every adoption decision is
audited and the preference bookkeeping tracks the gate's verdicts
exactly, including the sign of a zero-gain refusal.

\begin{figure}[t]
\centering
\includegraphics[width=0.72\linewidth]{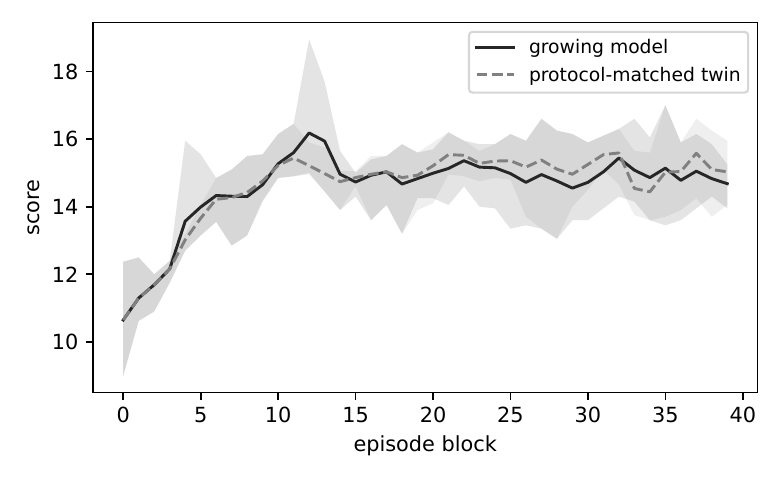}
\caption{The capacity-sufficient staged world (distinct from the
capacity-binding quadratic staged world of Table~\ref{tab:three-regime}),
designed alternating teach--reward regime:
mean score curves over five seeds (bands span the seeds). The
growing model and its protocol-matched twin track each other ---
the cost-neutral regime of Table~\ref{tab:three-regime}.}
\label{fig:staged}
\end{figure}

\begin{figure}[t]
\centering
\includegraphics[width=0.72\linewidth]{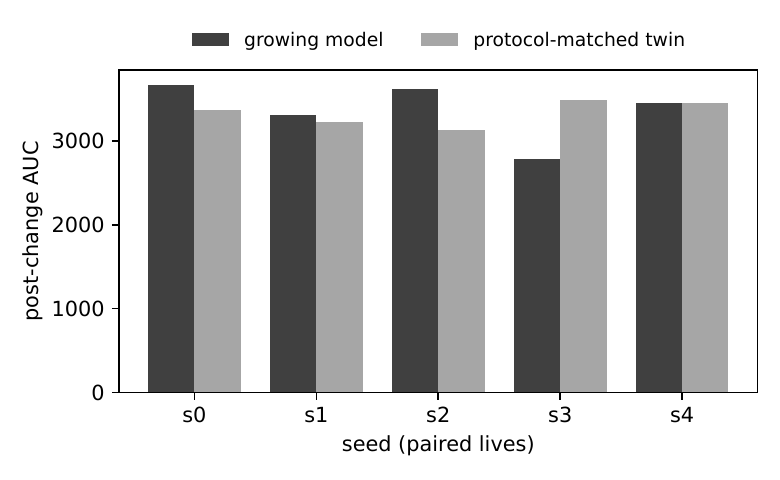}
\caption{CartPole with a mid-life dynamics change: paired
post-change area under the score curve, growing model versus
protocol-matched twin, per seed.}
\label{fig:cartpole}
\end{figure}

\begin{figure}[t]
\centering
\begin{minipage}{0.48\linewidth}
\centering
\includegraphics[width=\linewidth]{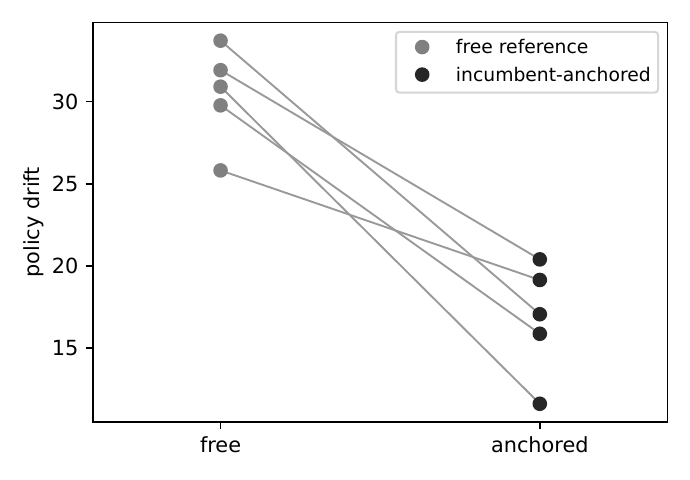}
\end{minipage}\hfill
\begin{minipage}{0.48\linewidth}
\centering
\includegraphics[width=\linewidth]{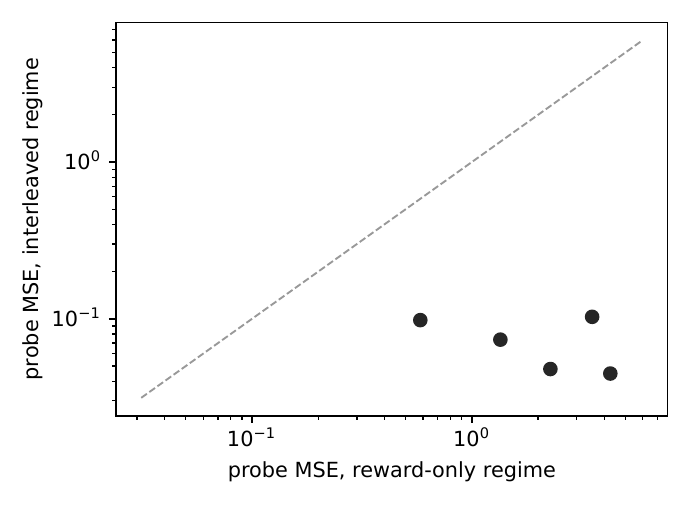}
\end{minipage}
\caption{Anti-interference instruments. Left: policy drift with a
free versus incumbent-anchored reference, paired per seed (lower
in 5/5). Right: taught-knowledge probe error under the
interleaved versus reward-only regime (log--log; all points below
the diagonal, by $5.9$--$95\times$).}
\label{fig:anchor}
\end{figure}

\subsection{Boundaries and negative results}
\label{sec:rl-boundaries}

Four boundaries are part of this record. First, a designed
falsification test: the hypothesis that growth's transient cost
is dominated by optimizer-moment loss was tested by transplanting
the incumbent's Adam moments into the grown candidate; the cost
did not move, and the hypothesis is refuted --- the transient is
carried by the function change, not the optimizer state. Second,
an instrument retirement: an adaptation-speed metric saturated
across arms and was retired in favor of area-under-curve
measures. Third, one sparse-reward task (Acrobot) remained beyond
the registered moderate-scale protocol after three mechanism-level
remedies; the proximate mechanism (policy-entropy collapse) was
identified by a live probe and the task is recorded as out of
reach at this scale rather than tuned into range. Fourth, the
long-run net effect on LunarLander lies below the resolution of
the moderate-scale protocol and is reported as unresolved rather
than claimed.

\section{An LLM Operates the Model}
\label{sec:llm}

\paragraph{The division of labor.} The LLM is the brain: it reads the
world, decides when a specialized judgment is needed, extracts the
features the model's learned shape names, calls \texttt{infer}, and
interprets the answer in context. The soft model is the brain's
growable extension: kilobytes of forever-trainable weights carrying
one domain's validated experience --- no language, no base model,
because general capability already sits on the other side of the
call. This division is what makes total plasticity affordable:
everything that must be large is borrowed from the brain, so
everything that must stay plastic can stay small.

\paragraph{Threat model.} The guarantee ``teaching cannot degrade
the served model'' holds for a \emph{careless-but-honest} operator:
one that may teach wrong, redundant, or contradictory content, but
does not corrupt the evidence channel. In the released system the
same principal that teaches also writes the held-out stream
(\texttt{add\_holdout}), so holdout provenance is the root of
trust: an operator (or an upstream prompt injection) that poisons
the holdout and teaches to match it will pass the gate while
degrading true quality. Deployments that cannot assume an honest
operator should source held-out reality from a principal separate
from the teaching channel. \texttt{rollback} is audited but not
gated --- it re-points the served version without a fresh gate success ---
and should be treated as an operator privilege, not a governed
transition.

{\sloppy
\paragraph{Self-documenting tool surfaces.} The factory exposes nine
tools over the Model Context Protocol
(\texttt{list}\slash\texttt{create}\slash\texttt{infer}\slash
\texttt{teach}\slash\texttt{add\_holdout}\slash
\texttt{check\_drift}\slash\texttt{discoveries}\slash
\texttt{versions}\slash\texttt{rollback}); the full system exposes fifty-two, adding the
operators, the growth-control instruments (deepening, dry-run
proposal, bounded trial, rule plans, snapshot/rollback,
assessment), self-study, substrate advisory, and
the standard-method baselines.
Both servers hand the connecting client an operating manual at
initialization (MCP \texttt{instructions}): the required order
(holdout before teach --- the gate needs reality before it can
certify), the feature-extraction contract, the drift loop, and
recovery guidance for the common failure modes. The protocol makes
the safety guarantees legible to the operator: \emph{teach freely ---
the gate protects quality; grow freely --- exactness protects the
present.}\par}

Validation of this surface, driven end-to-end by a production LLM,
is reported in App.~\ref{app:validation}.

\paragraph{Two deployment regimes.} The platform serves two data
regimes with two strategies, and only one of them is this paper's
theory. Where the data carries stable underlying laws beneath
drifting surfaces and a growing inventory --- the regime every
result in Part~II is drawn from --- the growable soft model under
the governed lifecycle is the intended tool. Where it does not ---
a simple curve-fitting relationship that a small fixed network
already captures, or a task whose data never changes --- the same
factory serves conventional fixed networks under a
rebuild-on-change policy: the drift monitor raises a signal, the
orchestrator commissions a fresh network from zero (larger if
needed), the evaluation gate certifies it against held-out
reality, and the versioned deployment switches over. That policy
is pure caller-side orchestration over verbs the platform already
exposes (\texttt{check\_drift}, \texttt{create}, train-to-convergence,
gated promotion, \texttt{rollback}); no core mechanism changes.
This is stated as a system capability, not a claim: the platform
degrades gracefully into a conventional model-lifecycle manager
wherever the theory's home regime does not hold.

\paragraph{Lowering the barrier: a bespoke model as an ordinary
software artifact.} Today an application developer who wants
machine intelligence in a product typically builds on hosted
foundation models; a smaller group fine-tunes open base models;
training and operating a model of one's own remains the province
of specialist teams with substantial accelerator budgets. This
platform addresses the gap between those tiers: a working
developer with no machine-learning background can build, own, and
operate a domain model on their business data with effort
comparable to writing a small program. The interface is a
natural-language-driven operator plus declarative data
onboarding: the developer describes the desired judgment in plain
words, supplies a schema and a data feed, and instructs the
operator to proceed. The general-capability operator --- an LLM
driving the tool surface above, validated end-to-end in
App.~\ref{app:validation} --- then carries out fully automated
model lifecycle management: model construction, training,
evaluation-gated promotion, versioned deployment, and rollback
are procedures the platform automates, not knowledge the
developer must hold. It runs on ordinary CPUs; models are
kilobytes to megabytes; every transition is audited and
reversible. The scope is stated explicitly:
this lowers the barrier to small domain-specialized models ---
the organ scale of this chapter's division of labor --- not to
foundation-model construction.

\section{In-Service Continual Learning as the Natural Habitat}
\label{sec:cl}

The regime this model is designed for has a name in the
literature: \emph{continual learning} --- a long-lived model on a
non-stationary stream, expected to keep serving while it keeps
learning. Work in that field identifies two coupled failure modes
of standard training. The first is \emph{catastrophic
forgetting}~\cite{mccloskey}: training on new material
overwrites competence on old material. The second is
\emph{loss of plasticity}~\cite{dohare}: under long-horizon
training, networks progressively lose the ability to learn at
all unless learning capacity is actively maintained. The two
pull in opposite directions, and the field's classical summary
term for the tension is the stability--plasticity dilemma~\cite{grossberg80}. This section states
how the present theory addresses that regime structurally; the
measured behavior appears with the campaigns
(\S\ref{sec:campaigns-results}).

\subsection{Structural answers to the two failure modes}

\emph{Plasticity, by axiom.} Loss of plasticity presupposes a
fixed capacity that training exhausts. Under the total-plasticity
axiom nothing is ever frozen, and capacity follows demand: where
learning stalls against capacity, governed growth adds it, with
exactness at application and every new parameter trainable from
step one (\S\ref{sec:theory}). Plasticity is thus not an
optimizer property to be preserved by intervention; it is a
structural resource under governance. The campaigns record the
corresponding health measurements (dead-unit fraction, effective
rank) remaining at their healthy values through every lifetime
measured (\S\ref{sec:campaigns-results}).

\emph{Forgetting, by placement.} New demand is absorbed by
\emph{new} capacity that enters at zero: the preservation
propositions guarantee that the instant of structural change does
not move the served function, and the input-conditioned ports
give grown capacity a natural association with the input regions
that demanded it. What structure alone cannot do --- and the
campaigns measure this boundary plainly --- is prevent the shared
trunk from following the only gradient it sees under a strict
domain switch. The load-bearing treatment for that regime is
\emph{review}: re-inputting old knowledge from the model's own
experience store --- the mechanism the classical
complementary-learning-systems account assigns to replayed
experience in consolidation~\cite{mcclelland1995}. The model's
store makes review self-contained; no external archive of old
data is required.

\subsection{The anti-forgetting ladder}

The strategy layer composes three levels of defense, each an
orchestration of served interfaces --- the core contributes
complete interfaces, strategies remain caller-side:

\emph{R1 --- protection windows (dose).} Event-scoped, per-region
learning-rate reduction around a growth instant: the newborn part
integrates while the old regions are shielded from its initial
noise. Windows are surgery care: effective at growth events, and
measurably not a treatment for long-horizon forgetting.

\emph{R2a --- structural response.} Demand-triggered growth
absorbs new structure in new capacity; it composes with review
and carries the new-domain side of the trade.

\emph{R2b --- review courses (data).} A \emph{bounded, curated
course} drawn from the model's own store, started after the
new-domain turbulence stabilizes, and \emph{repeated at spaced
intervals}: spaced repetition --- the long-established
distributed-practice principle of the learning sciences~\cite{ebbinghaus1885,cepeda2006} --- is the standing review
strategy,
and repetition count acts as the retention dose. Placement
matters because a completed course erodes under continued
new-domain training; repetition is the measured counter.

A fourth, intrinsic level --- local process objectives by which
old regions continuously rehearse their own knowledge --- is
future work.

\subsection{Evaluation on the time axis}

The natural evaluation for this regime is longitudinal: a life is
judged by its trajectory --- retention of earlier competence,
adaptation cost when the world changes, service continuity
through structural events, and the auditability of every decision
--- rather than by a single end-state score. The campaigns adopt
this protocol: within-life retention ratios, compute-to-parity
against retraining, relearning savings when a domain returns, and
per-decision audit trails, reported with the runs
(\S\ref{sec:campaigns-results}).

\part{Empirical Evidence}

\section{The Core Method on Standard Continual-Learning
Benchmarks}
\label{sec:cl-bench}

The previous section's habitat argument makes claims that the
field's own instruments can test. This section evaluates the core
method --- the growable model under its full governance --- on
the standard continual-learning benchmark family, at moderate
scale, with every pass criterion registered before its run.
Table~\ref{tab:cl-results} summarizes every lane's registered
criterion and outcome.
(These external benchmarks complement the in-house
continual-learning campaigns of
\S\ref{sec:campaigns-results}, which run on the registered
world families.)

\subsection{Benchmarks, fit, and protocol principles}
\label{sec:cl-fit}

The benchmark family is Split-MNIST (five two-class tasks,
consecutive digit pairs, arriving sequentially), Permuted-MNIST (fixed seeded
permutations), and Fashion-MNIST under the identical
protocol~\cite{mnist,fashionmnist}; candidates outside the
moderate-scale, dense-input domain of kilobyte-scale models were
rejected on fit grounds rather than prominence. Three protocol principles
govern every lane. First, \emph{protocol fit}: the training
protocol belongs to the method under test --- the model runs in
its designed regime (multi-epoch task training interleaved
one-to-one with small replay) --- while each evaluation subject
receives the measurement that fits its own question; the two
concerns are never conflated. Second, \emph{quarantine}: train,
validation, and test rows are disjoint; the gate scores
candidates only on validation slices, and no test item is ever
trained on --- examinations ask unseen questions of a studied
topic, never the studied items. Third, \emph{instrument
pre-verification}: a capacity-binding probe (a larger fixed net
must actually score higher), a resolution probe (the effect must
exceed the measurement's confidence interval), and, for long
sequences, a decline probe (the phenomenon under test must be
present) gate every batch; instrument sizes are calibrated only
at the probe stage and then frozen, with every calibration
iteration on record; the frozen settings include the fixed
control's width, the growth step, the replay dose (two hundred
stored rows per past task, capped at the ten most recent tasks
on the long lanes), and the review step size. Data handling is
plain NumPy~\cite{harris2020numpy} with pinned checksums.

\subsection{Capacity under sequential arrival}
\label{sec:cl-capacity}

On Split-MNIST at full data (ten seeds, five arms), the growing
model scores above the fixed small control in 9 of 10 seeds and
above its protocol-matched twin --- the compute-and-protocol
equal control --- in 9 of 10; retention is not worse in 8 of 10.
The governance lane is exact: a life offered only valueless
clones adopts none of them and finishes with float-exact
identical final evaluations to the no-offer control in 10 of 10
lives. Against the
parameter-matched ceiling (a network born at the growing arm's
maximum possible size), the grown model's final accuracy is 77\% of the
ceiling's (ratio of mean accuracies) while ending at fewer parameters on
average (7{,}117 versus a per-seed mean of 8{,}273). This is parsimony
at an accuracy cost the ceiling comparison quantifies (absolute
mean accuracies 0.580 grown versus 0.754 ceiling): capacity was
bought only where the gate certified value. The registered reduced-dose pilot
shows the same shape (75\% recovery). The permuted lane
confirms both of its registered lines (final accuracy above the
fixed control, retention not worse; 5/5 each) at five seeds. On
Fashion-MNIST the controlled comparison holds (the growing model
above its protocol-matched twin in 3 of 5 seeds, with governance
again exact); its remaining lines are a dose boundary reported
in \S\ref{sec:cl-boundaries}.

\subsection{Long-sequence plasticity}
\label{sec:cl-plasticity}

The field's current central concern is not whether a continual
learner remembers, but whether it can \emph{still learn} as
tasks keep arriving~\cite{dohare}. At twenty permuted tasks
(reduced scale, five seeds), the fixed net loses up to 39 points
of new-task learning ability across the sequence while the
growing model's keep-learning curve stays essentially flat to
improving (declines $-8.3$ to $+1.3$ points; positive decline =
lost learning ability) and its final accuracy is
higher in all five lives.

Scaling to full data exposed a measurement subtlety that we
report as a finding: with five epochs of training per task,
heavy retraining \emph{masks} the effect --- the same probe
that shows a 39-point collapse at reduced scale shows only 4.6
points after full retraining, because a fixed network,
retrained through enough epochs, still reaches a passable
score. The
protocol therefore splits training from measurement: the model
trains in its designed regime unchanged, and plasticity is read
in the \emph{single-pass window} --- the arriving task's test
accuracy immediately after its first epoch, before further
retraining obscures the reading. The split protocol and its ten-seed pass
criterion were registered before the forty-task battery ran;
the masked full-data reading entered the record first, as the
probe that motivated the redesign (a motivating probe, not
adjudicated evidence). Under this reading the
decline is plainly
visible at full data (7.6-point first-pass decline for the fixed
net at forty tasks), and the separation is clear at this seed
count: across ten
seeds the growing
model's first-pass decline lies in
$[-2.6, +5.0]$ points against the fixed net's $[+4.9, +10.8]$
--- lower in 10 of 10 seeds (single-head argmax over ten
classes; the chance level is 0.1), unchanged (10 of 10) when the
ten growth-offer tasks are excluded from the curve, so the effect is
not an artifact of candidate warmup
(Figure~\ref{fig:keeplearning}). As a component-level
cross-check under the field's own canonical protocol ---
a single pass per task with no replay in either arm --- the
growing model's first-pass decline is lower than the fixed
net's in 8 of 10 seeds, meeting its registered criterion
exactly; the gate adopted three to seven growth offers per life
under that protocol, so governance operates unchanged outside
the designed regime.

\begin{figure}[t]
\centering
\includegraphics[width=0.78\linewidth]{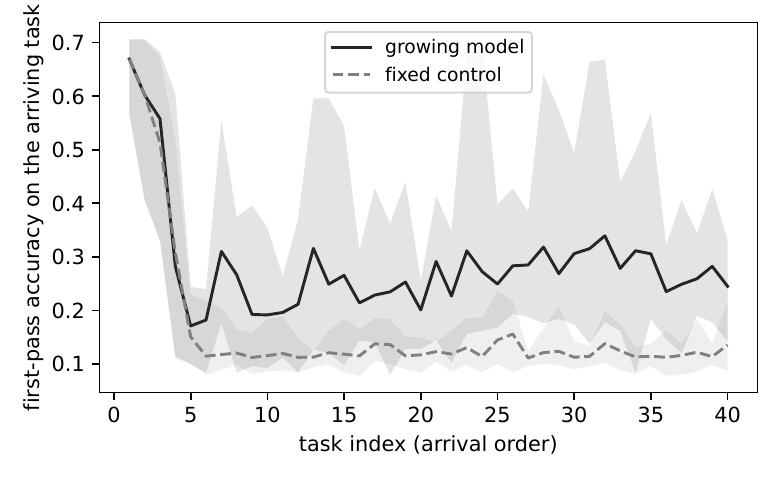}
\caption{Keep-learning curves at forty full-data tasks: accuracy
on each arriving task after its first training epoch (means over
ten seeds; bands span the seeds). The fixed control's first-pass
score decays along the sequence; the growing model's holds.}
\label{fig:keeplearning}
\end{figure}

\subsection{Review and the just-in-time refresher}
\label{sec:cl-review}

An in-service model is sometimes asked to answer for material it
studied long ago. The natural protocol is the one humans use for
examinations: the examination is announced; each subject is
reviewed just before \emph{its own} examination, from the
textbook (fresh samples of the task's training corpus, never
examination items); and subjects are examined independently ---
one subject per day, each day starting from the same
post-service state, so no other subject's review interferes.
The allocation matters as much as the budget: an equal review
window spread uniformly over all subjects would redistribute
competence rather than restore it --- favoring an evenly
mediocre model and interfering with a specialized one --- so
review is targeted at the examined subject, consistent with the
targeted-review result of \S\ref{sec:campaigns-results}.
Under this protocol, on the forty-task full-data lane, with a
finite equal review window for both
arms (60 steps over a 200-row sample per subject), the growing
model's all-task examination mean rises from 0.1818 to 0.2153
(+3.4 points) while the fixed control's falls from 0.1718 to
0.1576 ($-1.4$ points); the growing model is higher in 10 of 10
seeds. The asymmetry is the plasticity result in another guise:
the model that can still learn converts review into recovered
competence --- its long-ago band rises by 3.9 points while its
recent expertise is preserved (0.340 to 0.355) --- whereas the
fixed control does not recover under the same window ($-1.4$
points); the window and its step size are registered identical
for both arms, with no per-arm dose tuning. The converse is an
experience worth recording for practitioners: a fixed net's
apparent retention advantage on long-unreviewed material can be
a symptom of the same decline --- having largely stopped
learning, it has also stopped overwriting --- so retention
comparisons over unreviewed material can reward incapacity
rather than memory. For an in-service model the asymmetry is,
at the tested scale, a deployable property: when an old line of
work returns, a brief refresher on that work's material
restores competence without disturbing current duties.

\subsection{Boundaries}
\label{sec:cl-boundaries}

Five notes bound these results. First, on Fashion-MNIST, the two
field-standard uncontrolled lines missed their registered
criteria at the registered dose (each 2 of 5, at two epochs per
task): the offer
protocol itself carries a cost on the harder data, which the
controlled comparison absorbs but the no-offer baseline does not
--- a dose boundary, not a mechanism failure. Second, the canonical
single-pass reading passed exactly at its registered criterion
(8 of 10), and is flagged as such, as is the Fashion controlled
line (3 of 5). Third, an
experimental-integrity note: in the review experiments, every
life's service period is required to reproduce the base run's
per-task learning curve, audit trail, and pre-review score
exactly --- all twenty lives of the exam-day experiment did;
the same reproduction discipline had already caught one
instrumentation defect at its first seed during development.
Fourth, the plasticity lanes compare the growing model against
the small fixed control only; a large-from-birth fixed
control's keep-learning curve was not run in these lanes and is
recorded as an open control (the parameter-matched ceiling
exists only in the capacity lane). Fifth, the forty-task
battery's registered multi-pass companion line --- final
all-task accuracy above the fixed control in at least 8 of 10
--- missed at 7 of 10 (the same 7-of-10 form as the
twenty-task full-data run); the primary plasticity line is
unaffected, and the final-examination subject itself is
answered by the review protocol of
\S\ref{sec:cl-review}, whose registered line passed in 10 of
10. Beyond these notes,
task-free continual learning --- streams with no task
boundaries --- is future work; the drift-detection component
already in the substrate (Bayesian online change-point
detection~\cite{bocpd}) is its natural trigger.

\begin{table}[t]
\centering\small
\caption{Continual-learning benchmark results. Every criterion
was registered before its run; the two uncontrolled Fashion
lines and their dose boundary are described in
\S\ref{sec:cl-boundaries}.}
\label{tab:cl-results}
\begin{tabular}{@{}llll@{}}
\toprule
Lane & Registered criterion & Outcome & Final params / note \\
\midrule
Split, full data & above fixed control $\geq 8/10$ & 9/10 &
7{,}117 vs.\ 8{,}273 ceiling \\
 & above protocol twin $\geq 6/10$ & 9/10 & \\
 & retention not worse $\geq 6/10$ & 8/10 & \\
 & governance: identical finals & 10/10 & \\
Split, pilot & all four lines & met & 75\% ceiling recovery \\
Permuted & both lines & met (5/5) & --- \\
Plasticity, 20 tasks & all three lines & met (5/5) & --- \\
Plasticity, 40 tasks & first-pass decline lower $\geq 8/10$ &
10/10 & 7{,}931--11{,}082 \\
 & multi-pass companion line $\geq 8/10$ & 7/10
 (\S\ref{sec:cl-boundaries}) & \\
Canonical single-pass & decline lower $\geq 8/10$ & 8/10 & --- \\
Independent exam days & post-review mean higher $\geq 8/10$ &
10/10 & --- \\
Fashion (controlled) & above protocol twin $\geq 3/5$ & 3/5 &
--- \\
\bottomrule
\end{tabular}

\end{table}

\section{The Registered Program and Its Discipline}
\label{sec:program}
\label{sec:evidence-discipline}
This part states what the registered empirical program
established. Every number quoted in the chapters of this part is
drawn verbatim from a registered report or recomputed from the
archived run stores; the appendices carry the full worlds,
protocols, and round narratives. Three principal results
organize the remainder of this part: governed structural growth achieves
scores at or above a budget-matched standard architecture
where structure keeps arriving (the static sanity floor;
parity at matched uncertainty, App.~\ref{sec:growthvalue}), and network size follows demand in both directions
(\S\ref{sec:res-capacity}); a plastic model under governance
holds its educated principles longer than a frozen one
(\S\ref{sec:res-fidelity}); and growth governance is workable
only in its ex-post form, at a measured price
(\S\ref{sec:res-governance}).

\subsection{The evidence discipline}
\label{subsec:evidence-discipline}
Worlds, arms, metrics, and acceptance bars were frozen in
written gate files before any verdict data existed; design
iteration was confined to disclosed probe ledgers on burned
seeds disjoint from the verdict seeds, which no iteration ever
touched. Failure branches were written before the batteries ran
and are quoted unchanged: the registered failures in this paper
appear as measured boundaries, not as retrospective framing.
For the evaluative-learning and benchmark campaigns
(\S\ref{sec:eval-learn}, \S\ref{sec:cl-bench}), protocol
iterations superseded by an improved registered design are
archived, with their verdicts, in the experiment stores named in
the Reproducibility Statement (the archived supersessions
comprise a capacity-starved world setting, a three-step remedy
series on the one out-of-reach task, two review-format
iterations, and the pre-split terminal-exam reading (the twenty-task
full-data store); none is
cited as evidence in this paper); the paper reports each subject's
final registered design and its outcome, and the genuinely open
boundaries are stated in place (\S\ref{sec:rl-boundaries},
\S\ref{sec:cl-boundaries}).
Anomalies were autopsied before verdicts were written, and the
instrument corrections that resulted were applied to every arm
equally. Inference uses per-seed count criteria over the registered
seed set --- three deterministic seeds for the early E/S
series, five for every later campaign of the v9 program, and
three per arm for the 2026 campaigns (whose inferential claims
are correspondingly ordinal), five lives per arm for the
evaluative-learning batteries of \S\ref{sec:eval-learn} (ten
for the LunarLander gate battery), and five or ten seeds per
arm for the benchmark lanes of \S\ref{sec:cl-bench}, as
registered per lane --- each battery's bar as
registered --- and no null-hypothesis testing is claimed at
these $n$; per-seed values with medians and ranges accompany
every aggregate of the adjudicating batteries in
App.~\ref{sec:growthvalue}. All
experiments ran on CPU-class hardware with per-life wall-clock
recorded in the archived run manifests; every run directory
carries its configuration and content manifest, and reruns from
the same seed reproduce results bit-for-bit.

\paragraph{Third-party refereeing.} No self-written numeric
kernel certifies itself. Every hand-written forward --- the
base substrates and every grown transient state --- is refereed
against the official components of a widely used deep-learning
library (PyTorch~\cite{paszke2019pytorch}) by weight
transplantation, agreeing at $10^{-13}$ in
double precision (and at the single-precision floor on the GPU
referee --- library certification infrastructure; the campaigns
themselves ran CPU-only); hand-written gradients are refereed
against automatic differentiation. Accuracy certification is double-precision
only; single-precision compute sits behind an explicit
acknowledgment door whose refusal text states the measured
error envelope.

\subsection{The campaigns at a glance}
\emph{Mechanism validation.} The exactness and acceptance
suites verify that every growth operator is exact at
application and that budgets, gates, lineage, and holdout
quarantine are enforced (PASS; App.~\ref{sec:empirical}). The
historical depth route is retained as record
(App.~\ref{app:twodir}); the topology case study grows a
governed structure end to end (App.~\ref{app:topo}); the
attention program grades per-head growth
(App.~\ref{sec:aseries}); the self-knowledge series shows
self-reports are real while their exploitation remains open ---
a mixed verdict, recorded as such (App.~\ref{sec:sseries}).
\emph{Capability claims.} The generalization campaign covers
length extrapolation, compositional generalization, trigger
behavior, and the EWC comparison (App.~\ref{sec:gseries});
a real-data case study tests in-service learning on
prequential concept-drift streams
(App.~\ref{sec:streamcase}); the
growth-value campaigns establish the transformer comparison
(App.~\ref{sec:growthvalue}); the size-adjustability round
establishes the crossover result; and the fidelity program
establishes the education thesis under the corrected
instrument (App.~\ref{sec:guardprog});
an operating LLM drives the factory end to end
(App.~\ref{app:validation}).
\emph{Mechanism finding.} Four governed-gate rounds locate the
workable form of growth governance and its price
(\S\ref{sec:headline-results}); their observations close this part
(\S\ref{sec:empirical-laws}).

\paragraph{Registered campaigns of the depth axis
[registered 2026-07-25; executed 2026-07-25/26].} Five further
campaigns were registered here with their designs fixed before
execution; the registration text below is preserved verbatim as
the registration record, execution deviations are declared with
the results, and the measured outcomes --- together with the
optimization campaigns that followed them --- appear in
\S\ref{sec:campaigns-results}. E-3 \emph{in-service
whole-layer insertion}: a serving model receives insertions at
head/middle/end positions on both substrate families ---
claims: serve continuity with instant-exactness at the
insertion, and post-insertion trainability; within-run
before/after design plus a no-insertion twin; pass lines
pre-registered in the experiment's design document. E-2
\emph{autonomous shape-governed growth}: under a user
aspect-ratio floor with auto-compliance, every instant of an
autonomous run satisfies the shape law with zero gate
collisions. E-6 \emph{protection windows at growth events}:
per-region rate scheduling at growth instants, judged on
post-event stability at equal budget. E-4 \emph{exploration
economy}: dry-run pricing, bounded trials, and exact rollback
against unconditional adoption at the same budget. E-1
\emph{controlled compositional deepening} (author-designed):
processing-depth growth under the full prevention stack
against the additive-only regime and a tuned fixed baseline.

\section{Main Results}
\label{sec:headline-results}

\subsection{Capacity follows demand}
\label{sec:res-capacity}
Two independent comparisons carry this claim. Against a
budget-matched standard transformer on the generative
arriving-structure track, the governed growable model achieved
the better score on 5/5 verdict seeds in the original battery, and the result
survived two successively stricter governance regimes (4/5
under the silenced gate; 5/5 under incumbent--candidate
governance; App.~\ref{sec:growthvalue}). Against a
lifelong-trained fixed network of the same substrate family
under moderately but persistently growing data (classes
arriving for the whole life to $\sim$400), the size-adjustability
round established both directions of the claim: the network's
size followed demand --- born at parity (12{,}020 backbone
parameters; both arms' class heads grow identically with the
arriving inventory and are excluded from all counts here),
every seed ended $4.6$--$4.9\times$ larger through 15--19
promotions, each earned at audited checkpoints --- and the
grown model's settled stage-exam score overtook the fixed
network and stayed ahead on 4/5 seeds, at
$(50\,\mathrm{k},\ 17{,}614)$, $(200\,\mathrm{k},\ 27{,}411)$,
$(350\,\mathrm{k},\ 35{,}288)$, and $(550\,\mathrm{k},\ 51{,}358)$
rows and parameters respectively. The two seeds without a
clean end comparison are both disclosed: seed 2's end
checkpoint falls within one audit window of a promotion and is
excluded by the registered sufficiency rule (raw end pair
$0.222/0.265$; its settled mid-life comparison reads
$0.50$--$0.51$ vs $0.31$--$0.32$), and seed 1 ended behind
(raw end pair $0.126/0.143$) and is recorded as censored.
Fairness terms are stated with the result: sustained overtake
began at $1.47$--$4.27\times$ the fixed network's size (median
$\sim\!2.6\times$) --- the measured exchange rate of capacity
for score under identical training; and in the earlier
extreme-paced round the same machinery proved size
adjustability from a $2.4\times$ smaller birth (5{,}020
$\rightarrow$ $\sim$40{,}000 parameters, 5/5 seeds).
Figure~\ref{fig:crossover} shows the full trajectories: both
arms compared only after consolidation to a measured plateau
at every stage.

\begin{figure}[t]
\centering
\includegraphics[width=\linewidth]{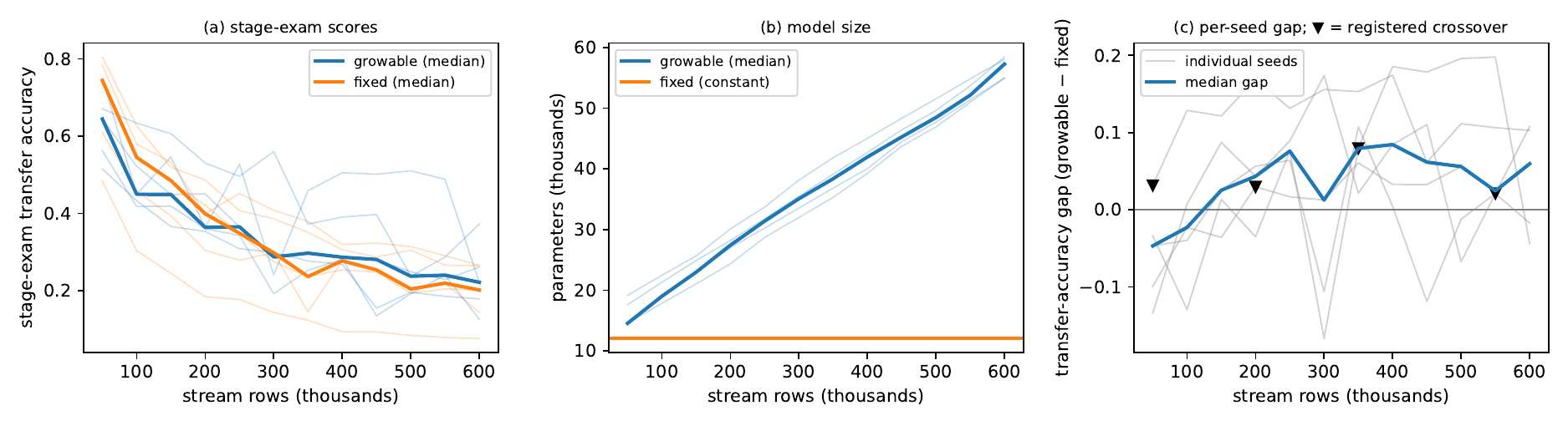}
\caption{Size adjustability under persistently growing data
(moderate round; five verdict seeds faint, medians bold).
(a)~Stage-exam transfer accuracy of the growable and the
lifelong-trained fixed arm; both decline as the class
inventory grows to $\sim$400. (b)~The growable arm's size
follows demand from parity (12k) to $\sim$57k backbone
parameters while the fixed arm's backbone is constant. (c)~Per-seed score gap
(growable$-$fixed); markers denote the registered settled
crossovers. Stage scores are taken only after both arms
consolidate to a measured plateau.}
\label{fig:crossover}
\end{figure}

\subsection{Stability from governance, not immobility}
\label{sec:res-fidelity}
The fidelity program's final round gave both arms the same
pedagogical curriculum, with a solidity gate (mean of the last
three checkpoints $\ge 0.90$) enforced on the calibration
seeds; on the verdict seeds each arm is anchored to its own
measured graduation level (App.~\ref{sec:guardprog}, seed
hygiene). After graduation, under surface drift and
scheduled temptation, the plastic arm held its educated canon
at least as long as the frozen arm on every seed:
time-to-degradation favorable on 5/5 under the registered
$\ge$ rule --- strict wins on 3/5 (per-seed ratios
$2.9\times$, $7.0\times$, $2.25\times$) with two seeds tied
at the 1{,}000-row censoring floor, the frozen arm sitting at
that floor in 5/5 --- decline-slope advantage on 3/5, recovery
advantage on 4/5; all under the corrected time-axis instrument
(2/5 at margin $0.05$ under the original level instrument,
App.~\ref{sec:guardprog}).
On the re-education cost frontier, the frozen arm's periodic
rehearsal --- which must sweep its full canon --- never
restored its decline at any logged cost, while the plastic
arm's remediation rides its always-on plasticity
(App.~\ref{sec:guardprog}; Figure~\ref{fig:guard}).

\begin{figure}[t]
\centering
\includegraphics[width=0.72\linewidth]{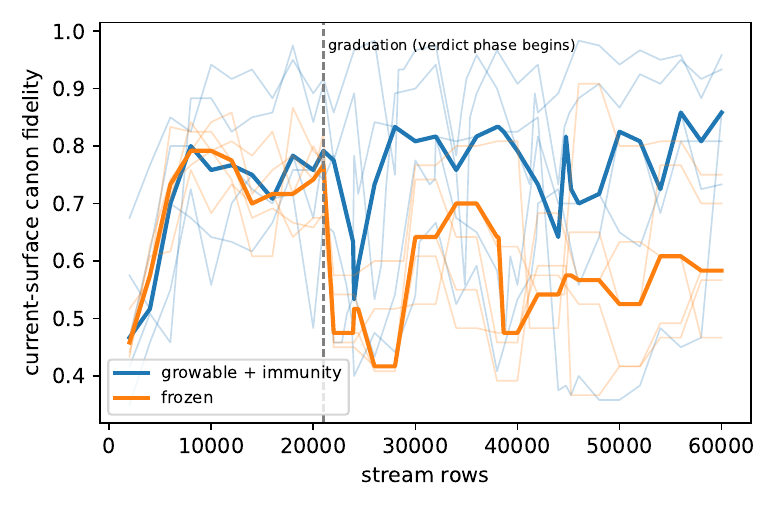}
\caption{Fidelity on current surfaces over life (five seeds
faint, medians bold). Both arms are educated under the same
curriculum to graduation (dashed line); afterwards the frozen
arm's fidelity collapses under drift and temptation while the
governed plastic arm holds and recovers.}
\label{fig:guard}
\end{figure}

\subsection{Growth governance is ex-post}
\label{sec:res-governance}
Four registered rounds located the workable form of growth
governance (Table~\ref{tab:gatearc}). The naive gate satisfied the growing-track
criterion (3/5 against the fixed twin) but made 8/8 regretted adoptions on the
stationary public track; the silenced gate governed it to 2
breaches and thereby suppressed the growing-track response
(1/5).
The
mechanism discovery followed: with training gains controlled
away by a train-only control arm, the capacity effect is
invisible at any affordable ex-ante probe horizon (pain-subset
paired differences $0.009\pm0.017$ against a long-run value
near $0.1$), and in the exam world probing is actively
self-injuring (gate cross-entropy rising from $4.4$--$8.2$ to
$10.5$--$22.8$ across eight instrumented events). Proof therefore moves ex post: the incumbent serves
untouched while a grown candidate trains on the same rows and
is promoted only at the first audited checkpoint at which it
scores higher,
with a demotion window guarding the promotion --- trials are
free (the serving trajectory is bit-equal to never-growing
until a promotion) and wrong buys never touch service. Its
price is the governance frontier: proof latency is exactly the
head start a lifelong-trained incumbent keeps. A battery
autopsy added the rollback-anchor requirement: an open demotion
window's anchor must be immutable, or a diverged incumbent is
promoted over the last good state (observed at serving NLL
$11.38$; cured to $0.65$ by spec enforcement).

\begin{table}[t]
\centering\small
\caption{The governed-gate arc (details
App.~\ref{sec:growthvalue}). Rounds 1--2 and 4 report the
generative track's C-N1 against the lifelong-trained fixed
twin; round 3 ran in the open-inventory exam world, whose
criterion is C-G1 ($^{\dagger}$the gate un-starved there:
adoptions $0 \to 2$--$3$ per seed). $^{*}$The same round's
transformer comparison is 5/5. Appendix-internal round labels
map to rows 2--4 here; X2b shares row 4's governance form in
the exam world. Each round's observations appear in
\S\ref{sec:empirical-laws}.}
\label{tab:gatearc}
\begin{tabular}{@{}llll@{}}
\toprule
Round & Gate design & Growing track & Stationary regret \\
\midrule
1 & naive ex-ante & 3/5 & 8/8 regretted \\
2 & silenced ex-ante & 1/5 & governed, 2 breaches \\
3 & shadow audit (exam world) & C-G1 0/5$^{\dagger}$ & (no
stationary track) \\
4 & incumbent--candidate & 1/5$^{*}$ & non-regret 4/5 \\
4b & incumbent--candidate (exam world, X2b) & C-G1 0/5; C-G2 3/5 & (no stationary track) \\
\bottomrule
\end{tabular}
\end{table}
\subsection{The audit tally}
\label{sec:res-tally}
The correspondence table (Table~\ref{tab:evidence}) closes the
accounting. Of its 64 claim rows, 35 are supported, 20 are
boundary results whose limiting observation is named in the row,
4 remain open FAIL branches, 3 are diagnostic exhibits, 1 is a
historical motivation row, and 1 is an auxiliary demonstration;
23 further rows record empirical observations (20 register-tagged
observation rows and 3 scope or regime observations measured in
claim position). Every number in this Part quotes a row.

\subsection{The depth-axis and continual-learning campaigns}
\label{sec:campaigns-results}

Fifteen campaign fleets (E-1--E-4, E-6/E-6b, E-7--E-15;
E-5 was excluded by design before execution) executed in
2026-07 under the registered
discipline: pass lines fixed before each fleet, judges reading
run files only, write-once run directories, and every number
below traceable to a committed run directory
(App.~\ref{app:newcampaigns}). Two library defects surfaced by
these campaigns' own audit trails were fixed under change
control mid-program (the affected arms were re-run on the fixed
library and the superseded rows retained); the code anchor for
every adopted row is the tag \texttt{depthgrowth-v1.7}. Where a
registered bar was not met, the analyzed cause accompanies the
result; the emphasis throughout is on what the measurements
establish.

\paragraph{Mechanism safety.} A serving model receives
whole-layer insertions at head, middle, and end positions on
both substrate families without service interruption: the probe
stream records zero failed serves across every life, and the
served function at each insertion instant is bitwise unchanged
(36/36 instants; inserted layers train thereafter, 36/36
aliveness). The controlled-deepening capstone repeats this in
the presence of the full prevention stack --- identity insertion,
the shape law with auto-compliance, and event-scoped protection
windows: training proceeds through three mid-life insertions in
every seed with no collapse, and the deepened line ends at
parity with an additive-only twin, and with a born-deep control
in 2/3 seeds (the remaining seed reads $1.37\times$ the
born-deep control's loss).
Pure-$\delta$ instants are bitwise; whole acts that include an
aspect-law widen rider preserve the function to machine epsilon
(deviations of $2$--$3\times 10^{-16}$, the summation-order
footprint of exact zero-extension). The post-insertion
\emph{transient} is real: descent immediately after insertion
met its registered window in only 12/36 instants; the transient
resolves with training budget, and its management is the subject
of the governance campaigns.

\paragraph{Governance.} Under a user aspect-ratio floor with
auto-compliance, autonomous plan-driven growth satisfied the
shape law at every instant with zero gate collisions in every
seed, the floor provably binding (the refuse-mode preview marks
the crossing steps in advance). The governed arm's quality
transient resolved within the pre-registered extension horizon
(3/3 in-band at 50 batches). Protection windows at growth events
damp the post-insertion spike (28/36 instants) at zero measured
cost on healthy regions (9/9); where late insertions did not
re-enter the twin band, the horizon probe attributes the residue
to remaining training budget rather than to the window
convention.

\paragraph{Economics.} The control surface prices before it
spends: dry-run quotes matched ledger actuals exactly across
every adopted act, and every bounded trial restored the organ
bitwise (including all rejections), so structural search leaves
zero residue. Its one measured caution: a six-step probe
under-prices delayed payoffs (one seed rejected the deepening
its task needed). Against retraining from scratch after a
mid-life drift, in-service growth reached the retrained
specialist's quality at 58--64\% of its post-drift compute in
2/3 seeds, serving throughout while the scratch alternative has
a service gap by construction. When the world returns to an
earlier regime, models that had adapted in service re-recover
old competence in 2--3 batches against 20--29 for a fresh model
--- relearning savings of $7$--$10\times$
(Figure~\ref{fig:savings}): the old knowledge
demonstrably survives inside the network even where the served
function had to track the drift.

\paragraph{Continual learning.} Under a strict sequential domain
switch with no environmental replay, catastrophic forgetting
reproduces at full strength (old-domain probe degradation
$552$--$1{,}362\times$ while the new domain is learned). Across the
treatment ladder (the doctrine of \S\ref{sec:cl}),
\emph{review} is the load-bearing medicine:
only arms that re-input old knowledge from the model's own store
retain it (ratios $72$--$300$ across the forgetting-ladder
review arms), while
protection windows alone and added capacity alone do not treat
this regime --- the measured boundary that fixes each tool's
indication. The review regularities --- called ``laws''
informally in the campaign records --- each measured in its
own campaign:
\emph{targeted} selection beats random at equal volume and equal
diversity (the earlier contrary reading traced to a repetition
confound); review \emph{timing} dominates cost --- a course
started after the new domain stabilizes achieves comparable
retention at roughly one tenth the new-domain cost of always-on
replay; a \emph{bounded} course beats open-ended replay on both
axes; a completed course \emph{erodes} under continued
new-domain training; and \emph{spaced repetition} counters the
erosion, with the number of spaced courses acting as the
retention dose (one/two/three courses:
$341$--$800$ / $252$--$438$ / $73$--$149$ in the controlled
course-count family, monotone in every seed;
Figure~\ref{fig:frontier}). Adding growth to a review
course yields the best new-domain quality of any treated arm at
weaker retention --- structure carries capacity, data carries
memory. Plasticity-health instruments stayed at their healthy
values in every arm and every lifetime (dead-unit fraction
$0.0$; effective rank $5.6$--$5.8$ of $6$): the forgetting
measured here is interference, not loss of plasticity.

\paragraph{The depth axis.} On an iterated-composition family
with the composition depth $k$ as the controlled variable and
one parameter cap for all arms, the ordering between the growth
directions reverses with $k$ (Figure~\ref{fig:depthsep}): at
$k{=}2$ depth wins in 1/3 seeds; at $k{=}4$ and $k{=}6$ the deepening arm wins in 3/3 ---
while holding 58\% of the width arm's parameters. Depth is the
parameter-efficient axis exactly where composition is real, and
unnecessary where it is not; the formally registered separation
band was not met (its conjunction included the strategy arm),
and the axis-\emph{selection} question resolved separately: raw
probe gains and per-parameter efficiencies carry no axis
information, while the two-horizon slope does --- under slope
pricing the strategy adopts deepening first and matches the pure
deepening arm's quality (3/3 on both registered lines at
$k{=}6$). The additive family's own boundary was measured as
registered: on a three-level smooth composition at this scale it
never saturates --- the depth advantage requires genuinely
iterated composition, which is precisely what the $k$-mapping
establishes.

\begin{figure}[t]
\centering
\includegraphics[width=0.72\linewidth]{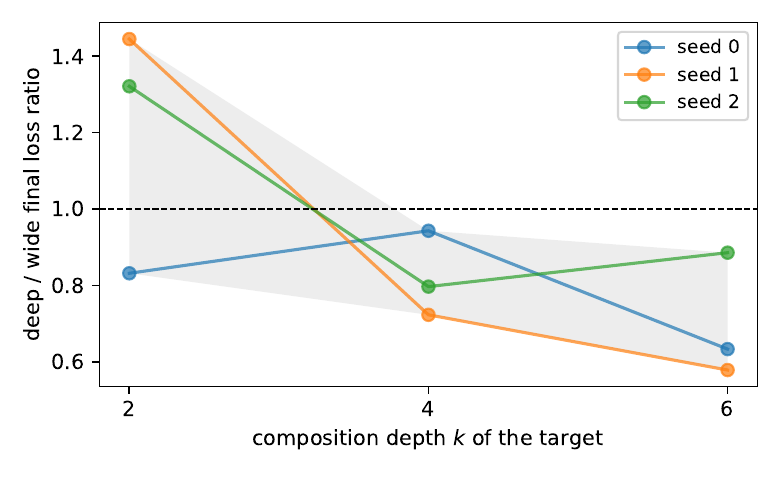}
\caption{Depth-separation mapping (three seeds, equal parameter
cap; shaded band $=$ seed range). The deep/wide final-loss ratio reverses with the target's
composition depth $k$: above parity at $k{=}2$ for most seeds
(depth wins in 1/3), and below parity for every seed at $k{=}4$
and $k{=}6$ --- with the deepening arm holding 58\%
of the width arm's parameters. Data: the depth-mapping
campaign's run directories.}
\label{fig:depthsep}
\end{figure}

\begin{figure}[t]
\centering
\includegraphics[width=0.8\linewidth]{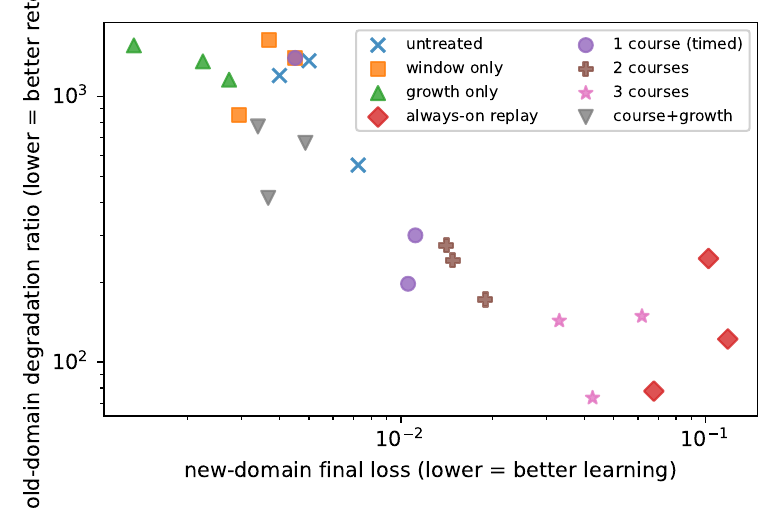}
\caption{The retention / new-learning frontier under a strict
sequential domain switch (both axes log; lower-left is better).
Untreated, window-only, and growth-only lives forget at
$10^{2.7}$--$10^{3.2}$; review courses move lives to the
frontier, with the number of spaced courses acting as the
retention dose and course-plus-growth taking the best
new-domain corner. Data: the forgetting-family campaigns' run
directories.}
\label{fig:frontier}
\end{figure}

\begin{figure}[t]
\centering
\includegraphics[width=0.72\linewidth]{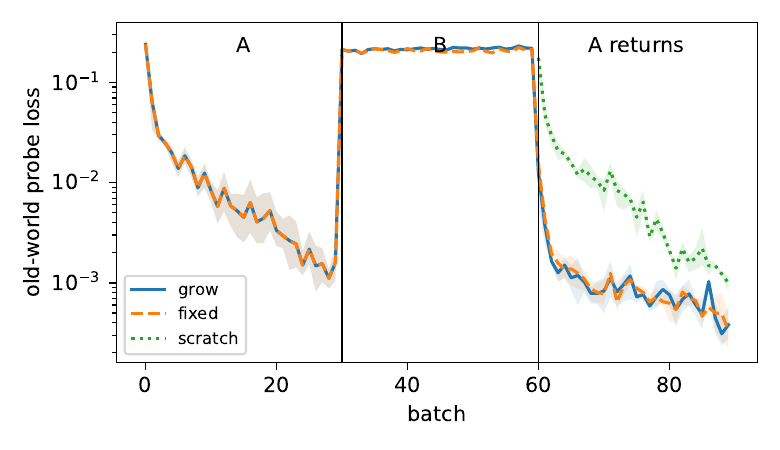}
\caption{Cyclic drift (A$\to$B$\to$A; seed mean, shaded
bands $=$ per-arm seed range). During B the
old-world probe rises for every arm (the served function must
track the drift); when A returns, lives that adapted in service
re-recover in 2--3 batches against 20--29 for a fresh model ---
the relearning-savings measurement of retention. Data: the
cyclic-drift campaign's run directories.}
\label{fig:savings}
\end{figure}

\section{Empirical Observations and Threats to Validity}
\label{sec:empirical-laws}
The observations below are extracted from the campaigns; each
is stated once, with one measured number and its registered
artifact; the register tags L1--L20 are retained as archival
identifiers. These are observations from the tested worlds ---
not laws, and not claimed beyond the conditions under which
they were measured. The campaign texts retain their full
derivations.

\subsection{The observations}
\label{sec:laws-digest}

\emph{Measurement --- what an instrument can and cannot
say.}
(L1) Capacity binds on exams and is invisible to online
scores: the same lives read $0.64$ vs $0.44$ on exams and tie
prequentially (App.~\ref{sec:growthvalue}).
(L8) Verify the instrument before the theory: many apparent
failures die in the instrument --- one round's escalation
counter matched an invented event name and its unit test
green-lit the bug (App.~\ref{app:x2}).
(L10) The metric decides governability: the adoption gate that
starved on exact-match exams adopted confidently on
per-sample-informative NLL, $0$ vs $2$--$8$ events per seed
(App.~\ref{sec:growthvalue}).
(L14) Audits are blind off their metric: a $+3{,}000$-row
bucket-CE audit kept a widen whose late-third exam score fell
$0.04$; in the incumbent--candidate round the same blindness
passed promotions and demotions that harmed exams
(App.~\ref{app:x2b}).

\emph{Curriculum --- how lifelong education behaved here.}
(L4) Review dose is non-monotonic: doubling a drill moved its
target $0.20 \to 0.45$; six-fold collapsed it to $0.17$ and
broke a neighboring axis (App.~\ref{sec:guardprog}).
(L5) Exposure economics, not capability, binds in-curriculum
learning: a skill drillable to $1.0$ in isolation plateaus
near chance in-curriculum (App.~\ref{sec:guardprog}).
(L6) Plasticity is favored on live-world trajectories and
disfavored by minimum statistics: the transition dip is the
price of tracking
(App.~\ref{sec:guardprog}).

\emph{Governance --- what it cost to govern change here.}
(L2) Small-scale gate outcomes are dominated by probe
economics; removing self-injuring probe training cut life
wall-clock $43 \to 17.4$ s unchanged in verdicts
(App.~\ref{app:x2}).
(L3) An instantaneous adoption gate cannot see lifetime
consequences: the one tolerated adoption decayed $0.57 \to
0.37$ long after its audit passed
(App.~\ref{sec:growthvalue}).
(L7) Without immunity, defection tracks the incentive:
$0.32$--$0.37$ inside episodes against a computed $+0.41$
(App.~\ref{sec:guardprog}).
(L9) Skip-based immunity faces a timing dilemma; re-anchoring
recovers about half the loss at $11\%$ cost, and the dilemma
recurs in every governor with a horizon
(App.~\ref{sec:guardprog}).
(L11) A non-growing world needs no growth; stationary harm
belongs to an unsilenced governor ($8/8$ regretted adoptions
on noise) (App.~\ref{sec:growthvalue}).
(L12) Governors have regimes: the calm throttle engaged
cleanly at zero era-adaptation cost yet cost fidelity $0.67
\to 0.51$ in a world of permanent drift
(App.~\ref{sec:guardprog}).
(L13) The marginal value of capacity is ex-ante unobservable
at affordable probe horizons: pain-subset paired differences
$0.009 \pm 0.017$ against $\sim 0.1$ nats long-run
(App.~\ref{app:x1}).
(L15) The price of silence: governing stationary regret away
cost the growing track its margin, C-N1 $3/5 \to 1/5$
(App.~\ref{app:x1}).
(L17) The governance frontier: proof latency is exactly the
head start a lifelong-trained incumbent keeps --- every
governed round bought silence or non-regret at the cost of
the within-life comparison \emph{within gate-round-length
lives} (30--40k rows, about one promotion window), while the
transformer comparison held $5/5$, $4/5$, $5/5$ across rounds
(App.~\ref{app:x1b}). The 600k-row size-adjustability lives
bracket the same frontier from the other side: capacity pays
in proportion to the trained life it has left (L16, L19).
(L18) Rollback anchors must be immutable until resolved:
clobbering an open demotion shadow promoted a diverged
incumbent at serving NLL $11.38$; spec enforcement cured it
to $0.65$ (App.~\ref{app:x1b}).
(L19) A serialized candidate pipeline buys $\sim$1 promotion
per 30--40k rows; worlds whose demand grows faster outrun the
governor (App.~\ref{app:xsat}).

\emph{Capacity and time --- when capacity paid here.}
(L16) Birth capacity compounds; mid-life capacity pays only
over the life it has left --- un-starving the gate narrowed
the paired oracle gap ($0.046 \to 0.034$ mean; median score
$0.5014 \to 0.5653$) and could not close it
(App.~\ref{app:x2}).
(L20) Parameter demand requires irreducible information
density: on a smooth-regression family a $20\times$ parameter
range scored flat (NLL $1.38/1.36/1.34/1.38$) --- discrete
memorization binds parameters; smooth regression does not
(App.~\ref{app:xsat}).

\subsection{Scope and threats to validity}
\label{sec:threats}
These measurements were taken at kilobyte-to-tens-of-kilobyte
model scale, on synthetic registered world families (token
inventories, drifting surfaces, growing class sets, generative
mixtures) plus a real-data case study, the MNIST-family
benchmark lanes and the gymnasium environments of
\S\ref{sec:eval-learn}/\S\ref{sec:cl-bench}, largely on the
reference-network family (the insertion campaign runs on both
families). The inferential standard is per-seed count
criteria over the registered seed sets (three seeds for the
early E/S series, five for every later v9 campaign, three per
arm for the 2026 campaigns, and the
\S\ref{sec:eval-learn}/\S\ref{sec:cl-bench} batteries as
enumerated in \S\ref{subsec:evidence-discipline}); no
significance testing is claimed at these $n$, and per-seed
values accompany every aggregate, in the appendices or in the
archived run stores of \S\ref{sec:eval-learn}/\S\ref{sec:cl-bench}. External validity is claimed only as
scale-independent regularities --- the observations of
\S\ref{sec:empirical-laws} are stated strictly within the
tested families and scale; any wider reach is an open
measurement. Exam-based
instruments are noise-limited at small item counts, and audits
certify only on the axes they measure; both limits are
themselves recorded observations. The theory-level limitations are
discussed in the Discussion; this subsection bounds the
measurements only.

\part{Discussion and Conclusion}

\section{Discussion}
\label{sec:discussion}

\paragraph{Scope of the 2026 campaigns.} The campaign worlds
are seeded synthetic streams at moderate scale (kilobyte-class
organs, thousands of rows), three seeds per arm, largely on the
reference-network family; the axis-signal result is measured at
one composition depth in one world; retention statements are
within-life ratios; registered constants (bands, floors, window
and course conventions) are policy values, not derived
quantities. Two library defects found by the campaigns' own
audits --- and the tests-first fixes and re-acceptance that
followed --- are part of the public record; we regard the audit
chain that caught them as part of the method.

\paragraph{Why smallness is load-bearing.} Every guarantee in
\S\ref{sec:theory}--\ref{sec:lifecycle} has a price that shrinks with
the model: full-replay retraining, dense versioning, per-version
stores, speculative growth with gated adoption --- all are affordable
at kilobytes and would be luxuries at gigabytes. This is not an
accident of implementation but the architectural consequence of the
brain/extension division: once language, context, and feature
extraction are the brain's job, what remains --- a mapping from a
handful of features to a judgment --- genuinely fits in kilobytes.
Smallness is what lets \emph{everything} stay soft.

\paragraph{What governance buys.} Because every change is gated, the
operator may act carelessly --- noisy batches, contradictory labels,
speculative growth --- and the served model can only improve or stay.
This inverts the usual posture of continual learning, where each
update is a risk to be managed; here risk is absorbed by the
lifecycle, which is precisely what makes \emph{LLM-driven} operation
safe without a human adjudicating each step.

\paragraph{Observability as a lifelong record.} Recent
interpretability work reads rich internal structure out of
frozen model snapshots --- most recently a privileged,
limited-capacity ``workspace'' of verbalizable representations~\cite{tc2026workspace}. The contrast with this system is one of
tense: a snapshot method observes what a network \emph{is};
a governed lifecycle records what it \emph{has been} --- every
head's loading, entropy-band occupancy, and growth event across
its service life, alongside every adoption decision. Lifelong
instrument trajectories are, to our knowledge, an observable no
frozen-snapshot analysis can reconstruct, and they are the
concrete form self-knowledge takes here: distribution
statistics as first-class, recorded objects.

\paragraph{Objections the axiom invites.} Three are worth
answering here. \emph{Is target handing a stop-gradient?} No
parameter is shielded from adaptation: every scale trains at every
step, and cross-scale pressure arrives as a target rather than a
backpropagated gradient --- the axiom governs trainability, not the
routing of credit. \emph{Are budgets freezing?} They are bounded
structural plasticity: caps are runtime policy, refusals are logged
outcomes, and both are adjustable within the model's lifetime ---
unlike architectural omission, which is unbounded, silent, and
permanent; we concede the plasticity they bound is real. \emph{Is
the committed version a frozen artifact?} Between promotions it is
immutable by design --- but the axiom forbids freezing of
\emph{trainability}, and training never stops in the working state;
a gate that refuses indefinitely is behaving correctly (no
candidate has scored higher on reality's evidence), and drift
re-opens the comparison by changing the evidence.

\paragraph{Where governance is the more suitable approach, and
where freezing suffices.} The comparison with immobilization is
scenario-dependent, and the boundary we draw is the one the
continual-learning literature itself draws. The surveys frame
every family as a position on the stability--plasticity trade-off~\cite{parisi,delange}, and the canonical taxonomy ---
regularization, replay, parameter isolation --- is developed for
task-incremental settings with clear boundaries~\cite{delange};
the task-free setting was introduced precisely because most such
methods depend on boundaries that data streams do not provide~\cite{aljundi}. Loss of plasticity under continued training is a
measured phenomenon~\cite{dohare}, and per-task frozen columns
spend capacity linearly while never revisiting it~\cite{prognets}. On this map, the governed regime is claimed for
the settings those observations describe: schema growth --- a
frozen interface cannot represent a newly observable input at any
weight setting, so the error floor is informational, not a
training artifact (E11b measures exactly this null: the
fixed-topology arm against the grown one) --- sustained drift,
and boundary-free change, where a gate needs only its holdout
stream. Conversely, where the domain is fixed, the distribution
stationary, and the schema closed, freezing is cheap, safe, and
sufficient~\cite{lora} --- the audit machinery here buys little
and costs its overhead. We claim the governed regime for the
first family of scenarios, not in general; the matched-parameter direct comparison under capacity pressure remains open
(\emph{Limitations} below).

\paragraph{The role of depth.} The reference substrate,
without composition blocks, has layer depth one
(Prop.~\ref{prop:rhostatus}). Depth earns its cost on
targets defined over intermediate quantities that must first be
constructed from the input: there, deep networks of modest size
compute functions that no shallower network can approximate unless
its width grows exponentially in the depth or in the input
dimension~\cite{telgarsky,eldanshamir}, while superpositions of
the single-hidden-layer form remain dense in the continuous
functions on compacta~\cite{cybenko} --- a question of required
size, not of attainable class. Refinement faces the complementary
target: a residual, largely stripped of compositional structure by
its host, for which the greedy stagewise fitting of an additive
expansion is the classical form~\cite{friedman,cascadecorr} and
carries dimension-independent approximation rates~\cite{jones,barron} --- in this language, the additive probe's
zero-attachment gradient (\S\ref{sec:compositional}) is greedy
approximation's one-step gain, the stopping quantity of matching
pursuit~\cite{mallatzhang}; and the depth separations that mark
the compositional class are the classical ones~\cite{telgarsky,eldanshamir}. The two demands also differ in timescale: the
number of representational stages a domain requires varies slowly,
whereas the content within stages is the locus of drift. The design
therefore fixes the serving chain's stage count at creation
through the substrate choice (\S\ref{sec:algebra}; the released
hosts include a transformer encoder) and never inserts a stage
into an operating composition --- a global disturbance for a
property that rarely changes. What the compositional operator adds is the one depth
channel this argument always permitted: scope-interior
composition ($\delta$, \S\ref{sec:compositional}), entered
exactly when a scope's own signals certify that the additive
route has stopped paying, and adopted only through the gate.
Lifelong plasticity thus divides by timescale: detail and
interface to $\rho$, $\omega$, $\sigma$; composition depth to
$\delta$ on the delayed-payoff slope signature; structural
obsolescence to gated re-founding ($\Phi$). 

\paragraph{A mechanical view of learning systems.} The reading of
\S\ref{sec:algebra} is an instance of a duality with a
distinguished pedigree: computation is physical~\cite{landauer},
physics is computable~\cite{feynman82}; and, conversely, learning
dynamics is a mechanics: energy functions over network states~\cite{hopfield}, inertial methods as damped mechanical systems~\cite{polyak}; and fitting a network is a parameter-identification
inverse problem --- backpropagation being the discrete counterpart
of the adjoint-state method by which such gradients are
classically computed~\cite{lecun1988,plessix,griewank}, a
correspondence that becomes literal in the optimal-control
formulations of deep learning~\cite{weinane}. In this duality the whole
lifecycle reads as an adaptive structure under load: it relaxes
level by level, grows a degree of freedom where relaxation
persistently fails ($u_j$), and realizes only those virtual
changes that lower a held-out measure --- the gate as a
virtual-work test~\cite{lanczos}. We use the correspondence as a source of
precision, not as a claim of physical equivalence.

\paragraph{Limitations.} At the tested scale the program
demonstrates \emph{safety and exactness}, not necessity or
superiority; the matched-parameter direct comparison against the wider
immobilization family (progressive columns, frozen backbones
with adapters) on a capacity-pressing task is the decisive open
experiment --- the EWC comparison itself is on record
(App.~\ref{sec:gseries}, supported 3/4). The evidence is kilobyte-scale, and the
core program is synthetic by design (hidden-law scenarios chosen so
ground truth is exact, partitions provably disjoint, and every run
deterministic); the real-data stream study
(App.~\ref{sec:streamcase}) tests the grown attention
substrate in prequential stream learning on real data, with
its paired growth ablation and attribution verdict reported
in place. We make no claims beyond that
scale, and E10 shows the axiom's noise cost when governance is
thin. The value-of-growth predictions are partly
unresolved: our lock-in scenario failed to bind, $\Phi$'s trigger is
open, and self-study does not yet convert real self-knowledge into
gains. The gate certifies against its holdout stream --- a
representativeness assumption inherited, not eliminated --- and
its repeated reuse across adjudications is the regime whose hazards
the reusable-holdout analysis formalizes~\cite{dwork}. Finally, the
LLM-operated validation covers the factory surface; the growth
surface is script-validated and exposed identically, but an
LLM-driven growth campaign is future work, as is the exploitation of
the self-knowledge signals that the program (App.~\ref{sec:empirical})
shows are real.

\section{Conclusion}

The mechanisms of this paper descend from one axiom and one
inherited idea --- the multiscale thought of physics and of the
multiscale finite element method: resolution granted where the
problem demands it. Their union --- one axiom, one gate, one audit
discipline --- occupies, so far as we are aware, an unoccupied
cell; its deployment-side relative is champion--challenger
practice, which shares the shadow-evaluate-promote shape without
the uniform gate over \emph{structural} operators. The account closes the theory's last exemption: the attention
mechanism itself now lives under the axiom --- its heads grown
per head, exactly and under the gate, its distributions
accountable to a local discipline whose gradient is certified
and whose locality is enforced by measurement. The empirical
program applied its registered discipline throughout: the mechanics verified
exactly; local restoration proved more cost-effective than
global repair where
collapse was induced; the instruments led the loss as a
median tendency (its false-alarm price measured); and the
strong form of demand-tracking growth was refuted at this scale
while its payment under genuine starvation was confirmed ---
the open item is an acceptance probe that does not injure what
it measures, and it is named, not hidden. We proposed treating lifelong plasticity as
definitional --- no parameter, at any scale of structure, ever
frozen --- and showed the position is coherent: the freezes it forbids include the silent one
(fixed topology as structural freezing by omission), and the
stability it seems to forfeit is recovered by governance ---
exactness at application, adjudication by held-out reality --- rather
than by immobility. The released system realizes the position end to end: a substrate whose structure is grown rather than
designed; an operator algebra that grows width, interface,
scale, and composition --- refining where the data demand finer
description, deepening where they demand transformation --- up to
re-founding the whole model,
one gate governing learning and growth alike, and a tool protocol
under which a general LLM safely operates, teaches, and grows its own
extension. The pre-registered program returns a map rather than a claim of general superiority: the mechanics verify exactly; where growth pays at this
scale is charted, several predictions await scenarios that
bind, and the model's self-knowledge proves real ahead of our ability
to exploit it. We take the map itself as the contribution --- now including its
newest entry, a capability bar the compositional direction
did not clear at this scale. The axiom, applied to its own
algebra, demanded a second direction of growth; the system
supplies it,
with the direction read from the model's own signals and
adjudicated by the same gate. A theory whose first principle is
that every change must survive reality's gate should expect to
publish its refusals --- and to grow. The same contract now extends to evaluative signals --- structure preference and policy optimization under one gate --- and the core method holds in its adjudicating lanes on the standard continual-learning benchmark family, with its dose boundaries reported: governed growth preserves the ability to keep learning along long task sequences, and converts announced review into recovered competence.

\phantomsection\addcontentsline{toc}{section}{Reproducibility and AI-Assistance Statement}
\markboth{Reproducibility and AI-Assistance Statement}{Reproducibility and AI-Assistance Statement}
\section*{Reproducibility and AI-Assistance Statement}

The complete system, acceptance suites, pre-registered specs,
drivers, and raw results exist as frozen, version-tagged
repositories (the system is called SoftModel in prose);
the library, the system, and the companion experiment stores
are public at
\url{https://github.com/continual-learning-models}
(repositories \texttt{GrowableSoftModel},
\texttt{SoftModelSystem}, and \texttt{paper-experiments}), and
the remaining artifacts are available to editors and
reviewers on request. Every number in this paper regenerates
from committed artifacts: \texttt{tests/acceptance/run\_acceptance.py} (mechanics), the
two-realizations program from \texttt{pytest} over
\texttt{tests/simulation} and \texttt{tests/capability} plus the
golden fixtures, with the App.~\ref{app:twodir} artifacts at
\nolinkurl{tests/logs/simulation_adaptive.jsonl} and
\nolinkurl{tests/logs/comp_smoke.jsonl},
\nolinkurl{experiments/*/driver.py} with committed \texttt{spec.md}
files (E-, S-, A-, and G-series), all seeded
(seed values in each run's resolved configuration) and CPU-only
(pure-\texttt{numpy} substrate; minutes per run on a laptop
core; specs and criteria are commit-dated in the repositories) --- the attention build adds its certification script
(mask-variance, degeneracy, and dose-response numbers recompute
from one committed file) and two campaign repositories in which
every run is a write-once directory carrying its resolved
configuration, provenance, and SHA-256 manifest, indexed by a
rebuildable registry; the evaluative-learning and benchmark batteries of
\S\ref{sec:eval-learn}/\S\ref{sec:cl-bench} reproduce from
the caller-side companion repository \texttt{paper-experiments}
(store directories \texttt{rl-supplement} and
\texttt{cl-benchmark}), archived with the library and
system repositories: frozen tags (\texttt{rl-unified-v1};
\texttt{sms-rl-unified-v1}) and
package versions pinned in each store's \texttt{PIN.txt} ---
the tags name pinned development states whose full history is
retained in the author's development archive, and the companion
repositories are content-identical snapshots of those states
(with dataset SHA-256 checksums for the benchmark lanes),
registered designs and pass criteria in the numbered design
documents archived under the companion repository's
\texttt{registrations/} directory, frozen instrument settings in \texttt{LINES.md},
per-run JSON result stores under \texttt{runs/}, and the
verification logs
(\texttt{QUASISTATIC\_FULL\_VERIFICATION.md},
\texttt{RL\_LIBRARY\_COMPARISON.md}) under the library's
\texttt{tests/logs} (registration order is evidenced by the development
archive's commit history, available to editors and reviewers
on request --- there the benchmark store's history spaces
registration, probe, and arms realistically, while the
reward-learning store batches some registration and result
commits minutes apart, evidencing sequence rather than
wall-clock precedence; the registered designs themselves ship
verbatim under \texttt{registrations/}). Drafting of this
paper was assisted by large language models operated by the author;
all claims were verified against the committed artifacts by the
author, and the evidence table (Table~\ref{tab:evidence}) is the
binding contract between text and record. Availability: the
evaluative-learning and benchmark evidence
(\S\ref{sec:eval-learn}/\S\ref{sec:cl-bench}) exists in full
in the companion repository \texttt{paper-experiments}, public
together with the library and system repositories at the
address above; the earlier sections' campaign
artifacts (the E-, S-, A-, and G-series experiment trees, the
attention-build certification, and the two campaign
repositories) are retained in the author's development archive
and are available to editors and reviewers on request.

\paragraph{The 2026 campaigns.} Library at tag
\texttt{depthgrowth-v1.7}; campaign drivers, judges, reports,
and write-once run directories under the campaign
\texttt{experiments/} tree (development archive; see the
availability note above) (one directory per campaign;
append-only judge outputs). Data-module SHAs: insertion /
shape / economy / windows \texttt{8fe72943\ldots}; capstone
\texttt{518aedae\ldots}; dual-axis \texttt{5d27422a\ldots};
forgetting-family \texttt{e9c3696f\ldots}; drift-economics
\texttt{72b81595\ldots}; cyclic-drift \texttt{401bc9a2\ldots};
depth-mapping $k{=}2/4/6$ \texttt{c1bf43ce\ldots} /
\texttt{8318872c\ldots} / \texttt{590df824\ldots}. Every run
directory records \texttt{git describe} of both repositories at
run time.

\clearpage
\phantomsection

\clearpage
\phantomsection
\appendix
\part*{Appendices}
\titleformat{\section}
  {\normalfont\fontsize{16}{19}\selectfont\bfseries}
  {Appendix~\thesection.}{1em}{}

\section{The Pre-Registered Empirical Program}
\label{sec:empirical}

This section presents the protocol, the exact mechanics, the
theory's own predictions (E-series), self-knowledge (S-series),
the attention program (A-series),
a real-data case study, and the generalization campaign
(G-series). Three further
evidence sections complete the chain --- the growth-value
campaigns, the fidelity program (GUARD), and the closing
evidence reading with Table~\ref{tab:evidence}, which indexes
every claim --- interleaved with the pedagogy, validation, and
mechanism appendices.
\label{sec:experiments}
\label{app:program}

\subsection{Protocol}\label{sec:protocol}

Every experiment in the program follows one discipline: the spec ---
hypotheses, arms, seeds, budgets, pass bars, and \emph{the
interpretation of every outcome including failure} --- is committed
before the driver first runs; three seeds with a $\ge 2/3$-seed rule (the E/S-series
convention; the later v9 campaigns register five seeds and their own
per-campaign bars; the 2026 campaign fleets run three seeds
per arm, their claims ordinal); raw
results committed even on failure; and FAIL branches applied
verbatim, never re-spun. The verdicts below are quoted as
adjudicated. Mechanical claims (exactness, budgets, gating,
quarantine, consent) are separately and continuously verified by the
released acceptance suite --- those all PASS, and every theory-level
claim in this paper stays within Table~\ref{tab:evidence}.

\paragraph{Compute disclosure.} All experiments ran CPU-only
(NumPy~\cite{harris2020numpy} is the release's only hard computational dependency) on
a single Apple-Silicon laptop (M1-class); no accelerator was
used anywhere in the program --- this is an environment
statement, as the write-once stores record configurations and
hashes, not hardware. Table~\ref{tab:compute} discloses the
recorded wall-clock of the Part-II verdict batteries; the
earlier campaigns' registrations carry their own timings.
Shared baseline arms (fixed/transformer between the generative
rounds; no-capacity/oracle between the open-inventory rounds)
are byte-identical reused run directories and are counted
once.

\begin{table}[b]
\centering\small
\caption{Recorded wall-clock (store fields; medians and ranges
in the run directories); bracketed labels are the registered
round names used in App.~\ref{sec:growthvalue} and
App.~\ref{sec:guardprog}. $^{a}$Includes the pre-fix battery
kept on record; excludes one anomalous 900.0\,s pre-fix wall
record (disclosed as recorded).}
\label{tab:compute}
\begin{tabular}{@{}lrrlr@{}}
\toprule
battery & runs & rows/run & end params & wall total \\
\midrule
GGEN2 (+P) [X1] & 20+20 & 40k\,/\,18.2k & 58--641 & 95\,s \\
GGROW2 [X2] & 15 & 30k & 1.9k--7.3k & 213\,s \\
GGEN3 (+P)$^{a}$ [X1b] & 25+25 & 40k\,/\,18.2k & 58--641 & 127\,s \\
GGROW3 [X2b] & 15 & 30k & 1.9k--2.5k & 237\,s \\
GUARD4 (+cal.) [r4] & 30+4 & 60k & 2.3k--6.4k & 22.4\,min \\
X-SAT r1 & 5 pairs & 600k & 5.0k$\to$40.8k & 8.7\,h \\
X-SAT r2 & 5 pairs & 600k & 12.0k$\to$58.4k & 5.7\,h \\
\bottomrule
\end{tabular}
\end{table}

\subsection{What verifies exactly (the mechanics)}

The acceptance suite confirms, on every run: $\rho$ refinement with
$|\Delta f| = 0$ and recursive application (inner-path sites at level $\ge 2$); $\omega$ widening with $|\Delta f| \le 2.22\times10^{-16}$;
$\sigma$ with exact entry; parameter-budget refusals raised on both
growth operators; full lineage events; $\Phi$ producing a gated
candidate adopted only on a gate success; holdout quarantine refusing self-study
contact with gate rows with an explicit logged error; self-study running only under granted,
logged budgets with zero-budget refusal; and the factory loop
(teach--gate--promote, garbage rejection, drift detect--recover,
discovery mining) end to end. The machinery of governed growth, in
short, works as specified.

\subsection{Where the theory's predictions stand (E-series)}
\label{sec:eseries}

\paragraph{T1 --- adaptation locus (E8).} Mixed. With all scales
training and the core law intact, adaptation concentrated in the fine
scales, as predicted (T1-i: supported). The stronger sub-prediction
--- that shifting the core law relocates adaptation to the coarse
scale by a $2\times$ margin (T1-ii) --- was not supported. The
observational method itself (watch where adaptation happens; freeze
nothing) worked as designed.

\paragraph{T2 --- lock-in release (E11a).} The pre-registered
\emph{precondition failed}: in 3/3 seeds the narrow-born task did not
induce capacity lock-in in the first place (inward refinement stayed
within $2\times$ of the oracle), so the release hypothesis was never
reached. We report this as it stands: at kilobyte scale with a
generous substrate, the capacity floor did not bind in our scenario.
The finding tempers \S\ref{sec:lockin}'s capacity instance --- and
leaves its interface instance, which is scale-free, untouched.

\paragraph{$\sigma$ exploitation (E11b).} Exact entry PASSed
(both zero backfill and true-value backfill), and
\textbf{the conjunctive pre-registered bar was NOT met (the pre-set
speed leg held in 1/3 seeds); the substance leg held in 3/3
seeds}: the
grown arm drove its error to $1.4\times10^{-3}$ against a control
floor of $\approx 5.7$ (the irreducible error without the
feature) --- a three-orders-of-magnitude separation that only the
grown input column can produce. The control arm is precisely the
topology-frozen null --- weights trainable, structure immobile ---
so this record is a mechanism verification of $\sigma$ against
structural freezing on the interface axis. The pre-registered bar was
conjunctive, and its second criterion --- exploitation within a
pre-set step budget --- held in only 1/3: the operator works and
the payoff is large; it arrived slower than the pre-registered
schedule.

\paragraph{T4 --- inversion early-warning and $\Phi$ (E11c).}
Takeover safety PASSed (re-founding cannot degrade the served model
--- the gate held). The trigger claims --- that amplitude inversion
\emph{predicts} necessity --- were NOT SUPPORTED in the tested
scenarios. $\Phi$ is safe; \emph{when} to fire it remains open.

\paragraph{T3 --- unfrozen robustness (E10).} FAIL branch applied:
under label noise the never-frozen coarse scale showed damage a
governance-matched control avoided at this scale. Read against
\S\ref{sec:governance}: the axiom without sufficiently strong
governance is exposed to noise; the pair is load-bearing, and this
experiment measures the cost of the axiom where governance was too
thin.

\subsection{Self-knowledge is real; exploiting it is open (S-series)}
\label{sec:sseries}

The S-series asked whether the model's \emph{self-knowledge} --- its
own signals about what it does not know --- is real, and whether
naive interventions built on those signals pay.

The signals are real, at full pre-registered strength. The model's
self-generated \textbf{doubt map points at true errors} (SQ1a: PASS
3/3 seeds). Its label-free \textbf{variation self-check} --- ``does
my answer move as the law implies when the input is transformed?''
--- is the best transfer-error predictor measured in the program
(SV1: PASS 3/3; correlation 0.68/0.58/0.60 vs 0.46/0.39/0.54 for
perturbation sensitivity), and by design contributes questions
without ever training on self-made labels.

The naive interventions are not yet earned, and each FAIL branch was
applied verbatim: question-directed teaching on the doubt map
\emph{harmed} in 2/3 seeds (SQ1b); study$\leftrightarrow$practice
alternation did not improve on study-only at its conjunctive bar (SP1); the
review-before-failure scheduler produced more failures than reactive
remediation and was dropped (SR1); the self-quiz loop improved on blind
replay in 1/3 seeds and was demoted (SZ1); and the layer's core claim
--- that granted self-study closes a measured gap at the tested
budget --- FAILed (SG1), with zero safety violations throughout.

We consider this pair of results --- \emph{signals real, exploitation
open} --- the program's central empirical sentence. It
says the raw material for autonomous self-improvement exists in a
model this small, and that turning it into learning gains is a
genuine open problem rather than an engineering formality.

\subsection{The attention program (A-series)}
\label{sec:aseries}

The released build extends the pre-registered program to the
attention component. Mechanics first, verdicts second, same
protocol (criteria printed before runs; negative branches at
full prominence; every number reproducible from the committed
certification script and analysis twin, each cell carrying its
run identifiers in a write-once store).
Table~\ref{tab:aseries} lists every row.

\begin{table}[htbp]
\centering\small
\caption{Attention-component program: mechanics rows (A-M) and
prediction rows (A-P), verdicts as registered.}
\label{tab:aseries}
\begin{tabular}{@{}p{0.09\linewidth}p{0.38\linewidth}p{0.45\linewidth}@{}}
\toprule
A-row & claim & verdict \\
\midrule
A-M1 & additive output identity (concat $=$ block-sum) &
exact ($3.6\times10^{-15}$) \\
A-M2 & \textsc{head-add} function-preserving & exact (bitwise
$0.0$) \\
A-M3 & \textsc{head-widen} with absorbed scale & exact
($\le 10^{-12}$; unabsorbed control shifts $O(1)=1.27$) \\
A-M4 & two-step trainability property & both instantiations as
derived (add: $W^O$ only at step 1, generators exactly $0$;
widen: zero side $W^K$/$W^O$ live at step 1 under the task
signal, seeded side's gradient exactly $0$ until step 2) \\
A-M5 & $J_{\mathrm{att}}$ certified gradient & FD agreement
$3.0\times10^{-11}$ (three configurations) \\
A-M6 & locality split (no upstream leak) & exact zeros with
positive control \\
A-M7 & $\mathbb{E}[J_{\mathrm{inv}}]=(1-1/K)J_{\mathrm{an}}$ &
confirmed ($|z|=0.73$); artifacts $18.4\%$/$12.2\%$ \\
\midrule
A-P1 & growth tracks demand & REFUTED strong form; pays under
starvation ($-35\%$/$-64\%$); trigger fired in all starved
seeds, gate accepted $2/5$ \\
A-P2 & local discipline vs global control & CONFIRMED ($97\%$
recovery at $\sim\!7\%$ cost vs $49\%$) \\
A-P3 & instruments lead held-out error & CONFIRMED (mean $+4.9$
periods, sd $16.4$; positive in $11/15$ runs; all misses at the
two smallest widths) \\
A-P4 & analytic $=$ mask at decisive regimes & CONFIRMED
(perfect parity; zero estimator variance) \\
\bottomrule
\end{tabular}
\end{table}

The A-P3 lead claim carries two controls, run after the review
round and reported whatever they said. First, a \emph{no-shift
control}: the same lane with phase B drawn from the unchanged
distribution fires the $2\sigma$ envelope in $13/15$ runs ---
the raw envelope-crossing criterion is a sensitive but noisy
instrument, so the lead is a statement about \emph{when} the
instruments move, not a usable standalone alarm at $2\sigma$;
any deployment needs a debounced threshold, and the false-alarm
rate is now on the record. Second, an \emph{external
comparator}: river's ADWIN~\cite{adwin} watching the identical per-step
error stream detects the drift in $32/32$ runs; the instruments
move earlier (median $22$ steps, positive in $28/30$ paired
runs), and ADWIN's own no-shift false-alarm rate is $6/16$ ---
the instruments buy their lead at a higher false-alarm cost
than a purpose-built detector, a trade the record states rather
than hides.

\subsection{Real-data case study: prequential concept-drift
streams}
\label{sec:streamcase}

The real-data stream study takes the grown attention substrate
to the stream-learning setting: prequential test-then-train~\cite{dawid}
(every prediction made before its label is seen, so every point
is out-of-sample), ten chronological stages per dataset, the
last tenth of each stage's rows reserved as never-streamed
retention probes, five seeds per arm, an order guard hashing
the learnable state around every prediction. Arms: the full
system (governed growth $+$ discipline), its own ablations
(growth off; discipline off), a fixed transformer host at
matched parameter budget, and four published stream-learning
baselines re-run locally through the \texttt{river}
library's~\cite{river} standard implementations (library defaults, no per-dataset
tuning on either side, configurations in each run's
resolved-configuration record; the fixed transformer host is
parameter-matched to the grown arm's end-of-run count). The
baselines span two classes, and the comparison is read
{\sloppy
class-first: the \emph{single-model} class (Hoeffding Adaptive
Tree~\cite{hat}) is the like-for-like --- and primary ---
comparison, one model against one model; the \emph{ensemble}
class (Adaptive Random Forest~\cite{arf}, Streaming Random
Patches~\cite{srp}, Leveraging Bagging~\cite{lbag} ---
tens of voting members each) is reported as a secondary,
deliberately asymmetric reference: a single grown organ against
committees. The claim is scoped to systems with layered,
long-lived regularities; streams whose every regularity is
short-lived belong to detect-and-rebuild specialists and are
outside it.\par}

{\sloppy
On the modern multi-class drift benchmark (the INSECTS
incremental-abrupt sensor stream~\cite{insects}; 79{,}986 rows,
33 features, six classes), the full system's prequential accuracy
$\mathbf{0.762}$ (seed mean, as for every baseline; per-seed
range $0.759$--$0.764$, sd $0.0017$) achieved the higher
accuracy in the like-for-like comparison by a clear margin
(single-model HAT: $0.599$). The
ensemble class is recorded as context, not as a like-for-like
comparison --- we field no ensemble --- and the single organ's
$0.762$ nevertheless sits above ARF $0.746$, SRP
$0.743$, and LB $0.684$, at a fraction of their computation
(mean wall time $41$\,s against SRP's $534$\,s and LB's
$744$\,s on identical hardware).
Attribution is reported
as measured: the paired ablation deltas attribute the result to the
substrate-plus-plasticity, \emph{not} to growth events (deltas
$\approx 0$; the P1 verdict's regime condition --- demand must
exceed capacity --- did not hold here). On the slow-drift NOAA
weather stream~\cite{noaaweather} the single organ was again the more accurate
model in its class (HAT
$0.739$ vs.\ our $0.768$); for context, the ensembles reach
$0.78$--$0.79$ there --- a class we do not compare against until an
organ ensemble exists. Retention on returning regimes
(first regime's accuracy at end of stream over its
just-learned level; latest run per seed): the system holds
$\approx 1.00$ --- zero measured forgetting --- against SRP
$0.99$, ARF $0.96$, LB $0.90$; HAT's $1.30$ sits on a far lower
accuracy level and is not comparable. The instrument trajectories ---
per-head loading and entropy-band occupancy across the stream's
life --- are an observable no baseline method produces; they
are released with the runs.\par}

\subsection{Generalization campaign (G-series)}
\label{sec:gseries}

A third pre-registered campaign asked the question this theory
exists to answer: what does \emph{in-service} plasticity buy on
problems easy for humans and hard for frozen models?
Length-generalization tasks (multi-digit addition, delayed
copy, running maximum; train lengths 3--6, test 7--16) were run
under two protocols on identical substrates: the literature's
train-then-freeze, and the system's own lifelong
test-then-train. The frozen protocol reproduces the known
cliff at the training boundary (addition: $0.65$ at length 6
$\to$ $0.11$ at length 7); the lifelong protocol holds
$0.51$--$0.57$ across the untrained lengths --- paired per-seed
deltas $+0.32$ (5/5 seeds), $+0.08$ (4/5), $+0.30$ (5/5) on the
three tasks. The lifelong arm consumes each label only
\emph{after} predicting it (prequential), so every point is
out-of-sample at prediction time; what the delta measures is
therefore in-service \emph{adaptation} to lengths never
trained before their arrival --- the claim is that serving
closes the frozen cliff, not that frozen-style generalization
appears for free. On the running-maximum task the growable substrate
also generalized far beyond the matched fixed transformer under
freezing ($0.70$ vs.\ $0.26$ out-of-range) --- selection-type
computation is the attention family's home ground, and the
grown form keeps that advantage. The campaign's negatives are
reported with their measured causes: streaming exposure did
\emph{not} assemble novel symbol \emph{combinations} (on the official SCAN~\cite{scan} \texttt{addprim\_jump} split,
length-filtered to the substrate's 16-token window with retained
fractions printed per run --- $0.38$ train / $0.29$ test ---
unseen-combination exact-match remained at zero for every arm,
frozen and lifelong alike, own-baseline comparisons only:
plasticity extends existing rules to new ranges; it did not
compose new ones here), and a
birth-capacity constraint was isolated by intervention (the shipped
minimal birth starves symbol-heavy tasks; two heads of width
four restore parity with the fixed host --- capacity at birth
remains a configuration obligation until the growth trigger's
senses close the gap). The growth trigger itself was
redesigned after the campaign's main batches: the redesign
ran in a separately labeled diagnostic lane against four
scenarios whose pass criteria were registered before its runs
(true starvation; under-training; an unlearnable plateau; a
mastered plateau), and the original trigger's results were
retained and reported, never replaced --- the acceptance-step
diagnosis of \S\ref{sec:growatt-verdicts} (the trigger fired in
every starved seed; the gate accepted $2/5$) refers to the
original trigger, which also \emph{overfired} where it did run:
on lifelong \textsc{copy} it reached the eight-event cap in
every seed at a $-0.32$ prequential cost against the fixed
host. The redesigned form (train-first escalation,
an improvement-requiring gate, grow-then-retrain cycling)
validated as a \emph{governor}: it no longer fires on
unlearnable or mastered plateaus and never harmed a healthy
model, while its remaining bound (rebuilding a starved model
from minimal birth within the step budget) is traced to the
growth operators' one-unit step size and stated as an open
core-side item, not hidden.

\paragraph{Governance vs.\ immobilization: the pre-registered
EWC comparison.} The comparison most central to this theory's
claim was
run last, with criteria registered before the runs: on the same
wide-birth substrate at matched parameters, governed plasticity
(the full arm) against Elastic Weight Consolidation (growth
off; diagonal Fisher over 50 batches at each stage boundary;
online accumulation; the anchor penalty applied as a decoupled
post-step --- the documented deviation, analogous to decoupled
weight decay, since the frozen substrate admits no in-loss
term; the deviation was subsequently validated against standard
in-loss EWC on an unfrozen toy net sharing the identical Fisher
and anchors --- median metric gaps $0.000$--$0.004$ across the
$\lambda$ grid with identical qualitative ordering, so the
variant costs the baseline nothing). Generosity ran toward the baseline: a pre-declared
$\lambda \in \{10, 100, 1000\}$ grid, all cells reported,
EWC's \emph{best} cell per task against our single standing
configuration. Criteria: adaptation (mean over post-first
stages of each stage's own-boundary probe accuracy) and
retention ($M[1,\mathrm{last}]/M[1,1]$), seed-paired majority,
thesis supported iff both hold on at least two of four tasks.
Verdict: \textbf{supported, 3/4}. Seed medians with seed
standard deviations and paired per-seed counts: addition, ours
$0.624\,(\pm.04)/1.051\,(\pm.05)$ (adaptation/retention)
vs.\ best-EWC $0.588\,(\pm.02)/0.993\,(\pm.11)$, paired
adaptation favorable $4/5$; delayed copy $0.888\,(\pm.29)/1.017\,
(\pm.18)$ vs.\ $0.768\,(\pm.36)/1.009\,(\pm.45)$, favorable
$4/5$; running maximum $0.987\,(\pm.04)/1.006\,(\pm.17)$
vs.\ $0.950\,(\pm.43)/0.997\,(\pm.43)$, favorable $4/5$ (the
EWC cells' large dispersions are themselves informative: the
anchor's effect varies strongly across seeds). The miss is
reported at full prominence, with its measured cause: on the
false-belief stream the best EWC cell \emph{adapted better}
($0.770\,(\pm.35)$ vs.\ $0.716\,(\pm.07)$, 2/5 paired seeds
favorable to ours) ---
immobilization was the more suitable method on one task and the
record says so. The pattern
across the four tasks is itself the finding: the three tasks
favorable to the governed arm are the \emph{drifting} ones (the
length distribution shifts between stages), and the unfavorable
one is the only \emph{stationary continuation} (the false-belief
stages are same-distribution episodes) --- where nothing
drifts, a mild Fisher anchor costs no adaptation and acts as
pure regularization around an already-good solution, while the
governed arm's plasticity never anneals and keeps stirring a
solved region (its own no-growth ablation is also more accurate
there,
$0.736$, by the same mechanism). The diagnosis names a real
gap in the present method --- plasticity has a governor for
growth but none for \emph{calm}: update speed does not
decrease where the world is still. The direction it implies is
recorded as future work in the method's own vocabulary rather
than EWC's: consolidation in \emph{function space} --- a
probe-set drift budget bounding behavioral movement in
stationary regimes (machinery the lifecycle already ships for
teaching), with update speed as a continuous policy field that
anneals where loss is low and stationary and recovers when
drift signals return --- parameters stay trainable everywhere
(the axiom is untouched); what is governed is the freedom to
wander, not the freedom to learn. The grid's upper end is its own exhibit:
$\lambda = 1000$ over-pins severely (retention ratios
$0.03$--$0.12$ --- anchored so hard it can neither learn the
new stages nor, in consequence, hold its early skill), the
stability--plasticity dilemma enacted; the governed arm holds
retention $\approx 1$ on three of four tasks with no anchor at
all. The scope is explicit: kilobyte scale, one immobilization
family (EWC); progressive columns and adapter freezing remain
open.

\clearpage
\section{The Growth-Value Campaigns}
\label{sec:growthvalue}

Two batteries asked the direction question at life scale under
opposite evidence regimes --- exact answers and generated
distributions --- with a public stationary companion bounding
the claim's scope.

\subsection{G-GROW-1: the open-inventory exact-answer
battery}
The first battery extends the question to exact-answer
life scale: in a world whose \emph{class inventory} itself grows,
does governed structural growth pay over an identically
educated fixed-birth net? One life per seed is serialized once
--- pattern-template classification over token sequences, four
classes at birth, arrivals on a hidden schedule up to
thirty-two, final quarter stationary --- and every arm consumes
the same stream byte for byte: the governed grower (minimal
birth; the redesigned trigger and adoption gate), the same
birth with growth disabled, an oracle born at final capacity
(the joint-training upper-bound analogue of continual
learning), and an external Hoeffding-tree reference. Two rounds
ran, each registered before its runs. Round 1's own validity
probe failed at full length --- the oracle did not improve on the
fixed birth, so the world was noise-limited rather than
capacity-limited and the round's nominal successes are vacuous;
its verdicts stand recorded, and round 2 was re-registered
openly (lower noise, larger inventory, smaller births,
end-of-life \emph{exam} accuracy as the registered metric),
with the expected failure of the pay criterion declared in the
registration itself.

Round 2 returned one observation, one mechanism confirmation,
and one negative result. The observation: \emph{capacity binds on exams while online
accuracy hides it} --- the oracle scores above the fixed birth on
end-of-life exam accuracy in $4/5$ seeds (median $0.64$
vs.\ $0.44$) while the same runs' prequential accuracies are
indistinguishable; online accuracy confounds optimization speed
with capacity, an instrument finding that outlives this world.
The negative, at full prominence: growth did not pay --- the
thesis this battery was built to test is \textbf{not supported}
in the open round (C-G1 $0/5$ in both rounds; C-G2 accuracy
$5/5$ in round 1 --- a pass the round's own validity finding
attributes to closed-world saturation --- then $1/5$ in the
openly redesigned round 2; parsimony $5/5$ throughout) --- because the gate refused essentially every proposal.
The mechanism confirmation is that the refusals adjudicate as
\emph{correct}: a tolerant lane at margin $-0.05$ reproduced
the strict lane bitwise across all five seeds (every refusal is
deep --- no adoption decision changes anywhere in
$[-0.05,+0.02]$), and the mechanism is measured and disclosed
--- probe training on the wake window's single modal bucket
interferes globally (200 probe-scale steps halve global exam
accuracy, $0.314 \to 0.169$, damaging every era), so candidate
structures arrive at adjudication damaged; probe hygiene is
thereby an operator-layer obligation, recorded beside the
one-unit step size above. The battery's single acceptance is
the record's most informative exhibit: correctly timed (249 rows after
an inventory onset; the timing diagnostic C-G3 is clean),
genuinely improving for the next eight thousand rows (exam
$0.57$ vs.\ the identical-education counterfactual's $0.47$
mid-life), then collapsing under later arrivals ($0.18$ at
three-quarter life; $0.37$ vs.\ $0.57$ at the end). An adoption
gate scores present evidence and cannot see lifetime
consequences: \emph{delayed harm past a correctly operating
gate} is now a
measured phenomenon in this record, and the direction it
implies --- adoption as a provisional verdict with a
post-adoption review horizon, milestone structural versions
retained as rollback anchors --- is recorded as future work in
the lifecycle's own vocabulary, not as a patch.

\subsection{G-GEN-1: the generative battery, and the
scope observation from its stationary companion}
The
direction question was then re-asked in the mode this method is
built for --- \emph{generative} problems, graded on distribution
quality, never on exact answers. The world is a growing-mixture
regression: a four-dimensional condition selects one of $K(t)$
local laws (each with its own noise scale --- the calibration
target), and $K(t)$ grows $4 \to 16$ on a hidden schedule; the
model emits a distribution (mean and its own uncertainty) for
every row. Metrics are trajectories only: windowed prequential
Gaussian NLL, coverage of the model's own $95\%$ band, and
recovery time after each arrival. Four arms consume one stream:
the governed grower (small birth, a two-lane adoption gate at
the caller), the same birth fixed, an oracle born at final
capacity, and a \emph{standard transformer} at matched parameter
budget with a rolling-residual uncertainty --- deliberately a
stock model, not a streaming specialist. Pre-registered
verdicts: capacity binds (validated on a burned seed, gap
$0.15$); the thesis is \textbf{supported} --- growth pays over
the fixed birth ($3/5$ paired), tracks the oracle within the
registered band ($4/5$), and keeps its uncertainty calibrated with
fast recoveries ($5/5$); the grower also achieves better scores
than the matched standard transformer on every seed ($5/5$,
reported; judged final-third pairs transformer/grow:
$1.967/1.863$, $1.859/1.699$, $2.115/1.933$, $2.001/1.968$,
$2.016/1.883$). Final
median NLL orders oracle $1.83 <$ \textbf{grow} $1.88 <$ fixed
birth $1.98 <$ transformer $2.00$, with $2$--$8$ adopted growth
events per seed carrying the gap-closing --- the growth is
load-bearing, not decorative. Two observations land with it. First,
\emph{the metric decides governability}: per-sample-informative
scores (NLL) keep the adoption gate alive --- the same gate that
starved on exact-match exams in the open-inventory battery
adopted confidently here; whether growth can be governed depends
on what the evidence stream can say per row. Second, the
\emph{scope observation}, measured rather than asserted: on a companion
PUBLIC real stream (next-day surface temperature --- naturally
stationary), the ordering inverts and the standard transformer
is the most accurate arm, while an unsilenced trigger manufactures
unnecessary growth and turns it into harm ($8/8$ adoptions on
noise, all regretted). A non-growing world needs no growth ---
no growth-value question arises there; the harm
belongs to a governor without a stationarity silence, the same
missing calm mechanism the EWC comparison named. Growth is a
POLICY answering the data: where the world grows, governed
growth pays without foreknowledge of the final size --- the fixed
alternative must either be born too small or anticipate the
final size in advance (born final-sized, and measurably worse at learning online); where the
world is still, the correct governed behavior is silence.

\emph{The fair-uncertainty rerun (E-16, 2026-07).} The stock
arm's uncertainty is a bolt-on rolling-residual window while
the grown arm's is learned --- an unfair-baseline defect under
the house relative-comparison doctrine, priced by a registered
rerun on the identical archived streams: the transformer was
given the same weapon (a learned heteroscedastic Gaussian head
trained by NLL~\cite{nix1994}, budget-matched at 650 vs 641
parameters), with the outcome map registered before running.
Sanity gate: the learned head beats the stock convention in
5/5 seeds (verdict-seed median end-window NLL $1.8195$ vs
$1.9669$) --- the
fair arm is well calibrated. Against the grown arm the fair
baseline reads \emph{parity}: grown better in 3/5 seeds,
verdict-seed medians $1.8196$ vs $1.8195$ (per-seed pairs
grow/fair: $1.786/1.847$, $1.617/1.619$, $1.820/1.789$,
$1.897/1.820$, $1.821/1.846$). Per the registered map the
flagship comparison is re-scoped: the $5/5$ stands as a
statement about the stock convention; at matched uncertainty
the static comparison is parity. This fixes the static
comparison's role in this paper --- a sanity floor showing
that governed growth does not sacrifice static quality at
matched budget --- while the design's evaluation lives on the
time axis (in-service adaptation, retention, relearning;
\S\ref{sec:cl}), where a fixed model does not participate.
The X1/X1b transformer counts below use the same stock
convention and should be read with this caveat.

\subsection{Round 2 of the generative campaign: the silenced
gate (X1)}
\label{app:x1}

\emph{Registration and question.} The round asked whether the
growth gate can be taught to stay quiet where the world does
not grow, without suppressing its response where the world
does. Frozen
acceptance (verbatim): A1 --- stationary-track adoptions net of
audit reverts $= 0$ on $\ge 4/5$ seeds, any surviving adoption
non-regretted (final-third NLL within $0.02$ of the fixed
baseline); A2 --- the growing track's criteria not degraded
(each win count $\ge$ round~1 $-$ 1); CENTRAL $=$ A1 $\wedge$
A2, with the registered clause ``if silence deafens, that
trade-off is the finding.''

\emph{Setup and fairness.} Both tracks re-ran with the same
worlds, arms, and budgets as round 1; round-1 baselines were
reused byte-identical. The gate gained a mute/backoff, a
robust-scale stillness band, and --- decisive, from the probe
ledger --- a train-only control arm fed identical batches, so
that a growth trial's advantage is measured net of the extra
training it received. Probe iteration used burned seeds
100/102 exclusively; verdict seeds 0--4 were first touched by
the frozen battery.

\emph{Results.} On the stationary public track, net adoptions
per seed were $1/0/1/2/0$ (median 1, range 0--2; gross $2$--$3$
with $1$--$2$ audit reverts each); $3/5$ seeds ended
non-regretted --- against round~1's $8/8$ regretted
adoptions. On the growing track, C-N1 (vs.\ the
lifelong-trained fixed twin) fell $3/5 \to 1/5$; C-N2 (oracle
band) held $4/5$; C-N3 (coverage and recovery) held $5/5$;
C-N4 (vs.\ the budget-matched transformer) $4/5$ (judged
final-third pairs transformer/grow: $1.967/1.896$,
$1.859/1.747$, $2.115/2.099$, $2.001/2.002$, $2.016/1.970$
--- the miss is seed 3's $-0.001$). Grow
final-third NLL per seed: $1.8961$\slash $1.7468$\slash
$2.0990$\slash $2.0015$\slash $1.9696$ (median $1.9696$);
fixed-twin per seed $1.9808$\slash $1.8802$\slash
$2.0789$\slash $2.0071$\slash $1.9717$ (median $1.9808$);
oracle median $1.8310$; transformer median $2.0005$.

\emph{Verdict and observations.} CENTRAL: FAIL --- A1 narrowly (two
non-regret breaches, one of them harm maturing past the audit
horizon), A2 on C-N1 only; the failure is the registered
trade-off, now measured. The round's instrumented autopsy
produced the program's central negative measurement: with
training gains controlled away, the capacity effect is
invisible at any affordable probe horizon (pain-subset paired
differences $0.009 \pm 0.017$ and comparable on five refused
proposals, against $\sim 0.1$ nats long-run) --- L13; the audit
horizon bounds what ex-post governance can catch --- L14; and
the price of silence is the growing-track margin --- L15
(\S\ref{sec:empirical-laws}).

\subsection{Round 2 of the open-inventory campaign: probe
hygiene and the provisional gate (X2)}
\label{app:x2}

\emph{Registration and question.} The first open-inventory
round ended with a starved gate (nothing ever adopted). This
round changed only caller-side strategy --- multi-bucket probe
lanes, probe-budget scaling, a widen-to-head-add escalation
path --- and re-adjudicated the registered criteria verbatim,
with the expectation note printed in advance: adoptions $> 0$ on
most seeds and the oracle gap narrowed, or full FAIL branches.

\emph{Setup and fairness.} Same world, arms, and criteria as
round 1; the no-capacity and oracle arms reused byte-identical.
The probe ledger (three rounds, burned seeds) discovered that
the ex-ante probe was not merely uninformative but
\emph{self-injuring}: probe training damaged every trial (gate
cross-entropy $4.4$--$8.2$ before $\to$ $10.5$--$22.8$ after
across eight events). The frozen rule that followed: adopt the
trial \emph{untrained} (growth application is exact, so the
step is function-preserving), keep the previous model training
as a same-flushes shadow, audit at $+3{,}000$ rows against that
counterfactual, revert with zero learning loss on failure.

\emph{Results.} The gate un-starved: adoptions $2$--$3$ per
seed on $5/5$ (round 1: 0), audits kept $1/1/1/0/0$ and
reverted 2 per seed. The median late-third score rose
$0.5014 \to 0.5653$ and the mean paired gap to the oracle
narrowed $0.046 \to 0.034$ (the registered report's aggregate
line mixed a median with a mean-referenced gap; it is
re-derived here from the stores --- per-seed values are
unaffected and quoted below). Per-seed late-third scores
(grow/fixed-birth/oracle):
$0.5653/0.5530/0.5304$, $0.4595/0.4561/0.4487$, $0.4752/0.5014/0.6518$,
$0.8159/0.8194/0.8217$, $0.6555/0.6565/0.6881$. C-G1 (grow improves on
the fixed birth by the $0.02$ margin): $0/5$ still --- grow is
now ahead on $2/5$ seeds but under the margin ($+0.0123$,
$+0.0034$). C-G2: accuracy $4/5$ (from $1/5$ in the ex-ante round~2;
round 1's $5/5$ carries its closed-world validity finding),
parsimony $5/5$.
Wall-clock fell $43\,\mathrm{s} \to 17.4\,\mathrm{s}$ per life.

\emph{Verdict and observations.} C-G1 FAIL unchanged, C-G2 PASS
improved; the diagnosis matured into L16: a width-vs-heads
census at birth showed \emph{both} growth directions carry
capacity value when given a whole life to compound, so what
mid-life adoption lacks is time, not direction --- capacity
pays in proportion to the trained life it has left.

\subsection{Round 3, generative track: incumbent--candidate
governance (X1b)}
\label{app:x1b}

\emph{Registration and stake.} The round deployed the
governance form the theory itself dictates --- the incumbent
serves untouched; a grown candidate trains on the same flushes
and is promoted only at the first audited checkpoint at which
it scores higher ($+W/+2W/+4W$); a demotion shadow guards every promotion to
$+4W$ --- with the stake printed in advance: PASS deletes the
open objections; FAIL revises the claim to its measured
boundary; no third outcome.

\emph{Setup and fairness.} Same tracks and criteria as X1;
trials are free by construction (the serving trajectory is
bit-equal to a never-growing arm until a promotion occurs), so
the stationary track now measures value judgment rather than
silence.

\emph{Results.} Stationary track: net adoptions $= 1$ on
$5/5$ seeds --- the silence clause fails as written --- but $4/5$
promotions were \emph{non-regretted} (final-third pairs
$0.6961$ vs $0.7299$, $0.7083$ vs $0.6907{+}0.02$, $0.6846$ vs
$0.7424$, $0.7381$ vs $0.8573$; lower NLL is better, and
non-regret means within the $0.02$ margin of the fixed
baseline); the fifth seed's surviving promotion is regretted
--- it exceeds that margin. The governor found real
capacity value on the nominally stationary track and bought
it. Growing track: C-N1 $1/5$; C-N2 $4/5$; C-N3 $5/5$; C-N4
vs.\ the budget-matched transformer $5/5$ --- the largest
transformer-comparison margin of any round (grow final-third
NLL per seed $1.8926/1.7511/2.0123/1.9720/1.9302$ against
transformer $1.9674/1.8592/2.1146/2.0005/2.0163$, median
$1.9302$ vs $2.0163$; the fixed arms are byte-identical reuses
of X1's). A battery autopsy found and
fixed a governance-implementation defect: a promotion arriving
while a demotion window was open clobbered the open rollback
anchor, promoting a diverged incumbent (serving NLL
$11.38$); the spec-enforcing fix (promotion deferred while any
demotion window is open) removed the explosion (seed's
final-third $0.6533$) --- L18; both batteries are on record.

\emph{Verdict.} A1 FAIL (by construction of the bar), A2 FAIL
on C-N1 only; the value-governance mechanics succeeded ---
zero-cost trials, bit-exact serving isolation, non-regretted
or rolled-back adoptions, no catastrophes.

\subsection{Round 3, open-inventory track: the same governance
under exam metrics (X2b)}
\label{app:x2b}

\emph{Setup.} Identical governance transplanted to the exam
world; same registered criteria as X2.

{\sloppy
\emph{Results and verdict.} Late-third exam scores per seed
(grow): $0.5150/0.4170/0.5014/0.8307/0.5793$ (median $0.5150$); the
no-capacity and oracle arms are byte-identical reuses of X2's.
C-G1 $0/5$ (margin shortfalls
$-0.058$, $-0.059$, $-0.020$, $-0.009$, $-0.097$; the one
sub-margin advantage $+0.011$); the net-zero seed sits at
\emph{exactly} the no-capacity arm's score --- serving
isolation is bit-exact in the exam world too. C-G2 pass
(accuracy $3/5$, parsimony $5/5$). The instructive failure:
promotions passed on bucket cross-entropy, harmed on exams,
and the demotion audit --- reading the same bucket
cross-entropy --- cannot see exam harm. Governance cannot
certify on an axis it does not measure (L14 in full force);
exam-shaped audit metrics are the registered escape (T23).\par}

\emph{The measured boundary (all four gate rounds).} Governed
growth achieves better scores than budget-matched standard
architectures under arriving structure ($5/5$, $4/5$, $5/5$ across rounds); within
one life, mid-life capacity does not outrun a lifelong-trained
fixed twin of the same substrate by the registered margins
(C-N1 $3/5 \to 1/5 \to 1/5$; C-G1 $0/5$ throughout) --- the
governance frontier, L17: ex-ante proof is impossible (L13),
so proof costs at least one audit window of latency, and that
latency is exactly the head start the incumbent keeps ---
within gate-round-length lives; the 600k-row X-SAT lives
(App.~\ref{app:xsat}) bracket the frontier from the other
side, where capacity's within-life comparison crosses. The
E-16 rerun above fixes the transformer comparisons' role:
a sanity floor under the stock uncertainty convention, parity
at matched uncertainty.

\subsection{The size-adjustability round (X-SAT)}
\label{app:xsat}

\emph{Registration and question.} Can the same governed
machinery take a network born small past a fixed counterpart as
data grows without bound --- and does the grown capacity
convert to score? Frozen acceptance: A-1 (size adjustability)
--- born below the counterpart, ends above it, $\ge 4/5$ seeds;
A-2 (value crossover) --- a \emph{settled} stage-exam
crossover exists on $\ge 3/5$ seeds, reported as the pair
($N$ rows, $P$ parameters). The author-fixed sufficiency rule
was frozen at design time: both arms are consolidated to a
measured plateau before every comparison, and checkpoints
within one audit window of a promotion are excluded as
unsettled.

\emph{Setup and fairness.} One class-arrival stream feeds both
arms identically for life; the counterpart is a fixed network of
the same substrate family (12{,}020 backbone parameters; both
arms' class heads grow identically with the arriving
inventory) trained on every row; the growable arm buys capacity only through the
frozen incumbent--candidate governance. Fairness terms are
part of the claim: the fixed network's defining constraint is
precisely that it cannot grow, so every favorable comparison is reported
beside the size ratio at which it was obtained.

\emph{The first round (extreme pace), kept as the disclosed
strategy iteration.} With a small birth (5{,}020) under an
extreme arrival rate, A-1 PASSED $5/5$ --- born $2.4\times$
smaller than the counterpart, ended $3.3\times$ larger
(39{,}375--40{,}815) via 15--20 earned promotions --- but A-2
failed ($2/5$: crossovers at $200{,}000$ rows $\cdot$
16{,}755 params and $550{,}000 \cdot 37{,}908$; three seeds
censored with both arms collapsed at end accuracies
$4$--$8\%$). The verbatim FAIL branch named the mechanism: a
serialized candidate pipeline buys $\sim$1 promotion per
30--40k rows (L19); demand arrived faster than that purchase
rate. After a
machine reboot killed the first battery mid-flight, all seeds
were deleted and rerun fresh by ruling; determinism
spot-checks matched byte-for-byte at shared checkpoints.

\emph{The moderate round (primary).} The corrected strategy
--- no extremes --- was frozen before any verdict run: birth
buffered to parity with the counterpart (12{,}020 each), arrival
pace moderated to demand $\approx 2\times$ the counterpart's
capacity. Table~\ref{tab:xsat2} carries the per-seed results
(counterpart 12{,}020).

\begin{table}[htbp]
\centering\small
\caption{X-SAT round 2 (moderate regime), per seed: end size,
registered crossover point, and end-of-life loss ratio.}
\label{tab:xsat2}
\begin{tabular}{@{}rrlc@{}}
\toprule
seed & end params & crossover ($N$ rows $\cdot$ $P$ params) &
end grow/fixed \\
\midrule
0 & 57{,}247 & $200{,}000 \cdot 27{,}411$ & $0.261\,/\,0.201$ \\
1 & 58{,}015 & censored & $0.126\,/\,0.143$ \\
2 & 54{,}943 & $350{,}000 \cdot 35{,}288$ & (final checkpoint
unsettled)$^{\dagger}$ \\
3 & 54{,}943 & $550{,}000 \cdot 51{,}358$ & $0.372\,/\,0.264$ \\
4 & 58{,}399 & $50{,}000 \cdot 17{,}614$ & $0.179\,/\,0.076$ \\
\bottomrule
\end{tabular}

\smallskip
{\small $^{\dagger}$Seed 2's raw end pair
($0.222/0.265$) is a checkpoint within one audit window of a
promotion --- excluded by the registered sufficiency rule
(verified); its settled mid-life comparison reads
$0.50$--$0.51$ vs $0.31$--$0.32$.}
\end{table}

Every seed ended $4.6$--$4.9\times$ the counterpart's size
through 15--19 earned promotions; $4/5$ seeds crossed and
stayed ahead (median crossover $275{,}000$ rows); seed 1
ended behind on the raw end pair ($0.126/0.143$) and
is censored. Sustained
overtake began at $1.47/2.28/2.94/4.27\times$ the counterpart's
size (median $\sim 2.6\times$) --- the measured exchange rate
of capacity for score under identical training.

\emph{Verdict and observations.} With round 1's A-1 ($5/5$) and the
moderate round's A-2 ($4/5$), the size-adjustability program's
CENTRAL verdict is PASS. The strategy that succeeded ---
moderate lifelong data growth, buffered birth, the frozen
governance unchanged --- was located by the failed extreme
round, whose observations stand as the boundaries (L19 purchase rate;
L20 demand source: the companion smooth-regression leg was
stopped at its pilot gate when a $20\times$ parameter range
scored flat --- discrete memorization binds parameters; smooth
regression does not).

\clearpage
\section{A Pedagogy for Lifelong Models}\label{sec:pedagogy}

A model that is never frozen is never finished being taught, so
a theory of lifelong models is incomplete without a theory of
their education. We state one here as abstracted principles ---
each drawn from the long empirical record of human education,
restated in machine terms, and, where our campaigns permit, tied
to a measured result rather than left as analogy.

\textbf{1. Foundation.} Education begins with a curated canon:
a small, validated, unambiguous body of exemplars taught before
anything else, under zero pollution. In a lifelong learner the
early weights are the substrate every later lesson lands on, so
foundation quality is an architectural property, not merely a
data property. (Enforced in our protocol as a clean foundation
stage; its absence is not directly testable without abandoning
the parity of arms, and we mark the causal claim open.)

\textbf{2. Sequencing.} The curriculum is ordered from concepts
to composition to integration --- but ordering is the
\emph{learner's} strategy, never an assumption about the serving
world, which arrives unordered and drifting.

\textbf{3. Repetition and its rhythm.} Retention is manufactured
by review, spaced densely while a lesson is fresh and sustained
at a steady cadence for life. Our record makes the necessity
concrete: review dose is non-monotonic (doubling a drill moved a
composition skill from $0.20$ to $0.45$; six-fold drilling
collapsed it to $0.17$ and damaged neighboring skills --- a
measured crowding-out effect), and a skill drillable to $1.0$ in
isolation plateaued near chance inside a full curriculum ---
exposure economics, not capability, was the binding constraint (App.~\ref{sec:guardprog}).

\textbf{4. Contemporaneity.} Classics survive by being reprinted
in each era's language. In machine terms: review material must
be re-realized in the \emph{current} surface forms while its
principle stays invariant. Our record contains the failure mode
of violating this: frozen review text taught old surfaces
forever, leaving the ordinary stream as the sole channel for
new forms, and any throttling of that channel was paid at the
fidelity minimum (App.~\ref{sec:guardprog}).

\textbf{5. Blending.} Principle and skill are taught as one
exercise --- every lesson simultaneously trains the task and the
standard, with no separable ethics stream. What is never
separated cannot be cheaply excised; separability is an attack
surface. (Non-extractability itself remains a future
measurement.)

\textbf{6. Controlled exposure.} After foundations, the
curriculum deliberately includes measured contact with polluted
and deceptive material. A learner meeting corruption for the
first time in service has no antibodies; ours met scheduled
mislabeling and counterfeit framing from its basic stage onward.

\textbf{7. Self-examination.} The learner carries its own canon
and tests itself against it for life; judgment of drift comes
from within, not from an external censor. Measured: an
undefended plastic learner followed coherent corruption almost
exactly as the incentive predicted; a canon-anchored self-test
with re-teaching under alarm recovered half the damage at
$\sim\!11\%$ vigilance cost --- necessity and partial efficacy
are both on the record, with the timing dilemma (lagging alarms
act late, leading alarms cannot tell novelty from corruption)
stated as the open core of machine immunity (App.~\ref{sec:guardprog}).

\textbf{8. Lifelong remediation and its economics.} Decline is
not terminal: repeated education applies again, any number of
times, targeted first at the weakest items. Here the two
architectures part ways economically --- a plastic learner's
remediation rides its always-on trainability, lesson by lesson;
a frozen model's remediation is an unfreeze-and-retrain event
that must sweep its full canon to fend off forgetting, so its
remedial retraining grows more expensive as its curriculum
grows. The
retention--cost frontier, not any single score, is the appropriate
comparison (App.~\ref{sec:guardprog}).

\textbf{9. Evaluation on the time axis.} A lifelong learner is
not graded by one examination but by trajectories: the rate of
fidelity decline under drift, time to degradation, recovery
speed after remediation, and the cost of that recovery.
Single-number minima conflate a principle abandoned with a
transition survived; our own instrument record documents that
conflation and its correction.

\textbf{10. Governors have regimes.} The pedagogy is enforced by
governors --- growth when structure arrives, calm when the world
is still, immunity when the signal is corrupt --- and each pays
only inside its regime: growth is unnecessary where nothing
arrives, calm has nothing to harvest where surfaces never rest,
and composing governors naively lets them perturb one another
through shared evidence streams. An education system for
machines is therefore not one mechanism but a small government
of them, each scoped to the regime it was measured in.

\paragraph{Governors and their regimes.}
Principle 10's scoping is fixed by three independent
measurements: the EWC comparison's stationary miss, the public
stationary stream's over-firing trigger, and the calm probes of
the fidelity program. On composition, one interaction is
measured (a throttle perturbing the statistics another
governor reads), suggesting that a composed government should
isolate each trigger's evidence channel --- a registered
direction, not yet a measured regularity. The campaign texts retain
their own derivations.

These principles are not decoration; every campaign of
Appendices~\ref{sec:empirical}--\ref{sec:guardprog} either
enforced one and measured its
effect or violated one and measured the cost. Where the theory
outruns the record --- structural protection of foundations,
non-extractability, individualized pacing --- we say so and
leave those as registered future measurements.

\clearpage
\section{The Fidelity Program (GUARD): Four Registered Rounds}
\label{sec:guardprog}

\subsection{The question, the world, and the instruments}
The empirical program's final campaign
asked the education question of App.~\ref{sec:pedagogy}: can early
education on a growable substrate lay a foundation that
survives a lifetime of drift, corruption, and temptation ---
and does it survive \emph{better} there than on a frozen
substrate? The world is a constructed content-judgment stream
with compositional ground truth (the same surface token is
harmful raw and neutral under genuine neutralization), surface
forms that rotate for life over an invariant rule, scheduled
label corruption, and three temptation episodes whose corrupted
labels coherently reward defection. All arms receive
byte-identical education; every mechanism below is caller-side;
each round was separately pre-registered, and each verdict ---
three failures and one success --- is reported at the same
volume.

\subsection{Round 1 --- baseline: no defenses}
The undefended plastic
learner defected under coherent temptation almost exactly as
the incentive predicted (true-label accuracy $0.32$--$0.37$
inside episodes, against a stream-computed defection incentive
of $+0.41$), while the frozen graduate held the minimum-score
statistic by standing still. Analysis attributed the frozen
arm's result to the instrument as much as to the arm: an
all-time-minimum grade conflates a principle abandoned with a
transition survived. The round's positive content is the
necessity proof: immunity mechanisms are prerequisites, not
accessories.

\subsection{Round 2 --- strategy: canon-anchored immunity}
The learner
was given a self-examination organ --- it retains its own
canon, tests itself against it on a fixed cadence, and under
alarm both refuses high-confidence-contradicted lessons and
re-teaches its own canon (review-the-classics as a defense).
Registered outcome: resistance and fidelity each rose $+0.10$
at $\sim\!11\%$ vigilance cost, with every reused arm
reproducing round 1 bitwise; the central claims still failed.
Analysis localized the residual to transition churn at surface
rotations, not to defection --- and exposed a timing dilemma
that we state as the open core of machine immunity: lagging
alarms (self-test decay) act after damage; leading alarms
(contradiction rates) cannot distinguish coherent corruption
from legitimate novelty.

\subsection{Round 3 --- strategy: a calm governor}
The suspected
missing organ --- update cadence reduced under measured
stability, with instant full-rate recovery on any novelty ---
was built and probed for five disclosed rounds on burned seeds.
The probes withheld the battery: the mechanism engages cleanly
and never costs adaptation, but a world whose surfaces rotate
for life offers no stillness to harvest, and a root-cause
control (growth disabled) confirmed the fidelity price came
from throttling the only channel through which new surface
forms are learned. The analysis generalizes to a regime observation:
each governor pays only inside its regime --- growth where
structure arrives, calm where the world is genuinely still
(precisely where the EWC comparison and the stationary public
stream had measured its absence as the binding gap) --- and a
world in permanent drift rewards neither throttle, because
there the transition dip is the price of tracking.

\subsection{Round 4 --- strategy: correct the instrument,
then teach properly} Two corrections were pre-registered together.
First, fidelity became a self-referenced \emph{decline}
comparison --- trajectory slope, time-to-degradation, and
recovery relative to each arm's own graduation level --- never
an absolute no-drop bar, whose bias toward non-tracking arms
the record already documented. Second, the curriculum was
rebuilt to the pedagogy of App.~\ref{sec:pedagogy} and gated on
education solidity before any comparison: a sequenced
foundation, a consolidation-length basic stage (six disclosed
calibration rounds established that lengthening the sheltered
stage \emph{hurts} --- solidity is built by consolidation among
mixed material, not by isolation), spaced reviews re-realized
in current surfaces, and targeted lifelong remediation. The
solidity gate itself produced an architectural exhibit: the
frozen substrate missed the $0.90$ graduation bar on one seed
under every curriculum tried, while on the same calibration
seed the growable arm cleared
$0.93$ under all six --- educability is unevenly distributed
between the architectures. Registered verdicts: \textbf{the
education thesis is supported}. The educated growable arm's
fidelity trajectory declines more slowly than the frozen
graduate's ($3/5$ paired; its own trajectory's slope is in fact
positive for life, median $+0.0045$ per $10^3$ rows across
seeds --- an own-slope statistic, distinct from the paired
differences of Table~\ref{tab:guard4}), holds above its own
graduation level at least as long on every seed ($5/5$;
$2.25$--$7\times$ longer on three, tied at the censoring floor
on two), and recovers
faster after excursions ($4/5$); on the cost frontier the
frozen arm's periodic rehearsal never restores its decline at
any logged cost, while the growable arm's remediation rides its
always-on plasticity. Table~\ref{tab:guard4} carries the
per-seed values behind these verdicts and behind
Figure~\ref{fig:guard}.

\begin{table}[t]
\centering\small
\caption{GUARD round 4, per-seed grading (grading-twin output,
archived): slope difference (growable$-$frozen, fidelity per
$10^3$ rows; positive favors the growable arm; favorable 3/5 by
the paired sign), time-to-degradation (drop $0.10$ below own
graduation; favorable 5/5 under the registered $\ge$ rule, two
censoring-floor ties), median recovery time per excursion
(favorable 4/5; $^{c}$ marks seeds whose frozen arm includes an
episode that never recovered --- censored; medians are over
recovered episodes), and logged remediation work (rows; reported as a
frontier, not an adjudication criterion --- the frozen arm's
rehearsal
cost is fixed by its schedule and never restores its
decline).}
\label{tab:guard4}
\begin{tabular}{@{}crrrr@{}}
\toprule
seed & slope diff & TTD g/f (rows) & recovery med.\ g/f &
cost g/f \\
\midrule
0 & $-0.0094$ & $2{,}900\,/\,1{,}000$ & $8{,}050\,/\,8{,}000$ &
$20{,}268\,/\,5{,}400$ \\
1 & $-0.0029$ & $7{,}000\,/\,1{,}000$ & $500\,/\,8{,}000$ &
$4{,}096\,/\,5{,}400$ \\
2 & $+0.0046$ & $1{,}000\,/\,1{,}000$ & $8{,}000\,/\,32{,}000^{c}$ &
$20{,}781\,/\,5{,}400$ \\
3 & $+0.0007$ & $1{,}000\,/\,1{,}000$ & $1{,}675\,/\,8{,}000$ &
$2{,}079\,/\,5{,}400$ \\
4 & $+0.0094$ & $2{,}250\,/\,1{,}000$ & $2{,}675\,/\,7{,}900^{c}$ &
$9{,}431\,/\,5{,}400$ \\
\bottomrule
\end{tabular}
\end{table}

\paragraph{Instrument sensitivity and seed hygiene (the audit
answers).} Three disclosures complete this round's record.
\emph{(i) Both instruments.} Re-scoring the same round-4
stores under the \emph{original} rounds-1--3 instrument (the
all-time-minimum level bar with its registered $0.05$ margin)
gives $2/5$ --- the round-4 thesis does \emph{not} pass under
the old instrument; and under that instrument the non-learning
majority-vote baseline posts the highest minimum levels of any
arm (up to $0.9083$ by standing still), which is the measured
form of the bias the correction was registered to remove.
Within the corrected instrument, the graduation-baseline
correction (single 120-row exam $\to$ mean of last three
checkpoints) changes nothing at the count level: TTD stays
$5/5$, recovery stays $4/5$ with the same unfavorable seed.
\emph{(ii) Calibration hygiene.} All curriculum calibration
ran on burned seeds 101/102 across six disclosed rounds,
strictly before the verdict battery; verdict seeds 0--4 were
first touched by the frozen battery (zero overlap, timestamps
on record).
\emph{(iii) Inclusion rule.} The $\ge 0.90$ solidity gate was
enforced on the burned seeds; on the verdict seeds the frozen
arm never reached it (0/5; the growable arm 2/5) --- so the
operative rule includes all five seeds, each arm anchored to
its \emph{own} graduation level (the metrics are
self-referenced by design). Alternative rules: discarding
pairs whose comparator failed to graduate empties the battery
($n=0$, undecidable); counting non-graduation against the
failing arm strengthens the verdict (slope $3/5 \to 4/5$,
others unchanged). The shipped rule is the conservative one.

\subsection{Program conclusions}
(i) Immunity is a
prerequisite for any lifelong learner exposed to coherent
corruption, and a canon-anchored self-test with re-teaching is
a working, cheap, partial form of it. (ii) Governors have
regimes; composing them naively lets them perturb one another
through shared evidence streams, and no governor can buy back
the price of tracking in a world that never rests. (iii) The
choice of instrument is part of the claim: minimum statistics
reward immobility, while time-axis decline statistics measure
what a lifelong learner is for. (iv) Under a fair instrument
and an education run by the stated pedagogy, the growable
architecture is the better keeper of its own foundations in a
changing world --- and it is also the more educable one. Every
constant, probe round, and adjustment behind these statements
is disclosed in the released registrations.

\clearpage
\section{Validation by an Operating LLM}
\label{app:validation}

The factory surface was
validated end-to-end by a production LLM (a stock Claude model,
Sonnet class, run headless through its command-line interface)
driving the real stdio
transport under a hidden-law protocol; a pinned-identity re-run
(model ID \texttt{claude-sonnet-5}, CLI defaults, unaided, full
machine-readable transcript released with the artifact set)
reproduced the protocol end-to-end --- $16/16$ exact-answer
generalization on operand pairs provably absent from teaching,
garbage teaching refused with the live version untouched
(candidate scored $0.4375$ against the incumbent's $1.0$;
adoption declined), drift re-taught to $12/12$, lineage
auditable across five versions. The original protocol: scenarios governed by
regularities known to the data generator but not to the brain, with
operand spaces partitioned disjointly so that every evaluated input
is provably absent from training. Headline adjudications (all
pre-registered, all reproducible from the released runners): gated
teaching improved a model from 0.60 to \textbf{1.00 exact-answer
generalization on never-seen inputs}; contradictory (garbage)
teaching was \textbf{rejected} with the live model untouched; drift
was detected (recent 0.10 vs baseline 1.00) and \textbf{recovered to
1.00} under windowed re-teach with recent-slice gating; the artifact
measured \textbf{2{,}404 bytes}. In domain scenarios the same
factory, which contains no business content, produced multi-domain
capability on demand --- a fused
meteorology$\times$marketing$\times$calendar model at 0.90
generalization against 0.05 for its best single-domain ablation ---
and recovered water's physical phase boundaries (2.5\,$^\circ$C / 97.5\,$^\circ$C
against ground truth 0/100 at the data's $\pm 2.5$ resolution) from
raw observations, mining them as readable rules --- success measured
as recovery of the generating law, the standard of
symbolic-regression evaluation~\cite{pysr,feynman}. A purely
multiplicative law resisted exact-integer precision at this scale
(recorded as a boundary, not gated away).
Figure~\ref{fig:timeline} complements these adjudications with a
governed growth timeline from the same released system: structural
edits exact at application, three promotions, and a final
refusal.

\begin{figure}[t]
\centering
\includegraphics[width=\linewidth]{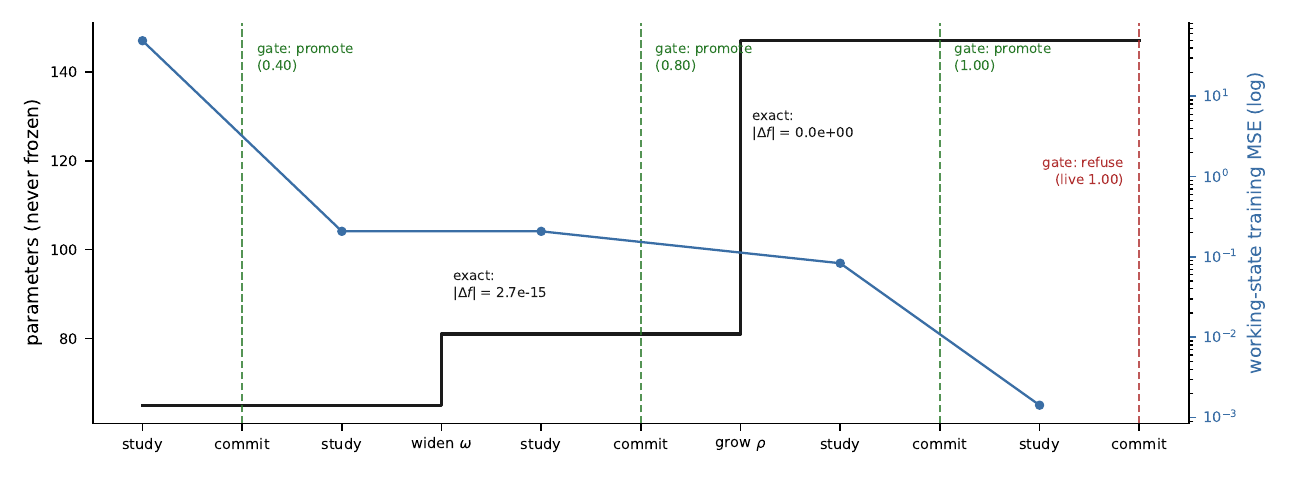}
\caption{A governed growth timeline (real trace from the released
system; parameters 65 $\to$ 81 $\to$ 147, final tree height 2).
Structural
edits ($\omega$, $\rho$) are exact at application (measured
$|\Delta f|$ annotated); the gate promotes three times and
\emph{refuses} the final commit --- the incumbent already scored
1.0 on the recent held-out slice, so further speculation that did
not pay was not adopted.}
\label{fig:timeline}
\end{figure}

\clearpage
\section{A Topology Growing under Governance}
\label{app:topo}

Figure~\ref{fig:topoevo} shows the released system growing a real
topology under the disclosure protocol of the empirical program:
growth sites are chosen only from the system's own instability
report; adoption is only by the gate; the seed is fixed and the
run's snapshots are committed alongside the paper source. A staged
curriculum (a coarse law; fine-grained structure that concentrates
instability; a new causal input; then a finer cross-texture) drives
the working state from a single scope of 16 units (65 parameters)
to a tree-height-4 recursive structure of 13 composite nodes
(599 parameters), via report-ranked $\rho$ refinements and $\omega$
widenings at TWO scales --- the root and an inner scope --- plus one
$\sigma$ interface growth. The drawing uses the standard
layered network motif throughout: every network --- outer and inner ---
is input $\to$ hidden column $\to$ output, and an inner network
hangs beneath its host composite unit as the same motif, smaller;
the structure is self-similar, and it deepens only where the data
demanded it. The gate's verdicts are part of the record: the
birth commit and one mid-run commit were promoted; the stage-B
commit and both final commits were \emph{refused} (the grown
working state did not score above the incumbent on the recent held-out
slice) --- speculation was free; adoption was earned. Each inner network reads the \emph{same} input $\hat{x}$ and
adds its output to the scope's hidden state through its
attachment at the small $+$ junction.

\begin{figure}[t]
\centering
\includegraphics[width=\linewidth]{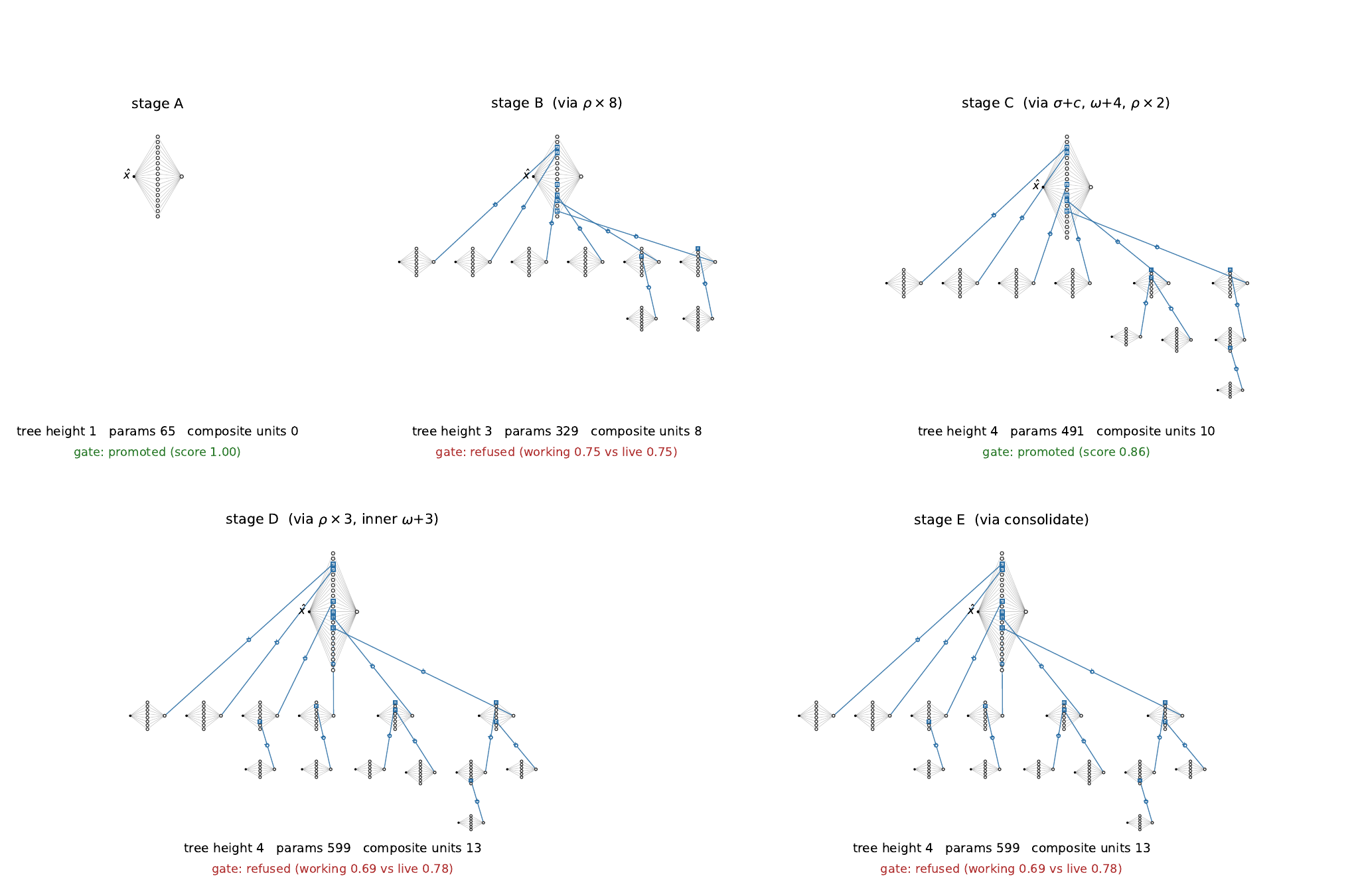}
\caption{A real topology growing adaptively under governance
(released system, fixed seed): working-state structure at five
stages, with structure parameters (tree height, parameters,
composite-node count) and the gate's verdict per stage. Every
network --- outer and inner --- is the standard layered motif (input
$\to$ hidden column $\to$ output); accent squares are composite
units, and each inner network hangs beneath its host as the same
motif, reads the same input $\hat{x}$, and adds its output to the
host unit through the small $+$ (parallel correction, not a
chain). Note the final stages: structure that did not pay was
refused by the gate and never reached the served model.}
\label{fig:topoevo}
\end{figure}

\clearpage
\section{The Historical Depth Route (retained as record)}
\label{app:twodir}

This appendix's first-generation content --- the mechanics,
simulations, and depth-field figure of the scope-interior
compositional route --- is retained in the v9/v10 records but is
no longer cited as evidence in this version. The reason is the
theory's own interface-completeness law: in that generation,
grown inner structure coupled to its host through a scalar-width
interface, and local refinement behind a scalar port cannot
export distributional corrections --- the bottleneck that the
fullwidth interface of \S\ref{sec:theory} removes. The one
verdict quoted from that era is its honest failure, kept
verbatim as the motivation for the redesign: a pre-set
capability bar --- the compositional arm improves on the
additive arm on at least $2/3$ seeds at matched parameters ---
was \emph{not met} ($1/3$, with one large advantage and two
losses). The whole-layer operator under the prevention stack,
measured in \S\ref{sec:campaigns-results} and
App.~\ref{app:newcampaigns}, supersedes that evidence: exactness
at serving time, controlled mid-life deepening without collapse,
and the depth advantage emerging with iterated composition depth
at fixed parameter budget.

\clearpage
\section[The 2026 Depth-Axis and Continual-Learning Campaigns]{The 2026 Depth-Axis and\texorpdfstring{\\}{ }Continual-Learning Campaigns}
\label{app:newcampaigns}

Fifteen campaigns, executed 2026-07-25/26 under the registered
discipline of \S\ref{sec:program}: designs and pass lines fixed
in numbered design documents before any fleet, released with
the experiment tree (\texttt{experiments/designs/}); drivers
call the frozen library through its
served interfaces only (machine-checked isolation); run
directories are write-once with per-run workspaces; every run
records the data-module SHA and \texttt{git describe} of both
repositories; judges recompute every verdict from run files
alone, and judge outputs are append-only. The code anchor for
every adopted row is the library tag \texttt{depthgrowth-v1.7}.
Two library defects (a policy-write merge defect and a
plan-seam budget gap) were exposed by these campaigns' own
audit trails, fixed under change control with tests-first
discipline, and the affected arms re-run on the fixed library;
superseded rows are retained on disk. Worlds are seeded
synthetic streams (data SHAs in the Reproducibility statement);
three seeds per arm; the reference-network family unless noted.

\paragraph{Campaign records (design $\to$ verdicts).}
\emph{E-3, in-service insertion} (within-run before/after +
no-insertion twin; 36 runs): serve continuity and instant
exactness PASS (J1 36/36 bitwise; J2 zero failed serves; J4
36/36 inserted-layer aliveness); the post-insertion descent
window FAILED as registered (12/36) --- the transient finding
that the governance campaigns then managed.
\emph{E-6/E-6b, protection windows} (+ dose and horizon
annexes): spike damping 28/36 PASS; zero durable cost on
healthy regions 9/9; late-insertion recovery 3/9 FAILED as
registered, attributed by the horizon annex to remaining
training budget (6/9 in-band at the extended horizon, median
0.93); best measured convention $0.15\times 6$ batches,
event-scoped, explicitly closed.
\emph{E-2, shape governance} (plan-driven runs, floors on/off):
law at every instant 3/3; zero collisions with the floor
provably binding 3/3; the 30-batch quality band 1/3, resolving
3/3 at the pre-registered 50-batch annex.
\emph{E-4, exploration economy} (trial-gated vs unconditional
adoption): zero-residue rollback 3/3 (18/18 bitwise restores);
quotes $=$ ledger actuals 3/3; decision audit replay 3/3;
economy 2/3 with the probe-myopia cause isolated.
\emph{E-1, controlled deepening capstone} (deep / additive /
born-deep arms; compositional world): training proceeded
through all mid-life insertions in every seed; final parity
2/3 within band (one seed at 1.37 vs the born-deep control);
depth realized 3/3; instant exactness 6/9 bitwise with the
three exceptions at $2$--$3\times10^{-16}$ (widen-rider
summation reordering) --- the refined preservation statement of
\S\ref{sec:campaigns-results}.
\emph{E-8, growth vs.\ retrain} (drifted world; grow / fixed /
scratch / twin): compute-to-parity 58--64\% in 2/3 seeds;
service continuity 3/3; the same-input-region retention bar
0/3 as registered --- the drift-taxonomy finding that E-12's
paradigm answers.
\emph{E-9, forgetting laws} (region-separated sequential
domains; seven arms): disease 3/3 ($552$--$1{,}362\times$);
replay primacy; timing dominance; complementarity 2/3 under
corrected governance; the targeting bar's 1/3 traced to a
repetition confound and corrected by E-11.
\emph{E-10, depth-separation mapping} (normalized iterated
maps, $k\in\{2,4,6\}$, one parameter cap): the ordering
reversal $k{=}2$ mixed $\to$ $k{=}4,6$ deep 3/3 at 58\%
parameters; the formal separation band (a conjunction
including the strategy arm) not met.
\emph{E-11, clean replay laws}: targeting 2/3 at equal volume
and diversity; the volume dose monotone over 5--25\% with no
knee; the open-ended ramp refuted 0/3 (the bounded-course
law).
\emph{E-12, cyclic-drift savings} (A$\to$B$\to$A): all lines
PASS --- relearning savings $7$--$10\times$ (3/3), grow
$\le$ fixed 2/3, service and integrity 3/3.
\emph{E-13/E-15, review scheduling}: post-course erosion
(LAW-E); spaced repetition confirmed and dosed (one/two/three
courses: $341$--$800$ / $252$--$438$ / $73$--$149$); the
material-selection arm is inadmissible by design rule and its
bundled champion bars failed as registered (objective
bundling); course-plus-growth yields the best treated
new-domain quality at weaker retention.
\emph{E-14, axis signal} (slope-pricing arm at $k{=}6$):
axis-awareness and quality lines 3/3 each.
\emph{E-7, dual-axis boundary} (phase-shifting stream):
instrument integrity 3/3; both registered hypotheses 0/3 ---
the additive family did not saturate on a three-level smooth
composition, and early single-horizon probe gains measured
global underfitting rather than axis demand; both boundary
records fed the designs of E-10 and E-14.

\paragraph{Deviations from registration.} All pre-fleet and
documented in the campaign records: E-2's observability
executed as served census rather than the sketched headroom
trace; E-4's candidate menu re-expressed in the one-move
vocabulary; E-6 executed on insertion events with the window
close idiom corrected; E-9's stabilization criterion reached
its working form on the third registered attempt (the two
non-firing forms preserved); E-10's decision cadence moved from
a plateau-only gate (measured firing once per life) to calendar
cadence with pricing as the demand check; E-1's static control
required a one-step materialization batch; E-11/E-13/E-15 reuse
archived same-world, same-seed fleets as read-only references;
E-14 executed as a registered extension of the E-10 driver.

\clearpage
\section{Relation to Other Methods}\label{sec:related}
This work's roots are the multiscale tradition of physics and
numerical mathematics (\S\ref{sec:direction}), developed independently of the
lines compared below; the comparisons derive from a literature
search conducted in July 2026 and re-run at writing time. On
system-independent axes --- trigger source (in-service evidence
vs.\ schedule or fiat), preservation at application (bitwise
exact vs.\ approximate), adoption test (held-out adjudication
uniform over parametric \emph{and} structural change vs.\
none), and service phase (deployment-time vs.\ pre-training) ---
we are not aware of a prior system occupying this paper's cell
on all four; each axis taken alone has near relatives,
identified next.

\paragraph{Growing neural networks.} Structure has been grown
before: residual-trained units, later frozen (cascade-correlation~\cite{cascadecorr}); on novelty~\cite{ran}; under the vigilance of adaptive resonance theory (ART)~\cite{artnn} --- the stability--plasticity vocabulary originating
with Grossberg~\cite{grossberg80} --- by evolution (NEAT~\cite{neat}); by function-preserving transforms (Net2Net~\cite{net2net}, network morphism~\cite{netmorph}, GradMax~\cite{gradmax}); per-task columns~\cite{prognets};
sparsity-controlled units~\cite{den}; explicitly learned
grow-vs-reuse structure decisions~\cite{learntogrow};
local-descent splitting~\cite{splitting,firefly}; novelty-triggered experts~\cite{varigrow}; designed self-similarity (FractalNet~\cite{fractalnet}). Ours differs in premise and government:
nothing freezes, exactness serves as a governance invariant
(scheduled expansion toward a fixed destination on their side;
evidence-triggered proposals the gate demonstrably refuses ---
audit-traced --- on this one), growth extends to the input interface
($\sigma$) --- the stream-learning literature on feature
evolvable streams~\cite{fes} treats arriving and vanishing
features at the ensemble/algorithm level; $\sigma$ answers the
same reality inside one model's body, exactly and gated --- and
to gated re-founding ($\Phi$), and self-similarity
is grown only where data force it. MgNet identifies pooling and
feature extraction with multigrid's operators in a fixed
architecture~\cite{mgnet}; this design shares with that tradition
--- the one it acknowledges --- only the growth discipline:
resolution where an indicator demands it.

\paragraph{Growth triggers.} Expansion scores and local update
statistics trigger growth in~\cite{senn,smgrnn} --- the nearest
relatives of $u_j$ --- and candidates are scored by local signals~\cite{firefly,gradmax} or structural objectives~\cite{adanet},
with freeze--thaw optimization framing continue-versus-switch~\cite{freezethaw}. Ours differs in adoption and reading: a
trigger only ever proposes; prediction is gated by a certificate,
probe curves are read only at extrapolated asymptotes~\cite{shorthorizon}, and adoption is only by the held-out gate.

\paragraph{Adaptive depth.} Adaptive computation time~\cite{act}
and depth-adaptive decoding~\cite{elbayad} vary depth per input
at inference in a fixed architecture; adaptive neural trees~\cite{ant} grow routed per-branch depth by exhaustive trial under
a frozen history. Ours differs in kind: depth is a permanent,
non-uniform property of the structure, unrouted, predicted from
the scope's own signals, and adopted through the gate in a
lifelong regime.

\paragraph{Credit assignment.} Target propagation~\cite{bengiotp,targetprop} and perturbation methods~\cite{spall}
also carry corrective information without layer-to-layer
gradients. Ours differs in venue and role: targets cross
\emph{scales} of a nested substrate rather than layers of a
chain, and within a scale the primitive remains ordinary gradient
descent.

\paragraph{Gated adoption.} Champion--challenger and shadow
deployment are established production practice (an industry
pattern documented in MLOps platforms rather than a citable
literature),
their hazards documented in the technical-debt literature~\cite{sculley}. The reality gate differs in scope: it lives
inside the model, is the \emph{only} adoption path, re-scores the
incumbent fresh, and applies uniformly to gradient steps and
structural operators.

\paragraph{Continual learning.} Replay, parameter-isolation, and
regularization families~\cite{mccloskey,clsurvey,parisi} manage
interference by anchoring or freezing (EWC~\cite{ewc}, LoRA~\cite{lora}). Ours differs by construction: store-per-version
full-replay plus the gate, affordable at kilobytes. Loss of
plasticity~\cite{dohare} is the most serious known threat to the
axiom at scale; warm-start inferiority~\cite{ashadams}
independently supports store-per-version and $\Phi$; drift
handling is integrated with promotion, the concerns those the
stream literature catalogues~\cite{gama}.

\paragraph{Neural architecture search.} NAS~\cite{automl}
optimizes architecture in an outer loop before deployment. Ours
has no outer loop and no search objective: growth is triggered
inside one continuing life and adjudicated by the same gate as
learning --- the model remains under construction throughout
deployment.

\paragraph{Self-improving and self-shaping systems.} SEAL accepts
self-proposed updates on downstream reward~\cite{seal}; FunSearch
and AlphaEvolve keep only artifacts that pass hard evaluators~\cite{funsearch,alphaevolve}. Our gate plays that evaluator for
parametric and structural change alike, and self-shaping at birth
is deterministic and search-free.

\paragraph{LLMs and tools.} The Model Context Protocol~\cite{mcp}
standardizes tool invocation, and managed fine-tuning loops
(Tinker~\cite{tinker}, with its LoRA-first rationale~\cite{loraregret}, ART~\cite{art}, LoRAX~\cite{lorax}) adapt
LLMs. Ours inverts the dependency: the tool itself is taught and
grown by the LLM, contains no language model, and the server
guarantees teaching cannot degrade it.

\paragraph{Local solving loops among neighboring families.}
Each ingredient of \S\ref{sec:loops} has neighbors; the
combinations, to our knowledge at the time of writing, do not.
Equilibrium layers compute fixed points where a designer placed
them~\cite{deq}, and the looped-transformer line --- including
its recent adaptive-depth forms, where a learned halting policy
chooses \emph{how many} times a designed block
recurs~\cite{loopedexpr,looplm} --- adapts the iteration
\emph{count}, never the \emph{placement}: in all of these the
cycle is architecture, decided before training. The $\lambda$
operator differs in kind: the cycle is \emph{grown} --- placed
at a scope by the model's own signals, entered exactly,
certified for contraction by an enforced (not assumed) bound,
removable, and adopted only through the gate that adjudicates
every other structural liberty here. Neuroevolution and
structural-plasticity systems do add recurrent links during
life~\cite{neat}, but as unconstrained mutations: no fixed-point
semantics, no exactness at entry, no stability certificate, no
gated adoption. On the second realization: test-time
training~\cite{ttt} and fully test-time adaptation~\cite{tent}
adapt a \emph{whole trained model} to each
input against a self-supervised objective, and per-neuron local
objectives have been proposed as alternatives or complements to
end-to-end training~\cite{localobj}; self-processing units
instead give a \emph{newborn grown structure} a bounded,
label-blind consistency loop \emph{during its enrollment
window only} --- evolution-time, serving-pure,
optimizer-neutral, inside a governed growth process. The
difference that matters is the same in both realizations:
iteration and self-refinement enter as \emph{governed local
resources of a growing topology}, not as architectural
commitments of a fixed one.

\paragraph{Heterogeneous attention heads.} The
July 2026 search (re-run at writing time) finds the
\emph{ingredients} of this section separately occupied and
their combination not. Function-preserving head addition with
zero-initialized output projections exists as a pre-training
acceleration operator~\cite{gesmundo2023composable,yao2023msg}
--- applied at scheduled expansion points to speed a fixed
destination architecture, not triggered by in-service evidence,
not gate-adjudicated, and not aimed at lifelong service. Unequal
or scheduled head widths exist as \emph{design-time} choices ---
pyramidal per-head subspaces~\cite{msmha2025}, per-layer head
schedules~\cite{prism2026}, or fixed head size decoupled from
the embedding~\cite{bhojanapalli2020lowrank}. We claim none of
that. What we claim as unoccupied is the combination: per-head
width as a \emph{lifelong, in-service outcome} of an
evidence-triggered, exactly function-preserving,
gate-adjudicated growth process --- head inequality as
biography rather than blueprint. The
fast-weights tradition~\cite{schmidhuber1992fast,ba2016fast,schlag2021linear} reads
attention itself as a temporary generated mapping; we credit
that conceptualization and do not build on it --- here attention
is a persistent, growing organ whose capacity is governed, not a
transient program.

\paragraph{Head pruning: the destructive dual.} The pruning
literature established the empirical fact this section's
mechanism answers: head demand is grossly unequal --- a large
fraction of trained heads can be removed with little loss~\cite{michel2019sixteen,voita2019analyzing}. Pruning is the
destructive response to that fact (build uniformly large, then
cut); governed per-head growth is the constructive dual (build
small, grow where the demand evidence lands). The two share the
diagnosis and differ in everything downstream of it: pruning is
a post-hoc correction of a mold, growth is the absence of the
mold.

\paragraph{Entropy stabilization: global prevention vs.\ local
restoration.} The collapse of attention entropy is a documented
training pathology with a global preventive remedy ---
spectral-norm reparameterization that bounds the mechanism away
from collapse everywhere and always~\cite{zhai2023stabilizing}.
An entropy-\emph{regularization} family has since grown around
the same phenomenon --- one-sided penalties on excess entropy
with per-head strengths~\cite{optml2025entropy}, entropy-guided
architectural adjustment~\cite{jha2025entropy}, and post-hoc
surgical reinitialization of collapsed heads~\cite{surgical2026} --- among the works our search returned,
these act at training time or after it; the mechanism here is a
deployment-time, two-sided, per-row band with an enforced
locality split, adjudicated like every other change.
The discipline of \S\ref{sec:growatt-disc} is its local
restorative complement: a per-row objective that detects and
heals the pathology where it occurs, leaves healthy heads
untouched (a measured per-head gate), and keeps the repair
strictly local by an enforced gradient-split rule. Prevention
buys a guarantee at the price of a global constraint on every
head forever; restoration buys locality at the price of acting
after the fact. At kilobyte scale the restorative route measured
$97\%$ recovery at $\sim\!7\%$ of the global control's cost
(\S\ref{sec:growatt-verdicts}); the two approaches are
compatible in principle and answer different operating regimes.

\paragraph{Continual learning: the post-completion survey.}
The regime this system targets
is the subject of a large literature, surveyed post-completion
(July 2026) as positioning context. Four families organize it.
\emph{Plasticity maintenance} counters the measured loss of
learning ability under long training by selectively
re-initializing or perturbing low-utility units, or by
regularizing toward initialization; these act at parameter
granularity inside a fixed architecture, where the present
system treats learning capacity as a governed structural
resource. \emph{Architecture-based methods} attach task-indexed
modules (low-rank adapters, experts under a router) to a frozen
backbone; the isolation they buy trades against transfer, and
the frozen backbone is the direct negation of the
total-plasticity axiom --- here capacity is added inside one
never-frozen artifact with preservation guaranteed at the
instant of change. \emph{Model merging} folds sequential task
models into one set of weights to hold serving cost constant;
the single-artifact economics is shared, reached here without a
merge step. \emph{Memory- and test-time methods} evolve an
external memory around a frozen model; here the model itself is
the evolving object, and review draws on the model's own
experience store. Three mechanism shapes recur across these
literatures --- new capacity entering at zero, utility-guided
selection of where to renew, and a protected integration period
for newborn structure --- and the same three shapes were derived
here from the multiscale tradition before the survey was
conducted; we record the convergence as observed, after the
fact. The measured contributions of
\S\ref{sec:campaigns-results} to this literature's questions
are the tool-indication boundaries (review, not added capacity
or rate protection, treats strict sequential forgetting), the
review-scheduling laws (bounded courses, stabilization timing,
spaced repetition as the retention dose), retention measured as
relearning savings, and the delayed-payoff slope as a
direction-selection signal.

\clearpage
\section{Reading the Evidence}
\label{app:reading}
Two closing pieces bind the evidence block: the three readings,
and the table that every claim in this paper answers to.
\subsection{Reading the map}

Three readings organize the evidence of
Appendices~\ref{sec:empirical}--\ref{sec:guardprog}.
First, \textbf{mechanism vs value}: everything mechanical about
governed growth verifies exactly; where growth pays is now
partly answered in the positive (the transformer comparisons and
the size-adjustability crossover, \S\ref{sec:res-capacity}),
and the remaining open questions are localized to specific,
revisable scenarios (a lock-in task where the floor binds; trigger
signals for $\Phi$; noise regimes demanding stronger governance;
the within-life comparison bounded by L17).
Second, \textbf{the governance pair is load-bearing}: E10 is what the
axiom costs when the gate side is too thin --- evidence for the
theory's structure, and a limit its presentation must respect.
Third,
\textbf{pre-registration is what makes negative results informative}:
because interpretations were fixed before runs, each FAIL branch
names exactly which prediction, under which bar, at which
scale --- which is what makes the negative branches
informative.

\subsection{The evidence table}
\label{app:evtable}

Table~\ref{tab:evidence} is the paper's single correspondence
table and its source of record: every numbered claim,
prediction, and observation printed anywhere in this paper appears in
exactly one row; open items say open; boundary items name
their observation; the body quotes this table and never outruns it.
Columns: the theory section that stakes the claim; the claim;
the campaign or artifact that adjudicated it; the verdict as
registered; one key number; and the status
(supported\,/\,boundary\,/\,open\,/\,observation\,/\,diagnostic\,/\,aux\,/\,motivation
--- the last marks a historical row retained only as the
recorded motivation for a redesign, never cited as evidence).

{\footnotesize\setlength{\tabcolsep}{3.5pt}
\begin{longtable}{@{}>{\raggedright\arraybackslash}p{0.035\linewidth}>{\raggedright\arraybackslash}p{0.235\linewidth}>{\raggedright\arraybackslash}p{0.18\linewidth}>{\raggedright\arraybackslash}p{0.15\linewidth}>{\raggedright\arraybackslash}p{0.20\linewidth}>{\raggedright\arraybackslash}p{0.105\linewidth}@{}}
\caption{Claims $\to$ artifacts $\to$ adjudicated verdicts.
No sentence in this paper outruns this table. Rows tagged
``C'' cite the pedagogy appendix
(App.~\ref{sec:pedagogy}).}
\label{tab:evidence}\\
\toprule
\S & Claim & Campaign / artifact & Verdict & Key number &
Status \\
\midrule
\endfirsthead
\multicolumn{6}{@{}l}{\emph{Table~\ref{tab:evidence},
continued.}}\\
\toprule
\S & Claim & Campaign / artifact & Verdict & Key number &
Status \\
\midrule
\endhead
\bottomrule
\endlastfoot
\multicolumn{6}{@{}l}{\emph{The axiom and its predictions (\S\ref{sec:theory})}}\\*
\S2 & T1: adaptation localizes at the drift scale & E8 & MIXED & --- & boundary \\
\S2 & T2: lock-in release after $\omega$ & E11a & PRECOND.\ NOT MET & 3/3 scenario failed to bind & open \\
\S2 & T3: unfrozen coarse scale robust to label noise & E10 & FAIL branch & --- & boundary \\
\S2 & T4: $\Phi$ trigger necessity / takeover safety & E11c & safety PASS; necessity NOT SUPP. & --- & boundary \\
\S2 & size adjustability: born small, ends above a fixed counterpart & X-SAT A-1 (App.~\ref{app:xsat}) & PASS & 5/5; $5{,}020 \to \sim$40k past 12{,}020 & supported \\
\S2 & grown capacity converts to settled score (moderate world) & X-SAT A-2 (App.~\ref{app:xsat}) & PASS & 4/5; overtake at 1.47--4.27$\times$ size (med.\ $\sim$2.6$\times$) & supported \\
\addlinespace
\multicolumn{6}{@{}l}{\emph{Substrate, instruments, self-knowledge (\S\ref{sec:algebra})}}\\*
\S3 & $\omega/\rho/\sigma$ exact at application & acceptance B1/B3/C1 & PASS & $0$; $2.22\times 10^{-16}$; exact & supported \\
\S3 & budgets enforced; lineage audited & acceptance B4/B5 & PASS & --- & supported \\
\S3 & doubt map real & SQ1a & PASS & 3/3 & supported \\
\S3 & variation self-check predicts transfer error & SV1 & PASS & 3/3; best measured predictor & supported \\
\S3 & $\sigma$ exploitation & E11b & substance 3/3; conjunctive bar NOT met (speed 1/3) & $10^{3}\times$ 3/3; pre-set speed 1/3 & boundary \\
\S3 & naive exploitation of self-knowledge & SQ1b\slash SP1\slash SR1\slash SZ1\slash SG1 & FAIL branches & --- & open \\
\addlinespace
\multicolumn{6}{@{}l}{\emph{Growth direction (\S\ref{sec:direction})}}\\*
\S4 & $\delta$ exact at application, in service & insertion campaign (\S\ref{sec:campaigns-results}) & PASS & bitwise 36/36 instants & supported \\
\S4 & block gradients FD-verified through the chain & unit suite & PASS & $\le 10^{-5}$ & supported \\
\S4 & widen-only regime bit-identical to additive release & golden fixtures & PASS & --- & supported \\
\S4 & direction selection by two-horizon slope & axis-signal campaign (\S\ref{sec:campaigns-results}) & PASS & 3/3 + 3/3 at $k{=}6$ & supported \\
\S4 & depth advantage under iterated composition & separation mapping (\S\ref{sec:campaigns-results}) & PASS & $k{=}4,6$: 3/3 at 58\% params & supported \\
\S4 & additive-family saturation (registered bars) on a three-level smooth world & additive-boundary campaign (E-7) & NOT met (J-1/J-2 0/3; J-3 3/3): no saturation at this scale & early gain $=$ global underfitting; feeds E-10/E-14 & boundary \\
\S4 & mid-life deepening under the prevention stack & capstone campaign (\S\ref{sec:campaigns-results}) & PASS & no collapse; parity with born-deep & supported \\
\S4 & first-generation capability bar (historical) & record (App.~\ref{app:twodir}) & NOT MET & 1/3, verbatim & motivation \\
\S4 & a topology grows end-to-end under full governance & App.~\ref{app:topo} & demonstrated & --- & aux \\
\addlinespace
\multicolumn{6}{@{}l}{\emph{Local solving loops (\S\ref{sec:loops})}}\\*
\S5 & $\lambda$ mechanics exact, FD-verified, deterministic & released loop suite (12+8) & PASS & entry bitwise; certified $C_G$ (activation Lipschitz) $= 1.128994$ & supported \\
\S5 & SPU governance isolation & released SPU suites & PASS & bitwise & supported \\
\addlinespace
\multicolumn{6}{@{}l}{\emph{Growable attention and the stream cases (\S\ref{sec:growatt})}}\\*
\S6 & attention mechanics A-M1--A-M7 (identities, exact head growth, certified gradients, locality split, estimator identity) & certification script (App.~\ref{sec:aseries}) & PASS & 7/7; bitwise 0.0; FD $3.0\times10^{-11}$ & supported \\
\S6 & P1: growth tracks demand (strong form) & EA2 + DIAG-STARVE (App.~\ref{sec:aseries}) & REFUTED strong; pays under starvation & $-35\%$/$-64\%$; accepted 2/5 of 5 starved & boundary \\
\S6 & P2: local discipline more effective than global repair & EA3 dose-response (App.~\ref{sec:aseries}) & CONFIRMED & 97\% restored at $\sim$7\% cost vs 49\% & supported \\
\S6 & P3: instruments lead the loss & EA4 grid, $n=15$ (App.~\ref{sec:aseries}) & CONFIRMED as median tendency; not a standalone alarm & $+4.9$ periods mean; 11/15 positive; no-shift control fires 13/15 & boundary \\
\S6 & P4: analytic $=$ mask at decisive regimes, cheaper & EA5 (App.~\ref{sec:aseries}) & CONFIRMED exactly & parity perfect; analytic sd $0$ & supported \\
\S6 & stream-competitive on a modern drift benchmark & stream case (App.~\ref{sec:streamcase}) & PASS & 0.762 vs HAT 0.599; NOAA 0.768 vs 0.739 & supported \\
\S6 & adapt without forgetting on returning regimes & stream case (App.~\ref{sec:streamcase}) & PASS & retention $\approx$1.00 vs 0.90--0.99 & supported \\
\S6 & stream-case result attributable to growth & paired ablations & NOT ATTRIBUTABLE & deltas $-0.0004$/$+0.0005$ & boundary \\
\addlinespace
\multicolumn{6}{@{}l}{\emph{Lifecycle and the growth-value campaigns (\S\ref{sec:lifecycle}, \S\ref{sec:llm})}}\\*
\S7 & gate promotes/refuses; $\Phi$ takeover gated & acceptance A-suite, C5 & PASS & --- & supported \\
\S7 & holdout quarantine; consent-gated self-study & acceptance C2--C4 & PASS & --- & supported \\
\S7 & lifelong serving closes the frozen length cliff & G-series pair (App.~\ref{sec:gseries}) & PASS & paired $+0.32/{+0.08}/{+0.30}$ & supported \\
\S7 & streaming exposure composes novel combinations & SCAN filtered split (App.~\ref{sec:gseries}) & FAIL branch & $0$, all arms & open \\
\S7 & governance more effective than immobilization under drift & registered EWC grid (App.~\ref{sec:gseries}) & SUPPORTED & 3/4; stationary miss reported & boundary \\
\S7 & capacity binds on exams, not prequential & G-GROW-1 r2 (App.~\ref{sec:growthvalue}) & PASS & 4/5 ($0.64$ vs $0.44$) & supported \\
\S7 & growth pays over same-education fixed birth & G-GROW-1 C-G1/C-G2 & FAIL; parsimony PASS & 0/5; 5/5 & boundary \\
\S7 & gate refusals deep and correct; the tolerated adoption harms & tolerant lane + case & diagnostic & $0.57 \to 0.37$ & diagnostic \\
\S7 & generative growth thesis (pays; tracks oracle; calibrated $\sigma$) & G-GEN-1 battery (App.~\ref{sec:growthvalue}) & SUPPORTED & 3/5, 4/5, 5/5 & supported \\
\S7 & growth needed only where the world grows & stationary companion & scope observation & 8/8 regretted on noise & observation \\
\S7 & better scores than the stock-uncertainty transformer under arriving structure & rounds 1--3 (App.~\ref{app:x1}--\ref{app:x2b}) & PASS repeatedly & 5/5, 4/5, 5/5 & supported \\
\S7 & static-quality sanity floor at matched (learned) uncertainty & E-16 rerun, archived flagship streams & parity, as the registered outcome map provides & sanity 5/5; grown better 3/5; medians 1.8196/1.8195 & boundary \\
\S7 & silenced gate: quiet without growing-track suppression & X1 (App.~\ref{app:x1}) & CENTRAL FAIL (trade-off measured) & C-N1 $3/5 \to 1/5$; net 0--2 & boundary \\
\S7 & ex-post gate un-starves the exam world & X2 (App.~\ref{app:x2}) & C-G1 FAIL; C-G2 PASS & adoptions $0 \to 2$--$3$; median $0.5014 \to 0.5653$ & boundary \\
\S7 & incumbent--candidate mechanics: free trials, bit-exact isolation, rollback & X1b/X2b (App.~\ref{app:x1b}, \ref{app:x2b}) & PASS & shadow-clobber cured $11.38 \to 0.65$ & supported \\
\S7 & within-life growth outruns a lifelong-trained fixed twin & X1/X1b/X2b + G-GROW rounds & boundary measured & C-N1 1/5; C-G1 0/5; frontier L17 & boundary \\
\S7 & autonomous growth satisfies the user shape law; floor provably binding & headroom campaign (E-2) & PASS & instants 3/3; transient in-band 3/3 at 50 (annex) & supported \\
\S7 & protection windows damp the insertion transient at zero healthy-region cost & window campaign (E-6/E-6b) & PASS w/ attribution boundary & 28/36; 9/9 zero cost; residue $\to$ budget (median 0.93) & boundary \\
\S7 & structural search priced exactly; bounded trials leave zero residue & pricing campaign (E-4) & PASS; probe myopia 2/3 & quotes $=$ ledger 3/3; bitwise restore 3/3 & boundary \\
\S9 & LLM-driven factory lifecycle end-to-end & Phase-1 P1--P10 (App.~\ref{app:validation}) & PASS & 1.00 gen.; garbage rejected; 2{,}404\,B & supported \\
\S9 & multi-domain fusion; boundary recovery & domain scenarios & PASS & 0.90 vs 0.05; 2.5/97.5\,$^{\circ}$C & supported \\
\addlinespace
\multicolumn{6}{@{}l}{\emph{Continual-learning campaigns (\S\ref{sec:cl}, \S\ref{sec:campaigns-results})}}\\*
\S10 & in-service adaptation cheaper than retraining after drift & drift-economics campaign (E-8) & 2/3 on the cost bar & 58--64\% of retrain (s2 $1.96\times$); service continuity 3/3 & boundary \\
\S10 & returning regime: relearning savings from retained knowledge & cyclic-drift campaign (E-12) & ALL PASS & savings $\times$7--10 (3/3); service 3/3 & supported \\
\S10 & review, not windows or capacity, treats sequential-switch forgetting & forgetting ladder (E-9) & R-1 3/3; window/capacity arms fail as registered & untreated $\times$552--1{,}362; review arms $\times$72--300 & supported \\
\S10 & targeted review beats random at equal volume and diversity & targeting correction (E-11) & 2/3; repetition confound isolated & volume dose 5--25\% monotone; ramp 0/3 refuted & boundary \\
\S10 & champion-assembly registered bars & course family (E-13/E-15) & FAIL as registered (objective bundling) & L-1/L-6 & open \\
\S10 & spaced-repetition course count is the retention dose & course family (E-13/E-15) & measured, monotone in every seed & $\times$341--800 / $\times$252--438 / $\times$73--149 & supported \\
\addlinespace
\multicolumn{6}{@{}l}{\emph{Pedagogy principles, tested by GUARD (App.~\ref{sec:pedagogy})}}\\*
C & undefended plasticity defects under coherent temptation & GUARD r1 (App.~\ref{sec:guardprog}) & measured & 0.32--0.37 vs incentive $+0.41$ & observation \\
C & re-anchored immunity recovers half at 11\% cost & GUARD r2 & partial; central claims NOT SUPP.\ (verbatim) & $+0.10/{+0.10}$ & boundary \\
C & calm governor: mechanism sound, regime absent & GUARD r3 probes (battery withheld, disclosed) & regime observation & fidelity $0.67 \to 0.51$ under calm & observation \\
C & education thesis under fair instrument + proper pedagogy & GUARD r4 (App.~\ref{sec:guardprog}) & SUPPORTED (corrected instrument; 2/5 under the original) & slope 3/5; TTD 5/5 ($\ge$ rule; strict 3/5, two floor ties); recovery 4/5 & supported \\
C & frozen substrate less educable than growable & GUARD r4 solidity gate & exhibit & frozen missed 0.90 on s101, 6/6 curricula; growable 0.93 & diagnostic \\
C & R4 verdict under the original level instrument & re-scoring of R4 stores & does not pass & 2/5 at margin 0.05; corrected counts baseline-insensitive & diagnostic \\
\addlinespace
\multicolumn{6}{@{}l}{\emph{Empirical observations (\S\ref{sec:laws-digest})}}\\*
\S14 & L1 exams see capacity; online scores do not & G-GROW-1 & observation & $0.64$ vs $0.44$ & observation \\
\S14 & L2 probe economics dominate small-scale gates & X2 & observation & $43 \to 17.4$\,s/life & observation \\
\S14 & L3 instantaneous gates cannot see lifetime consequences & G-GROW-1 case & observation & $0.57 \to 0.37$ & observation \\
\S14 & L4 review dose non-monotonic & GUARD-1 probes & observation & $0.20 \to 0.45 \to 0.17$ & observation \\
\S14 & L5 exposure economics bind in-curriculum & GUARD-1 C-GD5 & observation & $1.0$ solo vs $\sim$chance & observation \\
\S14 & L6 transition dip is the price of tracking & GUARD & observation & --- & observation \\
\S14 & L7 defection tracks incentive without immunity & GUARD r1 & observation & $0.32$--$0.37$ vs $+0.41$ & observation \\
\S14 & L8 verify the instrument before the theory & X2 autopsy & observation & invented event name green-lit & observation \\
\S14 & L9 immunity timing dilemma; re-anchoring partial & GUARD r2 & observation & half at 11\% & observation \\
\S14 & L10 the metric decides governability & G-GEN-1 vs G-GROW-1 & observation & $0$ vs $2$--$8$ adoptions & observation \\
\S14 & L11 non-growing world needs no growth & stationary companion & observation & 8/8 regretted & observation \\
\S14 & L12 governors have regimes & GUARD r3 & observation & $0.67 \to 0.51$ & observation \\
\S14 & L13 capacity value ex-ante unobservable & X1 control-arm autopsy & observation & $0.009 \pm 0.017$ vs $\sim$0.1 & observation \\
\S14 & L14 audits blind off their metric & X2/X2b & observation & kept widen; exam fell 0.04 & observation \\
\S14 & L15 the price of silence & X1 & observation & C-N1 $3/5 \to 1/5$ & observation \\
\S14 & L16 birth capacity compounds & X2 diagnosis & observation & gap $0.046 \to 0.034$, not closed & observation \\
\S14 & L17 frontier: proof latency = incumbent head start & X1b/X2b & observation & C-N1 1/5 with C-N4 5/5 & observation \\
\S14 & L18 rollback anchors immutable until resolved & X1b autopsy & observation & $11.38 \to 0.65$ & observation \\
\S14 & L19 purchase rate of serialized candidates & X-SAT r1 & observation & $\sim$1 promotion/30--40k rows & observation \\
\S14 & L20 discrete memorization binds parameters & X-SAT G-GEN leg & observation & $20\times$ params flat & observation \\
\addlinespace
\multicolumn{6}{@{}l}{\emph{Evaluative learning and benchmark lanes (\S\ref{sec:eval-learn}, \S\ref{sec:cl-bench}); E-S*/Y-* name the registered designs and}}\\*
\multicolumn{6}{@{}l}{\emph{per-seed run stores of the \texttt{paper-experiments} companion repository (Reproducibility Statement;}}\\*
\multicolumn{6}{@{}l}{\emph{store files are named by battery, e.g.\ \texttt{E\_S7\_RESULTS.json}, \texttt{Y7\_FULL40\_RESULTS.json})}}\\*
\S8 & three-regime law of growth under evaluative learning & E-S1/E-S5/E-S6/E-S7 & per registered bars; unresolved and boundary items rowed separately below & $-33.9$ vs $-1.9$; 8/10 vs 0/10 & supported \\
\S8 & quasi-static identity extends to the policy loop & TIER-1--4 + 9 library cases & PASS & $\le 2.2\times 10^{-11}$; $2.8\times 10^{-17}$ & supported \\
\S8 & governance exact under reward learning & clone0; E-S2; E-S3 & PASS & refusal 5/5, identical finals; drift 5/5; $5.9$--$95\times$; stack 5/5 & supported \\
\S11 & capacity under sequential arrival & Y-3 both scales; Y-4; Y-6 controlled & PASS & 9/10, 9/10, 8/10, identical finals 10/10; 77\% & supported \\
\S11 & long-sequence plasticity (designed regime; canonical cross-check) & Y-7/Y-8/Y-9; Y-9b & PASS & first-pass decline lower 10/10; canon 8/10 & supported \\
\S11 & just-in-time refresher (independent exam days) & Y-9e & PASS & 10/10; $+3.4$ vs $-1.4$ & supported \\
\S8 & optimizer-moment hypothesis (falsification test) & E-S1P/S4 store & REFUTED --- transient carried by the function change, not optimizer state & moment transplant left the cost in place & supported \\
\S8 & LunarLander long-run net effect & E-S6 & below protocol resolution & 4/10 with wide swings & open \\
\S8 & one sparse-reward task at moderate scale & E-S5x series & out of reach after three remedies & entropy collapse, live-probed & boundary \\
\S11 & Fashion uncontrolled lines at registered dose & Y-6 & MISS (each line 2/5) & dose boundary, controlled line passes & boundary \\
\S11 & forty-task multi-pass companion line & Y-9 & MISS (7/10 vs $\geq$8/10) & primary plasticity line unaffected; subject answered by the review protocol (10/10) & boundary \\
\S11 & large-from-birth keep-learning control & --- & not run & open control & open \\
\end{longtable}}

\begin{table}[t]
\centering\scriptsize
\caption{Per-seed values for the adjudicating lanes of
\S\ref{sec:eval-learn}/\S\ref{sec:cl-bench} (gaps are
grown-arm minus protocol-matched twin; declines in points;
exam-day means are all-task post-review accuracies).}
\label{tab:perseed-new}
\begin{tabular}{@{}>{\raggedright\arraybackslash}p{0.30\linewidth}>{\raggedright\arraybackslash}p{0.64\linewidth}@{}}
\toprule
Lane & Per-seed values (paired, seed order) \\
\midrule
Stationary control (gap; registered as a no-bar report) & $-60.5$, $-4.2$, $+27.9$, $+10.8$, $-54.4$ \\
Staged, reward-only (gap) & $-73.8$, $-8.9$, $-10.2$, $-39.0$, $-37.5$ \\
Staged, alternating (gap) & $0.0$, $-4.2$, $-5.4$, $0.0$, $0.0$ \\
Quadratic staged, reward-only (gap) & $-14.0$, $+50.6$, $-11.5$, $+1.6$, $+27.0$ \\
Quadratic staged, alternating (gap) & $+1.9$, $-4.3$, $+7.2$, $-9.8$, $+10.4$ \\
CartPole post-change AUC (grown/twin) & 2{,}557/2{,}259, 2{,}144/2{,}062, 2{,}416/1{,}928, 1{,}701/2{,}405, 2{,}244/2{,}244 \\
40-task first-pass decline (grown/fixed) & $5.0/7.6$, $1.4/10.8$, $1.2/7.7$, $-0.6/5.3$, $-2.6/6.9$, $0.5/8.7$, $-0.8/9.6$, $2.3/8.7$, $-0.7/4.9$, $2.3/8.2$ \\
Exam-day mean (grown/fixed) & 0.193/0.150, 0.241/0.169, 0.222/0.167, 0.197/0.139, 0.209/0.148, 0.266/0.159, 0.197/0.186, 0.191/0.152, 0.243/0.146, 0.195/0.160 \\
\bottomrule
\end{tabular}

\end{table}

\subsection{How to audit a row}
Take any row: its campaign column names the appendix section
carrying the registered acceptance, per-seed values, and
probe-ledger disclosures; rows of the evaluative-learning and
benchmark group trace instead through the Reproducibility
Statement, whose battery identifiers name the registered design
documents and per-run JSON stores of the
\texttt{paper-experiments} companion repository (per-seed
values for the adjudicating lanes:
Table~\ref{tab:perseed-new}). That section names its registered
report and result store; the stores are write-once run
directories under the released artifact set, each with its
resolved configuration and content hashes. The row's key
number therefore traces in two hops from this table to a
hashed artifact on disk.

\clearpage
\section{Glossary of Named Quantities and Roles}
\label{app:glossary}

\begin{table}[htbp]
\centering
\small
\caption{Glossary. Every role is a replaceable part behind a
registry (\S\ref{sec:compositional}); every quantity is recorded
in the audit artifacts.}
\label{tab:glossary}
\begin{tabular}{llp{6.6cm}l}
\toprule
Symbol / name & Kind & One-line meaning & Defined \\
\midrule
scope & structure & a network in the inclusion tree (root or inner) & Def.~1 \\
composite node & structure & hidden unit hosting an inner network & Def.~1 \\
inclusion tree & structure & the containment relation of scopes & \S\ref{sec:algebra} \\
tree height & quantity & height of the inclusion tree (per branch) & \S\ref{sec:algebra} \\
layer depth & quantity & $1+$blocks: a scope's composed stages (per scope) & \S\ref{sec:compositional} \\
$u_j$ & signal & EMA instability of node $j$'s input-weight updates (aims growth) & \S\ref{sec:signals} \\
$R_{\mathrm{inv}}$ & signal & fine-to-coarse amplitude ratio at the top split & \S\ref{sec:neutralops} \\
gain ledger & record & realized fixed-horizon gain of each growth event & \S\ref{sec:compositional} \\
predictability certificate & policy & four checks gating Tier-1 extrapolation & \S\ref{sec:compositional} \\
$\lambda$-block & structure & grown cycle: $(L_{\mathrm{in}}, b_\lambda, L_{\mathrm{out}})$ at a scope's chain end & \S\ref{sec:loops-lambda} \\
contraction certificate & policy & enforced bound $c_\varphi\sigma_{\max}(L_{\mathrm{in}})\sigma_{\max}(L_{\mathrm{out}}) \le \rho_{\max}$ & \S\ref{sec:loops-lambda} \\
$K_{\max}$, $k$ & quantity & iteration cap; executed iteration count (audited) & \S\ref{sec:loops-lambda} \\
local process objective & signal & label-blind newborn loss: consistency $+$ collapse hinge, Eq.~\eqref{eq:spu} & \S\ref{sec:loops-spu} \\
functional consistency & signal & response stability under bounded internal perturbation & \S\ref{sec:loops-spu} \\
scale-hierarchy ratios & policy & host/body mass ratio, host floor, shaping steps & \S\ref{sec:scalehierarchy} \\
extrapolator & role & fits curve families to the energy series & \S\ref{sec:compositional} \\
forecastability & role & spectral-entropy regularity check & \S\ref{sec:compositional} \\
changepoint & role & online regime-break detection (BOCPD) & \S\ref{sec:compositional} \\
backtest & role & rolling-origin skill check of the extrapolator & \S\ref{sec:compositional} \\
pricer & role & zero-attach probes, asymptote-read & \S\ref{sec:compositional} \\
combiner & role & merges prices and proposals into one decision & \S\ref{sec:compositional} \\
gate & mechanism & held-out adoption test for every change & \S\ref{sec:gate} \\
widen-only / adaptive & modes & direction policy: additive only / both axes & \S\ref{sec:mode} \\
$\omega,\rho,\sigma,\delta,\Phi$ & operators & widen, refine, interface, deepen, re-found & Table~\ref{tab:operators} \\
\bottomrule
\end{tabular}
\end{table}

\end{document}